\documentclass{article} 
\usepackage{preprint_style,times}

\usepackage{amsmath,amsfonts,bm}

\def\eqref#1{equation~\ref{#1}}

\def\1{\bm{1}}

\DeclareMathAlphabet{\mathsfit}{\encodingdefault}{\sfdefault}{m}{sl}
\SetMathAlphabet{\mathsfit}{bold}{\encodingdefault}{\sfdefault}{bx}{n}

\def\gD{{\mathcal{D}}}

\def\gL{{\mathcal{L}}}

\def\gN{{\mathcal{N}}}

\def\gT{{\mathcal{T}}}

\def\sP{{\mathbb{P}}}

\newcommand{\pdata}{p_{\rm{data}}}

\newcommand{\E}{\mathbb{E}}

\newcommand{\R}{\mathbb{R}}

\newcommand{\KL}{\mathbb{D}_{\mathrm{KL}}}

\newcommand{\xhdr}[1]{{\noindent\bfseries #1}.}
\newcommand{\cut}[1]{}

\newcommand{\namelong}{\textsc{Normalizing Flow Maps}\xspace}
\newcommand{\nameshort}{\textsc{NFM}\xspace}
\newcommand{\nameshorts}{\textsc{NFM}\textnormal{s}\xspace}
\newcommand{\nfmr}{\textsc{NFM-R}\xspace}

\newcommand{\nfmd}{\textsc{NFM-D}\xspace}

\newcommand{\pbase}{p_{\rm{base}}}

\newcommand{\cifar}{CIFAR-$10$\xspace}
\newcommand{\imnet}{ImageNet-$64$\xspace}

\usepackage[framemethod=TikZ]{mdframed}
\mdfdefinestyle{MyFrame}{%
    linecolor=black,
    outerlinewidth=.3pt,
    roundcorner=5pt,
    innertopmargin=1pt, 
    innerbottommargin=1pt, 
    innerrightmargin=1pt,
    innerleftmargin=1pt,
    backgroundcolor=black!0!white}

\mdfdefinestyle{MyFrame2}{%
    linecolor=white,
    outerlinewidth=1pt,
    roundcorner=1pt,
    innertopmargin=3pt,
    innerbottommargin=2pt,
    innerrightmargin=7pt,
    innerleftmargin=7pt,
    backgroundcolor=black!3!white}

\mdfdefinestyle{MyFrameEq}{%
    linecolor=white,
    outerlinewidth=0pt,
    roundcorner=0pt,
    innertopmargin=0pt,
    innerbottommargin=0pt,
    innerrightmargin=7pt,
    innerleftmargin=7pt,
    backgroundcolor=black!3!white}

\mdfdefinestyle{nfmdefinitionstyle}{
    backgroundcolor=black!3,
    linecolor=black!35,
    linewidth=0.5pt,
    roundcorner=0pt,
    innerleftmargin=9pt,
    innerrightmargin=9pt,
    innertopmargin=6pt,
    innerbottommargin=6pt,
    skipabove=10pt,
    skipbelow=10pt
}

\usepackage{hyperref}
\usepackage{url}

\usepackage[utf8]{inputenc} 
\usepackage[T1]{fontenc}    
\usepackage{url}            
\usepackage{booktabs}       
\usepackage{threeparttable}
\usepackage{amsfonts}       
\usepackage{nicefrac}       
\usepackage{microtype}      
\usepackage[dvipsnames]{xcolor}

\usepackage[utf8]{inputenc} 
\usepackage[T1]{fontenc}    
\usepackage{hyperref}       
\usepackage{booktabs}       
\usepackage{enumitem}
\usepackage{mathtools}
\usepackage{todonotes}
\usepackage{amssymb,fge}
\usepackage{wrapfig}
\usepackage{bbm}
\usepackage{mathrsfs}
\usepackage{soul}
\usepackage{array}
\usepackage{multirow}
\usepackage{algorithm}
\usepackage{tikz}
\usepackage{subcaption}
\usepackage{wrapfig}
\usetikzlibrary{positioning,arrows.meta,shapes.geometric}
\usepackage{algorithm}
\usepackage{listings}

\usepackage{nicefrac}       
\usepackage{microtype}      
\usepackage{xcolor}         
\usepackage{colortbl}
\usepackage{enumitem}
\usepackage{xspace}
\usepackage{lineno}
\usepackage{titletoc}

\usepackage{amsmath, amsthm, mathtools}
\usepackage{thmtools}
\makeatletter
\@ifundefined{newcounteralias}{}{%
  \renewcommand\thmt@autorefsetup{%
    \@xa\def\csname\thmt@envname autorefname\@xa\endcsname\@xa{\thmt@thmname}%
  }%
}
\makeatother
\usepackage{cleveref}
\usepackage[normalem]{ulem}
\usepackage{changepage}
\usepackage{float}
\usepackage{needspace} 
\usepackage{makecell}
\usepackage{pifont}
\usepackage{tikz}
\usepackage{fontawesome5}

\newcommand{\xmark}{\ding{55}} 
\definecolor{sectiongray}{gray}{0.90}
\definecolor{refgray}{gray}{0.55}
\usepackage{algorithm}
\usepackage{algorithmic}
\newtheorem{theorem}{Theorem}
\newtheorem{lemma}[theorem]{Lemma}

\newtheorem{proposition}[theorem]{Proposition}

\newtheorem{remark}[theorem]{Remark}
\theoremstyle{definition}
\newmdtheoremenv[
    style=nfmdefinitionstyle
]{nfmdefinition}{Definition}
\theoremstyle{plain}

\newtheoremstyle{nfmremarktext}
    {5pt}{0pt}
    {\normalfont}{}
    {\bfseries}{.}{0.5em}{}

\theoremstyle{nfmremarktext}
\newmdtheoremenv[
    backgroundcolor=white,
    linecolor=black!25,
    linewidth=0.4pt,
    roundcorner=0pt,
    topline=true,
    bottomline=true,
    leftline=true,
    rightline=true,
    innerleftmargin=9pt,
    innerrightmargin=9pt,
    innertopmargin=6pt,
    innerbottommargin=6pt,
    skipabove=8pt,
    skipbelow=8pt
]{nfmremark}{Remark}
\theoremstyle{plain}

\hypersetup{
	colorlinks=true,       
	linkcolor=RoyalBlue,        
	citecolor=RoyalBlue,        
	filecolor=magenta,     
	urlcolor=RoyalBlue         
}

\definecolor{ink}{HTML}{1B2436}
\definecolor{muted}{HTML}{6E788A}
\definecolor{blue}{HTML}{3F6FA0}
\definecolor{green}{HTML}{21865B}
\definecolor{red}{HTML}{A64B45}
\definecolor{orange}{HTML}{D8872F}
\definecolor{purple}{HTML}{7155A6}
\definecolor{stream}{HTML}{CDD5E1}
\definecolor{rule}{HTML}{D9DEE7}

\title{Normalizing Flow Maps}

\author{%
Francesco M. Ruscio $^{1,2,3}$\thanks{Equal Contribution.}
\hspace{0.01em}
\thanks{Correspondence to \texttt{fruscio@aithyra.at},\, and \texttt{\{t.hitge,joey.bose\}@imperial.ac.uk}.}\hspace{0.01em} , Timothy James Hitge$^{2*\dagger}$, Fred Zhangzhi Peng$^4$, \\
\textbf{Danyal Rehman$^{1,5,6}$}, \textbf{Zander W. Blasingame$^1$}, \textbf{Nicholas M. Boffi$^7$},\\
\textbf{Michael M. Bronstein$^{1,8}$}, \textbf{Alexander Tong$^1$}, \textbf{Avishek Joey Bose$^{2,5\dagger}$} \\
 $^1$AITHYRA, $^2$Imperial College London, $^3$TU Wien, $^4$Duke University \\ $^5$Mila -- Qu\'ebec AI Institute, $^6$MIT, $^7$Carnegie Mellon University, $^8$University of Oxford\\[3pt]
{\small\normalfont
\faGlobe \enspace
\href{https://the-bose-lab.github.io/normalizing-flow-maps/}{\texttt{https://the-bose-lab.github.io/normalizing-flow-maps/}}
}
\\
{\small\normalfont
\faGithub \enspace
\href{https://github.com/francescoshox/NormalizingFlowMaps}{\texttt{https://github.com/francescoshox/NormalizingFlowMaps}}
}
}

\newif\ifshowcomments

\ifshowcomments
  \newcommand{\inlineComment}[3]{\textcolor{#2}{[#3 -\textit{#1}]}}
  \newcommand{\shox}[1]{\inlineComment{Shox}{magenta}{#1}}
  
  \newcommand{\tim}[1]{\inlineComment{Tim}{cyan}{#1}}
  \newcommand{\NB}[1]{\inlineComment{Nick}{red}{#1}}
  \newcommand\zander[1]{\noindent{\color{green} {\bf \fbox{Zander}} {\it#1}}}
\else
  \newcommand{\inlineComment}[3]{}
  \newcommand{\shox}[1]{}
  \newcommand{\tim}[1]{}
  \newcommand{\NB}[1]{}
  \newcommand{\zander}[1]{}
\fi

\newcommand{\headlinepanel}[2]{%
    \begin{tikzpicture}
        \node[inner sep=0] (img) {\includegraphics[width=\linewidth]{images/#1}};
        \node[anchor=south west, text=white, font=\scriptsize\bfseries,
              fill=black, fill opacity=0.45, text opacity=1, inner sep=2pt]
              at ([xshift=3pt,yshift=3pt]img.south west) {#2};
    \end{tikzpicture}%
}

\iclrfinalcopy 
\begin{document}

\begingroup
\renewcommand{\thefootnote}{}

\endgroup

\maketitle

\begin{abstract}
\looseness=-1
Scalable approaches to learning neural dynamical systems have placed diffusion, flow matching, and few-step generators in flow maps at the forefront of generative modeling. Such state-of-the-art models have limited access to explicit model densities despite evolving from classical generative modeling approaches that leveraged tractable likelihoods, e.g., Normalizing Flows and Autoregressive models. In this paper, we combine these threads of few-step generation and tractable sample likelihoods and introduce \namelong (\nameshort), a class of \emph{exactly invertible pushforward maps} that generalizes normalizing flows and flow maps. More precisely, \nameshorts learn the integrator of an underlying dynamical system, but crucially, they provide efficient exact likelihood evaluation by construction. 
To enable scalable training of \nameshorts, we derive two complementary axes of consistency inherent to \nameshorts that yield (self) distillation objectives: (1) pointwise matching, which unlocks $\ell_2$-based flow-map style regression losses, and (2) distributional matching, which matches pushforward distributions to the marginals of a prescribed forward process.
On conditional image generation on \cifar and \imnet, \nameshorts achieve state-of-the-art sample quality among parameter-matched normalizing flows, while supporting likelihood-based RLHF. For equilibrium sampling in Boltzmann Generators, \nameshorts outperform prior state-of-the-art normalizing flows on both single-system and transferable peptide benchmarks. 
\end{abstract}

\section{Introduction}
\label{sec:introduction}

\looseness=-1
Learning a transport map between a tractable reference distribution and the empirical distribution of complex data has become the de facto paradigm in generative modeling, with varied application domains spanning audiovisual media~\citep{ho2022video, liu2023audioldm}, biomolecular modeling~\citep{abramson2024accurate, watson2023novo}, and scientific simulation~\citep{shu2023physics, du2024conditional}. Driving this widespread adoption is the ability to leverage standard alignment techniques such as best-of-$N$ selection~\citep{dalle}, inference-time steering~\citep{dhariwal2021guidance,ho2022classifier,Bansal_2023_CVPR,ma2025inference}, and post-training algorithms~\citep{black2024training, liu2025flow} which effectively sharpen the model distribution---by increasing sample rewards assessed by an oracle. However, each generated sample necessitates computationally expensive numerical integration of an ODE or SDE---producing a substantial inference tax that computationally challenges the steering and post-training pipeline~\citep{wohlwend2025boltz,scarpellini2026few}.

\looseness=-1
Indeed, while typical Autoregressive LLMs have a mature post-training stack---e.g., PPO~\citep{schulman2017proximal}, GRPO~\citep{shao2024deepseekmathpushinglimitsmathematical}, etc.---corresponding algorithmic frameworks for diffusion and flow model post-training are both vast and varied~\citep{wallace2024diffusion, prabhudesai2023aligning, clark2024directly, imagereward}.
The crucial distinction is that LLMs natively provide tractable sample likelihoods, whereas evaluating exact densities for diffusion and flow-matching models is computationally challenging. Specifically, evaluating sample likelihoods of flow models requires integrating the instantaneous change of variables along the underlying ODE, computing a high-dimensional network Jacobian, and necessitates stochastic trace estimation~\citep{chenneuralode,meyer2021hutch++}. Instead, it is computationally more attractive to target the reward-tilted distribution, by modifying the drift of the underlying dynamics through path-space objectives. These objectives, however, crucially rely on high-variance value estimates of the terminal reward.
More precisely, value function estimation necessitates propagating terminal rewards through noisy intermediate states,
which requires backpropagating through time \citep{ben-hamu2024dflow}--or a noisy single-step denoising approximation~\citep{blasingame2025greed,lee2026reward}.
These algorithmic features combine and give rise to tilting objectives that are inherently high variance, often leading to unstable post-training and mode collapse.

\begin{figure*}[t]
  \centering
\includegraphics[width=\textwidth]{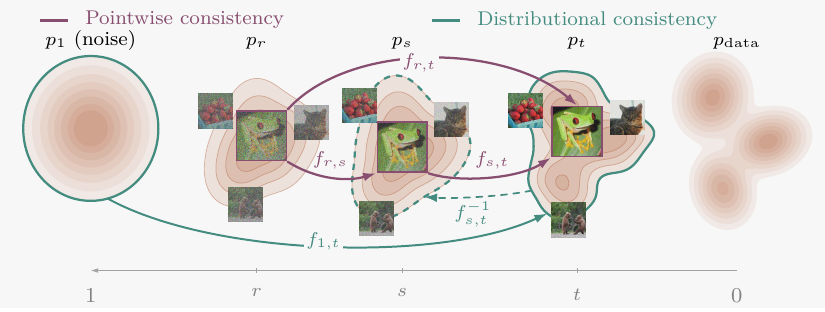}
  \vspace{-20pt}
  \caption{
    \small \looseness=-1 Overview of \nameshort, which performs generative transport along a stochastic interpolant. \nfmr exploits pointwise matching at the particle level, while \nfmd leverages distributional matching.
  }
  \vspace{-25pt}
  \label{fig:normalizing-flow-map}
\end{figure*}

\looseness=-1
As generative models are deployed at scale, the cumulative cost of repeated sampling for generation, steering, and post-training makes inference efficiency a central design consideration.
This practical motivation has led to the emergence of flow maps~\citep{boffi2025flow,geng2025mean,boffi2025how}, that directly learn \emph{the flow integrator}, effectively compressing numerical simulation into few network evaluations. Although the exact ODE flow is a diffeomorphism, standard parameterizations are not invertible, nor do they provide tractable Jacobian determinants. While flow maps reduce the inference tax, they do not resolve the crucial likelihood-access gap---limiting their usage in scientific applications like Boltzmann Generators~\citep{noe2019boltzmann}. As a result, classical Normalizing Flows (NF)~\citep{tabak_density_2010,tabak_family_2013,dinh_density_2017,rezende_variational_2015} trained with MLE remain an attractive exact likelihood-based model class. However, their practical deployment critically relies on noise injection into the clean data distribution for stable training, which unfortunately eliminates the promised access to cheap sample likelihood~\citep{chen_iterative_2026}.
 Taken together, these considerations motivate a new generative model that combines few-step sampling of flow maps with exact sample likelihoods for scalable, likelihood-based post-training.

\looseness=-1
\xhdr{Current work}
In this paper, we introduce \namelong (\nameshort), a new generative model family, depicted in~\cref{fig:normalizing-flow-map}, that is purpose-built to deliver two objectives of interest: 1.) \underline{\emph{Exact sample likelihoods}} and 2.) \underline{\emph{Few-step inference}}; thereby generalizing classical invertible normalizing flows and endowing flow maps with tractable likelihoods. An \nameshort defines a time-conditioned generative transport operator that is tied to the law of a known stochastic interpolant. However, unlike existing model classes, \nameshort by construction defines a family of tractable invertible maps connecting arbitrary points along a prescribed probability path. Consequently, \nameshorts provide exact sample likelihoods \emph{without} expensive numerical integration. Importantly, restricting this family to the prior and data endpoints recovers a normalizing flow, and querying arbitrary time pairs yields the few-step transport characteristics of a flow map.  As a result, \nameshorts are naturally suited for applications that require sample likelihoods, such as Boltzmann generators~\citep{noe2019boltzmann} and RLHF post-training. 

\looseness=-1
To train \nameshorts, we operationalize two complementary matching principles inherent to \nameshorts, which in turn lead to self-distillation training objectives. In particular, our key insight is that \nameshorts must satisfy a known stochastic interpolant at a \emph{pointwise} level and/or match its \emph{marginal distribution}. The former condition unlocks regression-based training common in flow maps, but adapted to \nameshort, while the latter is an altogether new algorithmic framework that trains few-step generators by exploiting \nameshorts family of (end-point) invertible maps. We  summarize our main contributions:

\begin{enumerate}[topsep=0pt, partopsep=0pt, itemsep=0pt, parsep=0pt, leftmargin=*]
\item \looseness=-1 \xhdr{Model class} We introduce \nameshort, a new generative model class, that learns an invertible pushforward map, while natively supporting few-step sampling and exact sample likelihoods.

\item \looseness=-1 \xhdr{Training algorithms} Defined by matching principles (\S\ref{sec:model_computational_properties}), we design two instantiations in \nfmr and \nfmd, which leverages pointwise and distributional matching respectively. In particular, \nfmr optimizes $\ell_2$-regression objectives. Meanwhile, \nfmd matches the induced pushforward distributions along the marginal law of the known forward process.

\item \looseness=-1 \xhdr{Empirical analysis} We evaluate \nameshort across image generation, post-training, and equilibrium sampling. On \cifar and \imnet, \nameshort improves over MLE-trained NFs in FID and FLD while for BGs \nameshort improves over prior continuous flow approaches. We also demonstrate that exact tractable likelihoods of \nameshort enable direct post-training unlike diffusion models. 
\end{enumerate}

\section{Background and Preliminaries}
\label{sec:background}
\looseness=-1
Given a dataset of samples, $\gD = \{x^k\}^n_{k=1}$, from a data distribution $\pdata \in \sP(\R^d)$, and a tractable base distribution $\pbase \in \sP(\R^d)$, generative models based on dynamical systems parametrize a map $f^\theta:\mathbb{R}^{d}\mapsto \mathbb{R}^d$---with learnable parameters $\theta$. The induced pushforward distribution is thus $p^\theta \coloneqq [f^\theta]_{\#}(\pbase)$, and the goal of learning is then to approximate $p^\theta(x) \approx \pdata(x)$ in a distributional sense. We use superscripts to designate indices, e.g., $k \in [n]$, and reserve subscripts to indicate a notion of time $t \in [0,1]$ associated with the mass transport problem. Finally, we set $p_{1}\coloneqq \pbase = \mathcal{N}(0,I)$ to be the base distribution, $p_{0}\coloneqq \pdata$ the data distribution and $p^{\theta}=[f^\theta]_{\#}(p_{1})$. 

\subsection{Stochastic Interpolants and Flow Maps}
\label{sec:stochastic_interpolants}

\looseness=-1
Given $\sP(\R^d)$ the space of probability distributions on $\R^d$, a \emph{stochastic interpolant}~\citep{albergo2023building,albergo2025stochastic} $I:[0,1]\times\R^d\times\R^d\to\R^d$ is a stochastic process that mixes samples from the base and target measures at the particle level $(t, x_0, x_1) \mapsto  I_t(x_0,x_1) := x_t = \gamma_t x_0 + \beta_t x_1$, for a choice of smooth functions $\gamma$ and $\beta$, with boundary constraints $\gamma_0 = 1$, $\gamma_1 = 0$, $\beta_0 = 0$ and $\beta_1 = 1$. Given a coupling, $\pi$, of base and target measures $(x_0, x_1) \sim \pi(x_0, x_1)$, the stochastic interpolant defines a probability path that connects the measures $p_0$ and $p_1$ in time 
\begin{equation}
\label{eq:marginal_prob_path}
    p_t = \mathrm{Law}(I_t), \quad \text{such that} \, \int \pi(x_0, x_1)d x_0 = p_1(x_1) \, \int \pi(x_0, x_1)d x_1 = p_0(x_0).
\end{equation}
\looseness=-1
The marginal law in~\cref{eq:marginal_prob_path} defines a \emph{probability-flow} in time $t\in[0,1]$ that pushes samples from $p_0$ to $p_1$ via the velocity field $v_t(x) = \mathbb{E}_{x_t \sim p_t(x_t)}[ \partial I_t / \partial t \mid I_t = x_t]$, which averages over all interpolants. Unfortunately, this marginal velocity is computationally inaccessible. Instead, the conditional probability flow $p_t(\cdot | x_0, x_1)$ indexed for all $t\in(0,1)$ can be leveraged to train a neural ODE~\citep{chenneuralode}. Specifically, setting $\gamma_t \equiv 1-t$ and $\beta_t \equiv t$ recovers well-known methods such as flow-matching~\citep{lipman2023flow,liu2023flow,albergo2023building}, under the linear interpolant $x_t=(1-t)x_0+t x_1$, yielding a simple simulation-free $\ell_2$-regression objective, 
\begin{equation}
\label{eq:flow_matching}
    \gL_{\rm{FM}} = \mathbb{E}_{(x_0, x_1)\sim \pi(x_0,x_1),\, t \sim \nu(t)}\left[\|v_t^\theta(x_t) - v^*_t(x_t|x_0, x_1)\|^2\right], \quad v^*_t(x_t|x_0, x_1) = x_1 - x_0,
\end{equation}
\looseness=-1
where $\nu(t)$ is a distribution in the interval $(0,1)$.
Despite scalable training of~\cref{eq:flow_matching}, likelihood evaluation is expensive due to computing the  divergence~\citep{grathwohl2018ffjord,meyer2021hutch++}.

\looseness=-1
\xhdr{Flow Maps}
To reduce the number of sampling steps at inference, it is possible to learn a \emph{flow map} $X_{s,t}:[0,1]^2\times\mathbb{R}^d\mapsto \mathbb{R}^d$~\citep{boffi2025flow, boffi2025how}. The flow map is the operator that uniquely satisfies the jump condition $X_{s,t}(x_s)=x_t$, for all $(s,t)\in[0,1]^2$ when $(x_t)_{t\in[0,1]}$ is any solution to the probability flow ODE. Flow maps are often parametrized by the \emph{average velocity} $v^{\theta}_{s,t}(x_s)$:
\begin{equation}
    \label{eq:flow_map_avg_velocity}
    X^{\theta}_{s,t}(x_s)=x_s+(t-s)v^{\theta}_{s,t}(x_s), \quad v^{\theta}_{s,t}(x_s)=\frac{1}{t-s}\int_{s}^{t}v_{\tau}(x_{\tau})d\tau.  
\end{equation}
\looseness=-1
Furthermore, the true flow map simultaneously satisfies the following consistency conditions,
\begin{equation}
\label{eq:flow_map_conditions}
    \underbrace{
    \partial_t X_{s,t}(x_s)
    = v_{t,t}\!\left(X_{s,t}(x_s)\right)
}_{\text{Lagrangian}},\,
\underbrace{
    \partial_s X_{s,t}(x_s)
    = -v_{s,s}(x_s)\nabla X_{s,t}(x_s)
}_{\text{Eulerian}},\,
\underbrace{
    X_{s,t}\!\left(X_{r,s}(x_r)\right)
    = X_{r,t}(x_r)
}_{\text{Semi-group}}.
\end{equation}
\looseness=-1
Typical flow map objectives convert the above conditions into distillation losses with stop-gradients~\citep{frans2025one,geng2025mean,song2023consistency,sabour2025align}.

\subsection{Normalizing Flows}
\label{subsec:NF}
\looseness=-1
Normalizing Flows (NFs)~\citep{tabak_density_2010,tabak_family_2013, rezende_variational_2015} parametrize a diffeomorphism between distributions using exactly invertible architectures. The key benefit is exact likelihood evaluation via the classical change of variables formula,
\begin{equation}
    \label{eq:change_of_var}
    \log p_0^{\theta}(x_0)
    =
    \log p_1(x_1)
    - \log \left|
        \det \frac{\partial f^{\theta}(x_1)}{\partial x_1}
    \right|, \quad \text{where} \, p_1(x_1) := \gN(0,I).
\end{equation}
\looseness=-1
Certain classes of NFs are provably universal density approximators~\citep{teshima2020coupling, huang2018neural} and can be constructed through the composition of $m$ elementary diffeomorphisms, $f^{\theta}=f^{\theta,m}\circ f^{\theta,{m-1}}\circ\cdots \circ f^{\theta,1}$, and
trained using Maximum Likelihood Estimation (MLE) according to~\cref{eq:change_of_var}. To scale modern NFs, current approaches leverage a causal Transformer backbone with Autoregressive flows (TARFlow)~\citep{zhai2025normalizing}, interchanged with permutation layers. 
The flow transformations are parametrized by scale and shift functions
$\alpha^{\theta,m}, \mu^{\theta,m}:
\R^{n\times d} \to \R^{n\times d}$.
Specifically, at position $k \in [n]$ in the sequence, the $m$-th transformation is
$
f_k^{\theta,m}(x)
=
\left(x_k-\mu_k^{\theta,m}(x_{<k})\right)
\odot
\exp\left(-\alpha_k^{\theta,m}(x_{<k})\right).
$

\section{Method}
\label{sec:normalizing_flow_maps}

\looseness=-1
We introduce \namelong (\nameshort), a new generative model class built from two-time \emph{pushforward maps} that learn generative transport along the probability path of a stochastic interpolant.
More precisely, \nameshorts generalize classical normalizing flows across time, combining the large transport-step capabilities of jump operators while endowing previous flow maps with tractable inverses that enable exact sample likelihoods, i.e., $p_0(x_0)$. 
We first define the model and its parameterizations (\S\ref{sec:model_definition}), then expose the inherent computational properties at the particle and distributional level. These principles yield (self)-distillation objectives (\S\ref{sec:distillation_losses}) resulting in \nfmr that is constructed from pointwise matching of the probability flow and \nfmd that leverages distributional matching of its marginals. Finally, we demonstrate RLHF-based post-training in \S\ref{subsec:posttraining}.

\subsection{Normalizing Flow Map}
\label{sec:model_definition}

\looseness=-1
Let $p_t=\operatorname{Law}(x_t)$ denote the marginals of a prescribed forward (corruption) process, induced by a stochastic interpolant
introduced in
\S\ref{sec:stochastic_interpolants}.
We seek to perform few-step transport along this probability path through
a shared family of two-time maps with tractable inverses and
exact change of variables at data time $t=0$. This motivates the construction of the following generative model, defined as:

\begin{nfmdefinition}[Normalizing Flow Map]
\label{def:nfm}
\looseness=-1
Let $\mathcal{T}\subseteq\{(s,t):0\leq t<s\leq 1\}$
be a set of admissible time pairs containing $(1,0)$.
A \emph{normalizing flow map} is a jointly
parameterized family
\begin{equation}
\label{eq:nfm_model_class}
    \mathcal{F}_{\theta}
    =
    \left\{
        f^\theta_{s,t}:\mathbb{R}^d\rightarrow\mathbb{R}^d
    \right\}_{(s,t)\in\mathcal{T}},
\end{equation}
\looseness=-1
where, for every admissible $\theta$ and $(s,t)\in\mathcal{T}$,
$f^\theta_{s,t}$ is a $\mathcal{C}^\infty$ diffeomorphism satisfying
\begin{equation}
\label{eq:nfm_time_invertibity}
    (f^\theta_{s,t})^{-1}\circ f^\theta_{s,t}
    =
    f^\theta_{s,t}\circ(f^\theta_{s,t})^{-1}
    =
    \mathrm{Id}.
\end{equation}

\looseness=-1
Additionally, given a source density $p_s$, the map induces the density
$p^\theta_{s\to t}:=(f^\theta_{s,t})_\#p_s$ through
\begin{equation}
    \label{eq:nfm_change_of_var}
    \log p^\theta_{s\to t}(x_t)
    =
    \log p_s(x_s)
    -
   \log\left|
        \det
        \frac{\partial f^\theta_{s,t}(x_s)}{\partial x_s}
    \right|,
    \qquad
    x_t=f^{\theta}_{s,t}(x_s),\,x_s\sim p_s.
\end{equation}
\end{nfmdefinition}

\looseness=-1
Each supported function $f_{s,t}(x_s)$ that maps $\;(s,t,x_s)\mapsto x_t$ provides an invertible transformation between intermediate times $s \to t$. Consequently, density evaluation at time $t$ is possible whenever the source density $p_s(x_s)$ is tractable.
Moreover, starting from the prior $p_1$,
compositions of admissible jumps support few-step sampling
and exact density evaluation through repeated application
of~\cref{eq:nfm_change_of_var}.
A map of particular interest is the one-step map which tractably admits the density $p^{\theta}_{1 \to 0}(x_0)$,
\begin{equation}
\label{eq:nfm_one_step}
    x_1\sim p_1(x_1),
    \qquad
    x_0=f^\theta_{1,0}(x_1),
    \qquad
    p^\theta_{1 \to 0}
    =
    (f^\theta_{1,0})_{\#}p_1.
\end{equation}
\looseness=-1
We note that~\cref{eq:nfm_one_step} provides one-step sample likelihoods. We now make the following remark.

\begin{nfmremark}
\label{rem:nfm-normalizing-flows}
\looseness=-1
Restricting \nameshort to $\mathcal{T}=\{(1,0)\}$ recovers
the structure of a normalizing flow.
\end{nfmremark}

\looseness=-1
\xhdr{Differences to modern NF MLE training}
To scale state-of-the-art NF's via MLE training~\citep{zhai2025normalizing,chen_iterative_2026} it is crucial to smooth out the data distribution $p_0$ via noise injection $p_{\sigma}$, i.e., $x_\sigma=x_0+\epsilon$, $\epsilon\sim\mathcal{N}(0,\sigma^2I)$, which prevents well-known training stability issues, e.g., poor performance, and \emph{numerical} loss of invertibility~\citep{behrmann2021understanding}. While the resulting flow remains invertible and provides exact likelihoods for $x_\sigma$, clean samples are recovered at inference time using Tweedie's posterior mean estimator at noise level $\sigma$,
$\hat{x}_0=\mathbb{E}[x_0\mid x_\sigma]=x_\sigma+\sigma^2\nabla_{x_\sigma}\log p^\theta_\sigma(x_\sigma)$~\citep{miyasawa1961empirical}.
This additional denoising transformation is not covered by the change-of-variables guarantee and thus likelihoods over clean samples are not available.

\begin{nfmremark}
\label{rem:tweedie}
NFs trained with noise-augmented MLE provide exact likelihoods at $p_\sigma$, whereas \nameshorts amortize transport across noise levels and provide clean sample likelihoods $p_0$.
\end{nfmremark}

\looseness=-1 
We provide further details of practical training instability of MLE without noise injection in~\S\ref{app:likelihood_tweedie}.

\looseness=-1
\xhdr{\nameshort parametrizations}
We construct \nameshorts from a base NF architecture $\Phi_{s,t}^{\theta}$ and use it to instantiate three distinct \nameshort parametrizations $(\text{A}) - (\text{C})$. The simplest parametrization is the \emph{dual-time \nameshort}, which directly conditions the flow on the source and target times,
$
(\text{A}): f^\theta_{s,t}(x_s)=\Phi_{s,t}^\theta(x_s, s,t).
$ 
This construction is invertible for every admissible pair $(s,t)$, but does not enforce the flow map identity $f^\theta_{s,s}(x_s)=x_s$.
When the limit $t\to s$ is needed, as in pointwise matching, the identity can be incorporated into the parametrization. To this end, we use $\Phi^\theta_{s,t}$ to parametrize a two-time denoiser $\delta^\theta_{s,t}$, corresponding to a generalized form of $x$-prediction~\citep{li2026basicsletdenoisinggenerative,lee2026flowmaplanguagemodels,lu2026onestep},
\begin{equation}
\label{eq:nfm_denoiser_param}
(\text{B}): \,  \delta^\theta_{s,t}(x)
\coloneqq
\Phi_{s,t}^\theta(x,s,t),
\quad
v^\theta_{s,t}(x)
\coloneqq
\frac{x-\delta^\theta_{s,t}(x)}{s},
\quad
f^\theta_{s,t}(x_s)
\coloneqq
\frac{t}{s}x+\frac{s-t}{s}\delta^\theta_{s,t}(x).
\end{equation}
\looseness=-1
This \emph{denoiser NFM} satisfies $f^\theta_{s,s}(x_s)=x_s$ by construction, while its diagonal velocity $v^\theta_{s,s}$ can be trained to match the instantaneous velocity of the probability flow. Moreover, at $t=0$ we have $f^\theta_{t,0}=\delta^\theta_{t,0}$, ensuring that the endpoint map, i.e., from any time $t\to0$, remains solely an invertible NF. Finally, for base NFs composed of affine transformations, such as TARFlow (cf.~\S\ref{subsec:NF}), we can enforce the diagonal identity by gating the scale and shift functions with $(s-t)$:
\begin{equation}
    (\text{C}):
    f_k^{\theta,m}(x, s, t)
    =
    \left(x_k-(s-t)\mu_k^{\theta,m}(x_{<k}, s, t)\right)
    \odot
    \exp\left(-(s-t)\alpha_k^{\theta,m}(x_{<k}, s, t)\right),
\end{equation}
Interleaving these transformations with permutation layers yields the \emph{gated dual-time} \nameshort.

\looseness=-1

\subsection{Pointwise and Distributional Matching}
\label{sec:model_computational_properties}

\looseness=-1
An \nameshort describes generative transport with respect to a target stochastic interpolant, through both its
action on particles and its induced pushforward density. At optimality, the particle and distributional perspectives must be \emph{consistent} in describing the same stochastic interpolant, yielding two complementary matching laws that describe transport along an \nameshort. We now develop these pointwise and distributional matching principles which, in turn, result in self-distillation losses in \S\ref{sec:distillation_losses}.

\looseness=-1
\xhdr{Pointwise matching laws}
At the particle level, an \nameshort resembles a conventional flow map operator and jumps along the stochastic interpolant $I_t$, with a target conditional velocity field, $v^*_t(x_t|x_0, x_1)$, as defined in~\Cref{eq:flow_matching}. In particular, for the time intervals $\gT$, an \nameshort $f^{\theta}_{s,t}$ parameterizes a conventional flow map restricted to admissible invertible maps. In pointwise matching, we seek to recover the particle correspondence specified by the probability-flow ODE. This is achieved via losses based on the flow map consistency conditions (\Cref{eq:flow_map_conditions}), which often involve the average and instantaneous velocities of the flow map, both of which can be recovered from the \nameshort parametrization.

\looseness=-1
\xhdr{Distributional matching laws}
At the distributional level, we aim to learn the pushforward distribution to match the \emph{marginal law} of the forward process $p_t(x_t)$. Specifically, given a time pair $(s,t) \in \gT$, the \nameshort induced pushforward distribution is required to match 
$p^\theta_{s\to t}:=(f^\theta_{s,t})_\#q_s$.
 Exploiting the invertibility of \nameshort, we can write equivalent characterizations as the pushforward or pullback; in other words, we have that $ \left(f^\theta_{s,t}\right)_\#p_s=p_t
    \Longleftrightarrow
    \left(\left(f^\theta_{s,t}\right)^{-1}\right)_\#p_t=p_s.$
This law can be quantified under the KL divergence, $\KL\left(
        p_t \,\middle\|\, p^\theta_{s\to t}
    \right)$.
Thus, matching the generated distribution to the destination marginal
is equivalent to matching the inverse-transformed destination
distribution to the source marginal.
When $s=1$, the tractable reference density $p_1$ enables exact
evaluation of $p^\theta_{1\to t}$ through the change of variables afforded by an \nameshort at intermediate times instead of optimizing for a weighted Evidence Lower Bound (ELBO) as in diffusion models~\citep{kingma2021variational}.

\subsection{Distillation Losses}
\label{sec:distillation_losses}

\looseness=-1
The matching laws developed in \S\ref{sec:model_computational_properties} yield two complementary perspectives on an \nameshort. We next convert these into training frameworks, resulting in two model variants \nameshort-Regression (\nfmr) and \nameshort-Distribution (\nfmd) that exploit pointwise and distributional laws, respectively.

\xhdr{\nameshort Regression}
\looseness=-1
To design \nfmr, we restrict our focus to the denoiser parametrization of an \nameshort~\cref{eq:nfm_denoiser_param}.
Furthermore, we instantiate a variant of the Eulerian consistency with conditional velocity $v^*_t(x_t | x_0, x_1)$ and the stopgradient operator ($\mathrm{sg}$) as introduced in improved Mean Flow (iMF) and Pixel MeanFlow (pMF)~\citep{geng2026improved,lu2026onestep},
\begin{equation*}
\label{eq:imf}
    \mathcal{L}_{\rm R}=\mathbb{E}_{(s,t) \sim \nu(s,t) ,(x_0,x_1) \sim \pi}\left[\left\lVert v^*_t +\mathrm{sg}((t-s)(\partial_s v^{\theta}_{s,t}(x_s)+\nabla v^{\theta}_{s,t}(x_s)v^{\theta}_{s,s}(x_s))) -v^{\theta}_{s,t}(x_s)\right\rVert^{2}\right].
\end{equation*}

\looseness=-1 We highlight that the loss $\gL_{\rm R}$ implicitly trains the one-step \nameshort map $f_{s,0}^\theta$ through intermediate time regression. Although the density of multi-step samples is inaccessible, \nfmr natively supports multi-step sampling. Crucially, inference can be performed using the one-step map $f^{\theta}_{1 \to 0}$, which enjoys exact sample likelihoods as it corresponds to inverting the NF backbone $\delta^{\theta}_{1 \to 0}(x_1) = \Phi^{\theta}_{1,0}(x_1)$. Additionally, as proven in \S\ref{app:imf_deriv_and_eulerian_connection}, $\mathcal{L}_{\mathrm R}$ directly controls the $2$-Wasserstein distance between $p_0$ and $p^{\theta}_{1\to 0}$, providing endpoint control for pointwise distillation.

\xhdr{\nameshort Distribution}
\label{sec:nfmd}
\looseness=-1
\nfmd exploits distributional matching across interpolation times to learn the marginal probability path. We first construct a single source time objective, and then provide its extension to intermediate source times.
Specifically, we fix $s=1$ and consider the restricted time pair set
$\gT_{\mathrm S}:=\{(s,t)\in\gT\mid s=1\}\subset\gT$. We seek to minimize the time-averaged marginal KL discrepancy $\E_{\gT_{\mathrm S}} \left[\KL\left(p_t \,\middle\|\, p^\theta_{1\to t} \right)\right]$. We exploit conditioning to circumvent the generally intractable marginal $p_t$, and introduce the conditional single-time objective, akin to SoftFlow~\citep{kim2020softflow},
\begin{equation}
\label{eq:nfmd_single_conditional}
\mathcal{L}_{\mathrm{D}}
\!\coloneqq\!
\mathbb{E}_{t,x_0}
\KL\left(
p_t(\cdot\mid x_0)
\middle\|
p^\theta_{1\to t}
\right)
\!\overset{\mathrm c}{=}\!
\mathbb{E}_{t,x_t}\!
\left[
\frac{1}{2}\left\|(f^\theta_{1,t})^{-1}(x_t)\right\|^2
\!-\!\log\left|
\det
\frac{\partial (f^\theta_{1,t})^{-1}(x_t)}
     {\partial x_t}
\right|
\right],
\end{equation}
\looseness=-1
where $\overset{\mathrm c}{=}$ denotes equality up to an additive constant independent of $\theta$. As stated in the following proposition, conditional supervision upper-bounds the marginal objective while preserving its minimizers, allowing us to learn the marginal pushforward without evaluating $p_t$.
\begin{mdframed}[style=MyFrame2]
\begin{restatable}[Conditional \nfmd controls marginal \nfmd]
{proposition}{conditionalnfmd}
\label{conditional_bound}
\looseness=-1
Assume that
$\mathbb{E}_t[\mathcal{I}(x_0;x_t)]<\infty$, where $\mathcal{I}$ denotes mutual information. Then, for every $\theta$,
\[
\mathcal{L}_{\mathrm{D}}
=
\E_{\gT_{\mathrm S}} \left[\KL\left(p_t \,\middle\|\, p^\theta_{1\to t} \right)\right]
+
\mathbb{E}_{\gT_{\mathrm S}}\left[\mathcal{I}(x_0;x_t)\right]
\geq
\E_{\gT_{\mathrm S}} \left[\KL\left(p_t \,\middle\|\, p^\theta_{1\to t} \right)\right].
\]
Consequently, the conditional and marginal objectives share the same minimizers,
\[
\operatorname*{arg\,min}_\theta
\mathcal{L}_{\mathrm{D}}(\theta)
=
\operatorname*{arg\,min}_\theta
\E_{\gT_{\mathrm S}} \left[\KL\left(p_t \,\middle\|\, p^\theta_{1\to t} \right)\right].
\]
\end{restatable}
\end{mdframed}
\looseness=-1
The gap between the conditional and marginal objectives is the mutual information, which depends only on the prescribed corruption process. Separately, empirical conditional \nfmd is bounded below (\S\ref{app:theory}), avoiding the unbounded likelihood degeneracy that can arise in empirical MLE. Finally, we also relate accuracy along the corruption path to the quality of the generated distribution.
\newpage
\begin{mdframed}[style=MyFrame2]
\begin{restatable}[Single-time \nfmd endpoint control]
{proposition}{nfmdendpointcontrol}
\label{prop3}
\looseness=-1
Assume $t\mapsto \KL\left(p_t \,\middle\|\, p^\theta_{1\to t} \right)$ is locally Lipschitz at $t=0$, with constant $C$ for $t\in[0,t_0]$.
Suppose the time-sampling density satisfies
$\nu(t)\geq\nu_{\min}>0$ on $[0,t_0]$. If
$\E_{\gT_{\mathrm S}}\left[\KL\left(p_t \,\middle\|\, p^\theta_{1\to t} \right)\right]\leq C\nu_{\min}t_0^2/2$, then
\[
\KL\left(
p_0
\,\middle\|\,
p^\theta_{1\to0}
\right)
\leq
\sqrt{
\frac{2C}{\nu_{\min}}
\E_{\gT_{\mathrm S}}\left[\KL\left(p_t \,\middle\|\, p^\theta_{1\to t} \right)\right]
}.
\]
\end{restatable}
\end{mdframed}
\looseness=-1
Thus, minimizing the single-time \nfmd objective provides direct control of the generated endpoint distribution, provided the marginal discrepancy is locally regular. We provide the proof in \S\ref{app:theory}.

\looseness=-1
\xhdr{Two-time extension}
The single-time objective learns transport from the tractable reference $p_1$ to each marginal $p_t$. Extending likelihood-based training to intermediate source times requires evaluating $p_s$, which is generally
intractable. We therefore use the single-time model as a tractable surrogate source density,
$\bar p_s:=p^{\bar\theta}_{1\to s}$, with frozen teacher parameters $\bar\theta$. The resulting surrogate pushforward is
$p_{s\to t}^{\theta,\bar\theta}
:=(f_{s,t}^{\theta})_\#\bar p_s$.
Its density can be evaluated exactly by mapping $x_t$ back to the intermediate source and then to the Gaussian reference, yielding the objective (see \S\ref{app:theory}):
\begin{equation}
\label{eq:nfmd_two_time_loss}
\mathcal{L}_{\rm D,2}
=
\mathbb{E}_{s,t,x_t}
\left[
\frac{1}{2}
\left\|\Phi(x_t)\right\|^2
-
\left.
\log\left|
\det
\frac{\partial (f_{1,s}^{\bar\theta})^{-1}({x}_s)}
     {\partial {x}_s}
\right|_{x_s=\Tilde{x}_s}
\right|
-
\log\left|
\det
\frac{\partial (f_{s,t}^{\theta})^{-1}(x_t)}
     {\partial x_t}
\right|
\right],
\end{equation}
\looseness=-1 where
$\tilde{x}_s=(f_{s,t}^{\theta})^{-1}(x_t)$ and
$\Phi:=(f_{1,s}^{\bar\theta})^{-1}\circ(f_{s,t}^{\theta})^{-1}$.
When $\bar p_s=p_s$, this objective equals
$\mathbb{E}_{s,t}\KL(p_t\|p^\theta_{s\to t})$,
up to an additive constant independent of $\theta$.
Otherwise, it matches $p_t$ using the teacher's approximate source
density. Although the teacher parameters are frozen, gradients
through its input are retained, allowing both the reference-density
and Jacobian terms to train $f^\theta_{s,t}$.
In practice, the single- and two-time maps share the same
architecture, and we jointly optimize
$\mathcal{L}_{\mathrm D}+\lambda\mathcal{L}_{\mathrm{D},2}$.
While one-step sampling already achieves strong sample quality
(\S~\ref{sec:image_experiments}), few-step sampling allows the strength
and timing of classifier-free guidance to be adjusted throughout
generation, as in diffusion models (\S~\ref{app:image_generation}).
Access to tractable intermediate densities also opens the door to
likelihood-based inference-time steering methods, such as adaptations of
twisted sequential Monte Carlo.

\subsection{Post-training \namelong}
\label{subsec:posttraining}

\looseness=-1
A key advantage of \nameshort is tractable access to the endpoint density $p_{1\to 0}^\theta(x_0)$. This makes distribution-level post-training objectives directly available:
\begin{equation}
\label{eq:general_rl}
\mathcal{J}(\theta)
= \mathbb{E}_{x_0\sim p_{1\to0}^\theta}[R(x_0)]
+ \beta\,\operatorname{Reg}(p_{1\to 0}^\theta),\, \beta>0.
\end{equation}
\looseness=-1
Here $R:\mathbb{R}^d\to\mathbb{R}$ is a scalar reward function , and $\operatorname{Reg}(p_{1\to 0}^\theta)$ is either the entropy $\mathcal{H}(p_{1\to 0}^\theta)$, as in our maximum-entropy experiments, or $-\KL(p_{1\to 0}^\theta\|p_{\rm ref})$ for a fixed reference density. 

For diffusion and flow matching models, endpoint likelihoods are expensive to evaluate, making~\Cref{eq:general_rl} challenging to optimize directly. Post-training must therefore leverage alternative learning signals, e.g., transition-level policy updates \citep{black2024training,liu2025flow}, adjoint or value-gradient methods~\citep{domingo2025adjoint,liu2025value}, forward process matching~\citep{zheng2026diffusionnft}, and ELBO-based proxies \citep{wallace2024diffusion,mcallister2026flow}.  \nameshorts avoid the complications introduced by these methods as endpoint log densities give direct access to score-function gradients~\citep{williams_simple_1992}, which can be used with non-differentiable rewards, and they provide the density ratios employed by prominent large language model post-training methods like DPO~\citep{rafailov2023direct} and GRPO~\citep{shao2024deepseekmathpushinglimitsmathematical}.

\section{Experiments}
\label{sec:experiments}

\looseness=-1
We empirically validate \nameshort on image generation and RL post-training, as well as on a variety of molecular conformation sampling tasks.
We use the denoiser parametrization (B) for \nfmr, the dual-time parametrization (A) for one-step \nfmd and gated dual-time parametrization (C) for few-step \nfmd (\S~\ref{sec:model_definition}). Both methods use iTARFlow~\citep{lu2026bidirectional} as the base invertible
architecture $\Phi_{s,t}^{\theta}$ (see \S~\ref{app:architecture}),
which also serves as our principal baseline when trained with MLE. Additional ablations and experimental details are reported in \S~\ref{app:results}.

\subsection{Image generation and Post-training}
\label{sec:image_experiments}
\vspace{-5pt}

\begin{figure}[h]
    \centering
    \begin{subfigure}[t]{0.24\textwidth}
        \headlinepanel{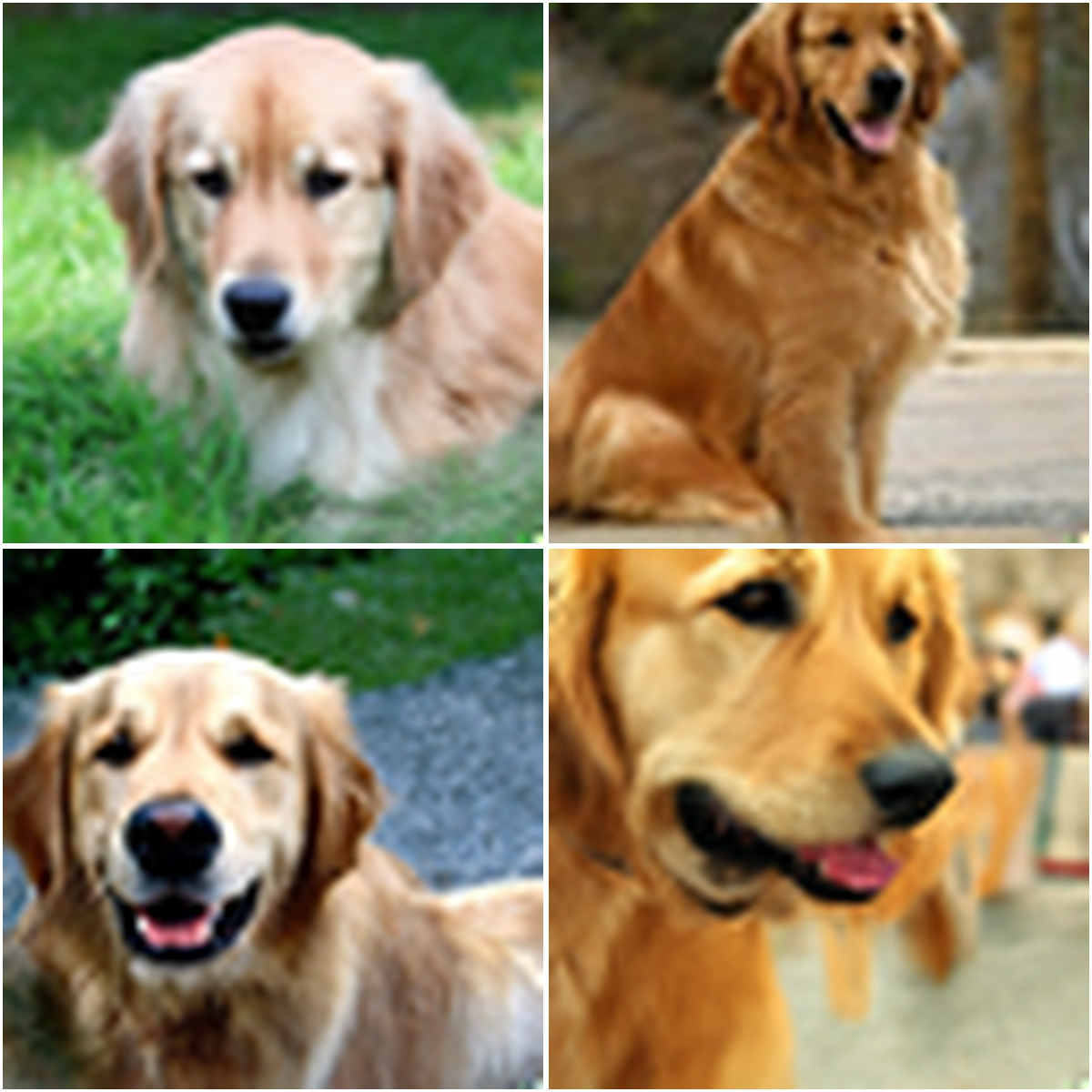}{(a) NFM-D}
        \phantomsubcaption\label{fig:nfmd_retriever}
    \end{subfigure}\hfill
    \begin{subfigure}[t]{0.24\textwidth}
        \headlinepanel{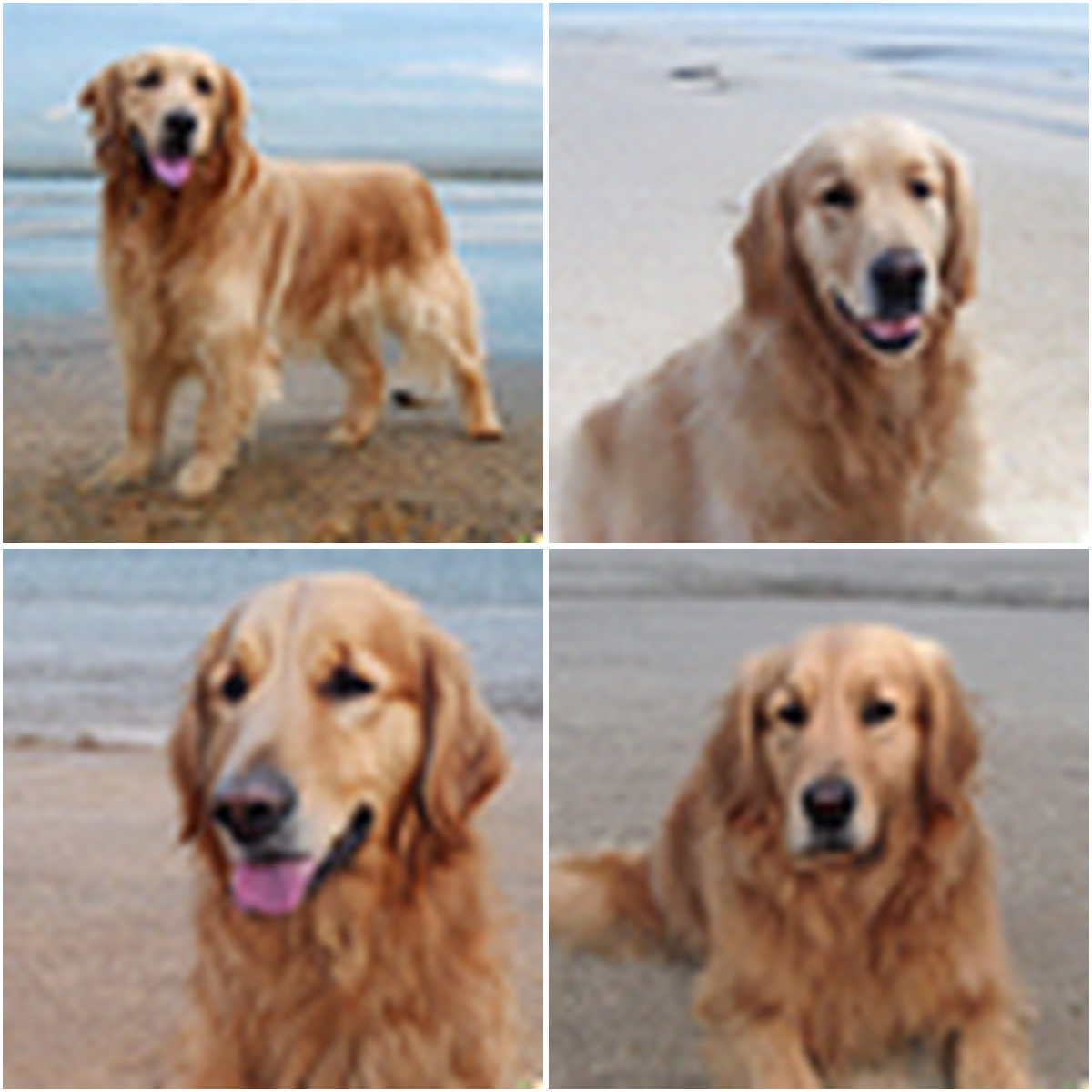}{(b) NFM-D + RL}
        \phantomsubcaption\label{fig:nfmd_retriever_rl}
    \end{subfigure}\hfill
    \begin{subfigure}[t]{0.24\textwidth}
        \headlinepanel{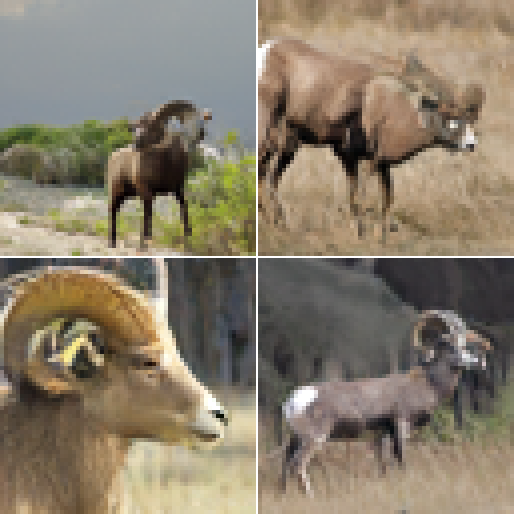}{(c) NFM-R}
        \phantomsubcaption\label{fig:nfmr_sheep}
    \end{subfigure}\hfill
    \begin{subfigure}[t]{0.24\textwidth}
        \headlinepanel{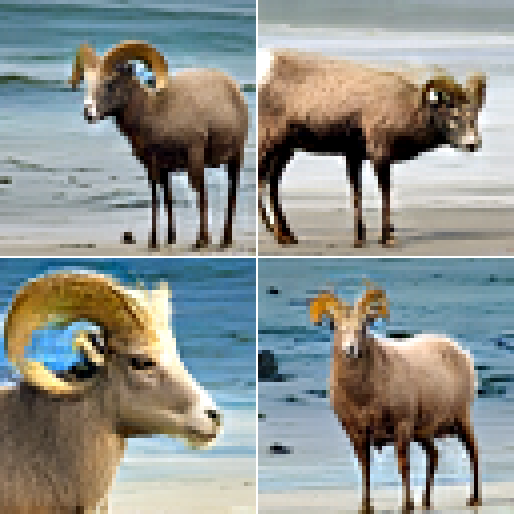}{(d) NFM-R + RL}
        \phantomsubcaption\label{fig:nfmr_sheep_rl}
    \end{subfigure}
    \vspace{-15pt}
    \caption{\small \looseness=-1 \nameshort post-training.
    (a) NFM-D, \emph{golden retriever}; (b) NFM-D post-trained, \emph{a single golden retriever on a sandy beach};
    (c) NFM-R, \emph{bighorn sheep}; (d) NFM-R post-trained, \emph{bighorn sheep standing on a sandy beach beside the sea}.
    }
    \label{fig:before_and_after_rl}
    \vspace{-5pt}
\end{figure}

\looseness=-1
\textbf{Image generation.}
We validate \nfmr and \nfmd on conditional image generation on \cifar~\citep{krizhevsky2009learning} and \imnet~\citep{deng2009imagenet} both in terms of Fr\'echet Inception Distance (FID)~\citep{heusel2017gans} and Feature Likelihood Divergence (FLD)~\citep{jiralerspong2023feature} (Tables~\ref{tab:cifar_10_results}~\&~\ref{tab:imagenet64_results}). We find that \nfmd and \nfmr improve both metrics over MLE trained iTARFlow, and \nfmd outperforms SoftFlow on \cifar. Moreover, \nfmr achieves the best overall performance. Maintaining invertibility with \nfmr requires noise annealing during training and \texttt{FP64} inference.
Additionally, multi-step sampling at inference time improves sample quality. For \nfmr, this comes at the cost of invertibility, while for \nfmd invertibility is preserved across steps. Finally, training iTARFlow on clean data preserves invertibility but substantially worsens sample quality. We report additional experiments and ablations in \S\ref{app:results}.

\begin{table}[h]
\centering
\caption{Class-conditional generation results. Best results are in \textbf{bold}, and second-best are \underline{underlined}.}
\vspace{-5pt}
\begin{subtable}[t]{0.4565\textwidth}
    \centering
    \small
    \setlength{\tabcolsep}{2.5pt}
    \renewcommand{\arraystretch}{1.05}

    \caption{\small \cifar results.}
    \label{tab:cifar_10_results}
    \vspace{-5pt}
    \resizebox{\linewidth}{!}{%
    \begin{tabular}{lcc}
        \toprule
        \textbf{Model}
        & \textbf{FID} $\downarrow$
        & \textbf{FLD} $\downarrow$ \\
        \midrule

        \rowcolor{gray!15}
        \textit{\textbf{Non-Invertible NFs}} & & \\

        iTARFlow
        & 6.55 & 8.38 \\

        \textbf{\nfmr (Ours)}
        & \textbf{3.87} & \underline{7.00} \\

        \textbf{\nfmr (Ours)} \emph{noise anneal, 2 steps}
        & \underline{4.63} & \textbf{6.89} \\

        \rowcolor{gray!15}
        \textit{\textbf{Invertible NFs}} & & \\

        iTARFlow \emph{no noise}
        & 40.83 & 23.21 \\

        SoftFlow
        & 6.76 & 7.84 \\

        \textbf{\nfmd (Ours)}
        & 6.26 & \underline{7.70} \\

        \textbf{\nfmd (Ours)} \emph{2 steps}
        & \underline{6.05} &  7.93 \\

        \textbf{\nfmr (Ours)} \emph{noise anneal}
        & \textbf{5.17} & \textbf{7.62} \\

        \bottomrule
    \end{tabular}%
    }
\end{subtable}%
\hfill
\begin{subtable}[t]{0.5085\textwidth}
\centering
\small
\setlength{\tabcolsep}{3.5pt}
\renewcommand{\arraystretch}{1.05}

\caption{\small \looseness=-1 \imnet results.}
\vspace{-5pt}
\label{tab:imagenet64_results}

\resizebox{\linewidth}{!}{%
\begin{tabular}{lccc}
    \toprule
    \textbf{Model}
    & \textbf{Param.}
    & \textbf{FID} $\downarrow$
    & \textbf{FLD} $\downarrow$ \\
    \midrule

    \rowcolor{gray!15}
    \textit{\textbf{Diffusion \& Few Step Models}} & & & \\

    \textcolor{refgray}{iDDPM~\citep{pmlr-v139-nichol21a}}
    & \textcolor{refgray}{270M}
    & \textcolor{refgray}{2.92}
    & \textcolor{refgray}{--} \\

    \textcolor{refgray}{EDM~\citep{karras2022elucidating}}
    & \textcolor{refgray}{296M}
    & \textcolor{refgray}{1.36}
    & \textcolor{refgray}{--} \\

    \textcolor{refgray}{CD~\citep{song2023consistency}}
    & \textcolor{refgray}{296M}
    & \textcolor{refgray}{4.70}
    & \textcolor{refgray}{--} \\

    \textcolor{refgray}{Iterative TARFlow~\citep{chen_iterative_2026}}
    & \textcolor{refgray}{350M}
    & \textcolor{refgray}{2.05}
    & \textcolor{refgray}{--} \\

    \rowcolor{gray!15}
    \textit{\textbf{Non-Invertible NFs}} & & & \\

    \textcolor{refgray}{TARFlow\footnotemark~\citep{zhai2025normalizing}}
    & \textcolor{refgray}{466M}
    & \textcolor{refgray}{2.90}
    & \textcolor{refgray}{--} \\

    iTARFlow
    & 122M
    & \underline{4.89}
    & \underline{6.98} \\

    \textbf{\nfmr (Ours)}
    & 127M
    & \textbf{3.60}
    & \textbf{3.15} \\

    \rowcolor{gray!15}
    \textit{\textbf{Invertible NFs}} & & & \\

    \textbf{\nfmd (Ours)}
    & 123M
    & \textbf{4.59}
    & \textbf{6.78} \\

    \bottomrule
\end{tabular}%
}

\end{subtable}
\end{table}

\footnotetext{\looseness=-1 We were unable to reproduce these results and could not find publicly available official checkpoints or successful independent reproductions. We leave this current number for transparency reasons.}

\begin{wrapfigure}{r}{0.46\columnwidth}
  \centering
  \vspace{-5pt}
  \includegraphics[width=\linewidth]{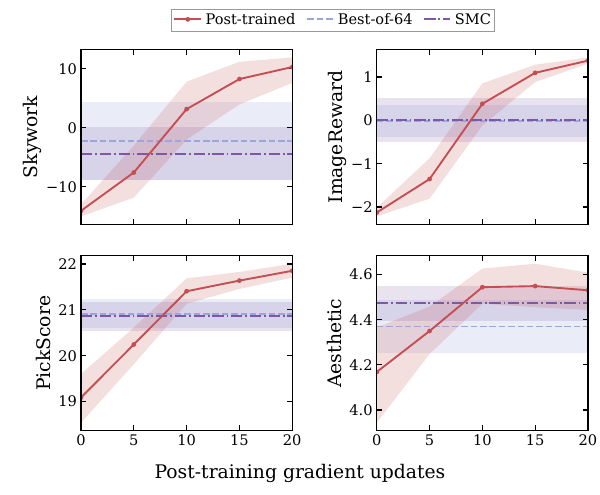}
  \vspace{-20pt}
  \caption{\small \looseness=-1 Post-trained \nfmd vs. best-of-$64$ and SMC baselines.}
  \vspace{-5pt}
  \label{fig:nfmd_b2_dog_metrics}
\end{wrapfigure}
\looseness=-1 \noindent\xhdr{RL Post-training}
\nameshort models provide access to exact sample likelihoods, enabling RL-based post-training. Here we post-train  \emph{purely class-conditional} \nameshort models to align their samples with a text-caption supplied to a Visual Language Model (VLM) via on-policy RL. Specifically, we optimize~\cref{eq:general_rl}; \nfmr uses maximum entropy regularization~\citep{jaynes1957,pmlr-v80-haarnoja18b} and pathwise gradients, while we train \nfmd with KL-regularized REINFORCE~\citep{williams_simple_1992}. We use a suite of rewards for training and evaluation: ImageReward~\citep{imagereward}, PickScore~\citep{pickscore}, HPSV2.1~\citep{wu2023human}, and Skywork-VL~\citep{wang2025skyworkvlrewardeffectivereward} (See \S~\ref{app:rl_finetuning} for more details). \Cref{fig:before_and_after_rl} clearly demonstrates the effect of RL fine-tuning on the models' output, and in~\Cref{fig:nfmd_b2_dog_metrics} we observe clear gains on multiple rewards against best-of-$N$ and sequential Monte Carlo (SMC) baselines over the course of training. We report additional experiments, including on-policy RL post-training on the \cifar \nameshort models with generated samples in \S\ref{app:results}.

\subsection{Boltzmann Generators}
\label{sec:exp_bgs}
\looseness=-1
We now apply \nameshorts to equilibrium sampling with Boltzmann Generators~\citep{noe2019boltzmann}~(BGs), which are designed to sample molecular states at thermodynamic equilibrium. BGs aim to sample from the \emph{Boltzmann distribution}: $\mu_{\text{target}}(x) \propto \exp\left(-\mathcal{E}(x)/k_{\text{B}}T\right)$, where $\mathcal{E}(x)$ is the potential energy of a conformation $x \in \mathbb{R}^{n \times 3}$ of a molecular system with $n$ atoms, and $k_{\text{B}}T$ denotes the thermal energy at temperature $T$. Observables $\phi(x)$ under $\mu_{\text{target}}$ are commonly 
estimated using self-normalized importance sampling (IS)~\citep{owen2013monte} and sequential Monte Carlo (SMC)~\citep{tan2025scalable}. We refer to \S\ref{app:BGs} for a detailed discussion on BGs, including the baselines we consider.

\begin{figure}[h]
    \centering
    \vspace{-10pt}
    \includegraphics[width=1\linewidth]{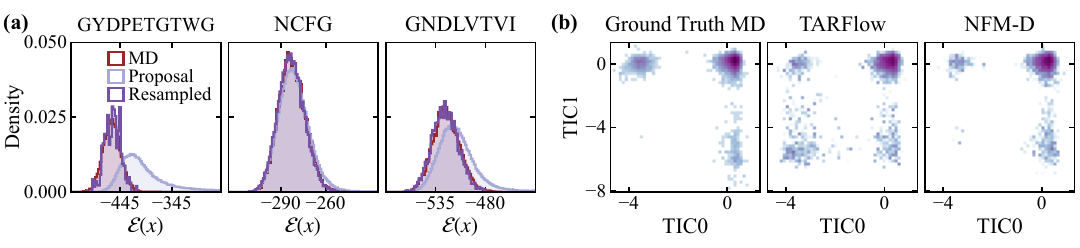}
    \vspace{-20pt}
    \caption{\small \looseness=-1 \textbf{a)} Energy histograms of MD data, the \nfmd proposal, and the resampled distribution, for a single-system model (left) and a transferable model evaluated on a $4$- and $8$-residue peptides (middle, right). \textbf{b)} TICA projections for the ground truth MD, TARFlow, and \nfmd on NHQYGSDP.}
    \label{fig:energy_distributoions}
    \vspace{-5pt}
\end{figure}

\looseness=-1
\xhdr{Metrics}
We evaluate BGs on single-peptide systems using the $2$-Wasserstein distance between energy distributions ($\mathcal{E}$-$\mathcal{W}_2$) and between dihedral-angle distributions ($\mathbb{T}$-$\mathcal{W}_2$). Transferable models are additionally evaluated with TICA-$\mathcal{W}_2$ on the first two TICA components. Following \citet{rehman2026autoregressive}, we omit effective sample size. The MLE counterpart of \nameshort is SBG~\citep{tan2025scalable}. We report \nfmd throughout, as it outperforms \nfmr. Further details in \S\ref{app:BGs_metrics} and \S\ref{app:results_BGs}.

\begin{wraptable}[12]{r}{0.55\linewidth}
    \vspace{-12pt}
    \centering
    \small
    \setlength{\tabcolsep}{2.5pt}
    \caption{\small \looseness=-1 Performance comparison on Chignolin. SBG-iTARFlow refers to our reimplementation of SBG, used to control for architectural, optimizer, and training-recipe changes. Best results are in \textbf{bold}; second-best results are \underline{underlined}.}
    \label{tab:chig_results}
    \vspace{-8pt}
    \begin{tabular}{@{}lcc@{}}
        \toprule
        Algorithm
        & $\mathcal{E}$-$\mathcal{W}_2 \downarrow$
        & $\mathbb{T}$-$\mathcal{W}_2 \downarrow$ \\
        \midrule
        ArBG~\citep{rehman2026autoregressive} & $7.94 \pm 1.05$ & $\underline{2.78 \pm 0.11}$ \\
        \midrule
        SBG SMC~\citep{tan2025scalable} & $10.82 \pm 7.21$ & $3.78 \pm 0.44$ \\
        SBG-iTARFlow SMC & $10.87 \pm 10.18$ & $3.56 \pm 0.92$ \\
        \textbf{\nfmd IS (Ours)} & $\underline{5.12 \pm 0.26}$ & $\mathbf{2.71 \pm 0.05}$ \\
        \textbf{\nfmd SMC (Ours)} & $\mathbf{3.42 \pm 0.27}$ & $2.81 \pm 0.12$ \\
        \bottomrule
    \end{tabular}
\end{wraptable}

\looseness=-1
\vspace{-0.0em}\noindent\textbf{Single-peptide Systems}
We evaluate \nameshort on equilibrium conformation sampling for Chignolin 
(GYDPETGTWG, 10 residues), with results reported in Table~\ref{tab:chig_results}. Figure~\ref{fig:energy_distributoions} shows the corresponding energy distributions. Chignolin is the most challenging single-peptide benchmark due to the presence of the $\beta$-hairpin secondary structure. 
\nfmd substantially outperforms SBG, and achieves performance competitive with the autoregressive ArBG baseline. To our knowledge, no other continuous flow baseline has reported comparable performance on Chignolin.

\begin{table*}[h]
\centering

\small
\setlength{\tabcolsep}{3pt}
\renewcommand{\arraystretch}{1.05}
\caption{
\small \looseness=-1 Transferable models comparison on unseen 4- and 8-residue peptides,
using a budget of $10^4$ energy evaluations and $T=1$.
Best results are in \textbf{bold}; second-best results are \underline{underlined}. 
}
\label{tab:transferable_final}
\vspace{-5pt}
\begin{tabular}{lccc ccc}
\toprule
& \multicolumn{3}{c}{4AA (30 systems)}
& \multicolumn{3}{c}{8AA (30 systems)} \\
\cmidrule(lr){2-4}
\cmidrule(lr){5-7}

Model
& $\mathcal{E}$-$\mathcal{W}_2 \downarrow$
& $\mathbb{T}$-$\mathcal{W}_2 \downarrow$
& TICA-$\mathcal{W}_2 \downarrow$
& $\mathcal{E}$-$\mathcal{W}_2 \downarrow$
& $\mathbb{T}$-$\mathcal{W}_2 \downarrow$
& TICA-$\mathcal{W}_2 \downarrow$ \\
\midrule

\textcolor{refgray}{TimeWarp~\citep{klein2023timewarp}}
& \textcolor{refgray}{$7.237$}
& \textcolor{refgray}{$2.204$}
& \textcolor{refgray}{$0.993$}
& \textcolor{refgray}{---}
& \textcolor{refgray}{---}
& \textcolor{refgray}{---}
\\
\textcolor{refgray}{BioEmu~\citep{lewis2025scalable}}
& \textcolor{refgray}{$90.079$}
& \textcolor{refgray}{$2.037$}
& \textcolor{refgray}{$1.479$}
& \textcolor{refgray}{$193.873$}
& \textcolor{refgray}{$4.638$}
& \textcolor{refgray}{$1.601$}
\\
\textcolor{refgray}{UniSim~\citep{yu2025unisim}}
& \textcolor{refgray}{$>10^4$}
& \textcolor{refgray}{$2.766$}
& \textcolor{refgray}{$1.733$}
& \textcolor{refgray}{$>10^3$}
& \textcolor{refgray}{$6.156$}
& \textcolor{refgray}{$1.495$}
\\
\midrule
ECNF++~\citep{tan2025amortized}
& $10.032$
& $1.121$
& $0.572$
& ---
& ---
& ---
\\
Prose~\citep{tan2025amortized}
& $\underline{0.932}$
& $\underline{0.752}$
& $\mathbf{0.367}$
& $10.038$
& $2.456$
& $0.988$
\\
Robin~\citep{rehman2026autoregressive}
& $1.077$
& $0.900$
& $0.447$
& $\mathbf{6.373}$
& $\underline{2.440}$
& $\underline{0.986}$
\\
\midrule
TARFlow~\citep{tan2025amortized}
& $1.260$
& $0.924$
& $0.492$
& $11.298$
& $2.733$
& $1.087$
\\
\textbf{\nfmd (Ours)}
& $\mathbf{0.907}$
& $\mathbf{0.750}$
& $\underline{0.387}$
& $\underline{7.396}$
& $\mathbf{2.314}$
& $\mathbf{0.890}$\\
\bottomrule
\end{tabular}
\end{table*}

\begin{wraptable}{r}{0.65\textwidth}
\centering
\small
\setlength{\tabcolsep}{2.5pt}
\renewcommand{\arraystretch}{1.05}
\vspace{-10pt}
\caption{
\small \looseness=-1 Performance of transferable models with temperature tuning.
Best results are in \textbf{bold}; second-best results are
\underline{underlined}.
}
\label{tab:temperature}
\vspace{-5pt}
\resizebox{\linewidth}{!}{%
\begin{tabular}{lccc ccc}
\toprule
& \multicolumn{3}{c}{4AA (30 systems)}
& \multicolumn{3}{c}{8AA (30 systems)} \\
\cmidrule(lr){2-4}
\cmidrule(lr){5-7}

Model
& $\mathcal{E}$-$\mathcal{W}_2 \downarrow$
& $\mathbb{T}$-$\mathcal{W}_2 \downarrow$
& TICA-$\mathcal{W}_2 \downarrow$
& $\mathcal{E}$-$\mathcal{W}_2 \downarrow$
& $\mathbb{T}$-$\mathcal{W}_2 \downarrow$
& TICA-$\mathcal{W}_2 \downarrow$ \\
\midrule

Robin 
& $1.168$
& $0.886$
& $0.471$
& $\underline{4.251}$
& $2.325$
& $\underline{0.943}$ \\

Robin SMC
& $\underline{1.079}$
& $\underline{0.874}$
& $\underline{0.463}$
& $4.263$
& $\underline{2.315}$
& $0.977$ \\

\textbf{\nfmd (Ours)}
& $\mathbf{0.865}$
& $\mathbf{0.745}$
& $\mathbf{0.398}$
& $\mathbf{4.148}$
& $\mathbf{2.296}$
& $\mathbf{0.892}$ \\

\bottomrule
\end{tabular}%
}
\vspace{-10pt}
\end{wraptable}
\looseness=-1
\noindent\xhdr{Transferable generation}
We evaluate zero-shot generalization to unseen peptides~\citep{klein2024transferable} using a transferable
\nfmd trained on ManyPeptidesMD~\citep{tan2025amortized}, with results
reported in Table~\ref{tab:transferable_final}. \nfmd consistently improves
over its MLE-trained TARFlow counterpart, demonstrating the benefit of our training objective in the
transferable setting. \nfmd also outperforms Prose and Robin on five of six metrics. Finally \nfmd uses substantially fewer parameters than Prose ($118$M vs.\ $288$M). Figure~\ref{fig:energy_distributoions} shows a representative TICA projection,
illustrating the improved agreement of \nfmd with the ground truth
distribution. Figure~\ref{fig:energy_distributoions} shows additional energy distributions.

\looseness=-1
\noindent\xhdr{Temperature tuning} Table~\ref{tab:temperature} shows that temperature tuning can further improve
transferable performance. For \nfmd, we tune the temperature of the base
distribution, $p_{1,T} \coloneqq \mathcal{N}(0, TI)$, while Robin tunes its
attention temperature. With their respective tuned temperatures, \nfmd
achieves the best performance, including a substantial
improvement in $\mathcal{E}$-$\mathcal{W}_2$ on 8-residue peptides.

\section{Related Work}
\label{sec:related_work}

\looseness=-1
\xhdr{Normalizing flows}
NFs~\citep{tabak_density_2010,tabak_family_2013, rezende_variational_2015,dinh_density_2017} are invertible generative models whose \emph{exact} likelihood can be evaluated via the change-of-variables formula. Despite this appealing property, scaling to tasks like image generation remains challenging due
to numerical instabilities~\citep{xu2023embracing, andrade2024stable} and architectural constraints~\citep{chen2019residual}. Recently, TARFlows~\citep{zhai2025normalizing} have demonstrated impressive capabilities in image~\citep{gu2026starflow,shen_starflow2_2026} and video generation~\citep{gu2025starflowvideo}, as well as in sampling from Boltzmann distributions~\citep{midgley2023flow,tan2025scalable, tan2025amortized}.  While TARFlow is trained via noise augmented MLE~\citep{loaiza2023denoising, zhai2025normalizing}, alternative training strategies include training over a probability path~\citep{kim2020softflow, chen_iterative_2026, gu2026normalizing}, and direct regression-based training~\citep{regflow}.

\looseness=-1
\xhdr{Stochastic interpolants and flow maps}
Diffusion~\citep{sohl-dickstein_deep_2015, ho_denoising_2020, song2021scorebased} and Flow Matching~\citep{lipman2023flow,liu2023flow, albergo2023building,tong_improving_2024} models have achieved state-of-the-art sample fidelity across numerous modalities~\citep{rombach2022high, abramson2024accurate, zeni2025generative}. Generating samples and evaluating likelihoods with these models
typically require costly numerical integration~\citep{salimans2022progressive,chenneuralode, grathwohl2018ffjord}, motivating the development of few-step models~\citep{song2023consistency,song2024improved,frans2025one, geng2025mean, geng2026improved}, and the flow map framework \citep{boffi2025flow,boffi2025how,sabour2025align}. While flow maps accelerate sampling, likelihood evaluation remains costly~\citep{rehman2026falcon}. Alternatively, proxy likelihood evaluation can be amortized using an auxiliary divergence predictor \citep{ai2026joint,scallop_2026}.

\section{Conclusion}
\label{sec:conclusion}
\looseness=-1 In this paper, we introduce \namelong, a family of invertible transport maps that combine few-step sampling with exact change-of-variables densities for one-step endpoint samples. We demonstrate the application of these cheap likelihoods in Boltzmann generators through importance sampling, and straightforward RL-based post-training of image generation models. Furthermore, we demonstrate superior sample quality to existing NF methods and avoid the use of Tweedie denoising which requires spatial gradients and breaks invertibility guarantees. While \nfmr obtains the highest sample quality, it is prone to becoming numerically non-invertible, which impedes RL methods relying on the score-function or KL-regularization. This points to the importance of designing training schemes that enjoy the benefits of both pointwise and distributional consistency.

\subsection*{AI use statement}

In this work, we used generative AI tools to improve the clarity of the author-written text, and improve the correctness and presentation of mathematical claims; identify potentially relevant literature; assist with implementation of the method; and provide critical feedback on our research ideas and methodology. We did not use generative AI tools to draft sections of the paper or generate synthetic datasets.

The authors reviewed all AI-assisted work. We read the literature
suggested by AI tools and used separate LLM reviewers to suggest code
checks and tests. We take responsibility for the final content of this work, including text, claims or artifacts produced with the aid of generative AI.

\subsection*{Ethics statement}

Our method is a general purpose generative modeling approach. In the image
experiments, we post-train class-conditional generators using text--to--image
reward models. These rewards are imperfect proxies for human preferences:
optimizing them may amplify biases in the reward models or favor images
that score well without matching a user's intent. In the molecular
experiments, we evaluate equilibrium conformation sampling on existing
peptide benchmarks using molecular dynamics data. The molecules studied in this work
are considered safe and non-toxic. Applying the method to other molecular systems would require evaluation
specific to those uses.

\cut{
\subsection*{Reproducibility statement}

We provide the following information to support reproduction of our results:
\begin{itemize}[leftmargin=*, nosep]
    \item \textbf{Method and theory:} Section~\ref{sec:normalizing_flow_maps}
    defines the model, training objectives, and theoretical results.
    \item \textbf{Main experiments:} Section~\ref{sec:experiments} describes
    the image generation, post-training, and Boltzmann generator experiments
    and reports the main results.
    \item \textbf{Implementation and additional experiments:} Appendix~\ref{app:c} and ~\ref{app:results} carefully explains all the implementation details and additional ablations to properly implement \nameshort.
    \item \textbf{Code:} We will release the code in a public GitHub repository.
\end{itemize}
}

\subsubsection*{Acknowledgments}
We would like to thank Alec Sargood, Kriztina Sincovics, Georg Maierhofer,  Soran Ghaderi and Roman Poletukhin for helpful discussions about the content of this work. We acknowledge EuroHPC Joint Undertaking for awarding us access to MareNostrum5 at BSC, Spain and for awarding the project ID EHPC-AIF-2026SC03-183 access to Jupiter at JSC, Germany. The authors acknowledge the use of resources provided by the Isambard-AI National AI Research Resource (AIRR)~\citep{mcintoshsmith2024isambardai}. Isambard-AI is operated by the University of Bristol and is funded by the UK Government’s Department for Science, Innovation and Technology (DSIT) via UK Research and Innovation; and the Science and Technology Facilities Council [ST/AIRR/I-A-I/1023]. AITHYRA is supported by the Austrian Academy of Sciences and the not-for-profit Boehringer Ingelheim Stiftung. DR received financial support from the Natural Sciences and Engineering Research Council’s (NSERC) Banting Postdoctoral Fellowship under Funding Reference No. 198506. M.B. is partially supported by the EPSRC Turing AI World-Leading Research Fellowship No. EP/X040062/1 and EPSRC AI Hub No. EP/Y028872/1.

\clearpage
\bibliography{iclr2027_conference}

@inproceedings{meyer2021hutch++,
  title={Hutch++: Optimal stochastic trace estimation},
  author={Meyer, Raphael A and Musco, Cameron and Musco, Christopher and Woodruff, David P},
  booktitle={Symposium on Simplicity in Algorithms (SOSA)},
  pages={142--155},
  year={2021},
  organization={SIAM}
}

@article{boffi2025flow,
title={Flow map matching with stochastic interpolants: A mathematical framework for consistency models},
author={Nicholas Matthew Boffi and Michael Samuel Albergo and Eric Vanden-Eijnden},
journal={Transactions on Machine Learning Research},
issn={2835-8856},
year={2025},
url={https://openreview.net/forum?id=cqDH0e6ak2},
note={}
}

@inproceedings{boffi2025how,
title={How to build a consistency model: Learning flow maps via self-distillation},
author={Nicholas Matthew Boffi and Michael Samuel Albergo and Eric Vanden-Eijnden},
booktitle={The Thirty-ninth Annual Conference on Neural Information Processing Systems},
year={2025},
url={https://openreview.net/forum?id=Di5apl8HSH}
}

@inproceedings{ai2026joint,
  title={Joint distillation for fast likelihood evaluation and sampling in flow-based models},
  author={Ai, Xinyue and He, Yutong Kelly and Gu, Albert and Salakhutdinov, Russ and Kolter, Zico and Boffi, Nicholas and Simchowitz, Max},
  booktitle={International Conference on Learning Representations},
  volume={2026},
  pages={44940--44977},
  year={2026}
}

@inproceedings{
    geng2025mean,
    title={Mean Flows for One-step Generative Modeling},
    author={Zhengyang Geng and Mingyang Deng and Xingjian Bai and J Zico Kolter and Kaiming He},
    booktitle={The Thirty-ninth Annual Conference on Neural Information Processing Systems},
    year={2025},
    url={https://openreview.net/forum?id=uWj4s7rMnR}
}

@inproceedings{
    frans2025one,
    title={One Step Diffusion via Shortcut Models},
    author={Kevin Frans and Danijar Hafner and Sergey Levine and Pieter Abbeel},
    booktitle={The Thirteenth International Conference on Learning Representations},
    year={2025},
    url={https://openreview.net/forum?id=OlzB6LnXcS}
}

@inproceedings{
    rehman2026falcon,
    title={{FALCON}: Few-step Accurate Likelihoods for Continuous Flows},
    author={Danyal Rehman and Tara Akhound-Sadegh and Artem Gazizov and Yoshua Bengio and Alexander Tong},
    booktitle={The Fourteenth International Conference on Learning Representations},
    year={2026},
    url={https://openreview.net/forum?id=FbssShlI4N}
}

@inproceedings{regflow,
    title={Efficient Regression-based Training of Normalizing Flows for Boltzmann Generators},
    author={Danyal Rehman and Oscar Davis and Jiarui Lu and Jian Tang and Michael M. Bronstein and Yoshua Bengio and Alexander Tong and Joey Bose},
    booktitle={The Fourteenth International Conference on Learning Representations},
    year={2026},
    url={https://openreview.net/forum?id=ctdnzPxDI3}
}

@article{chen2019residual,
  title={Residual flows for invertible generative modeling},
  author={Chen, Ricky TQ and Behrmann, Jens and Duvenaud, David K and Jacobsen, J{\"o}rn-Henrik},
  journal={Advances in neural information processing systems},
  volume={32},
  year={2019}
}

@article{xu2023embracing,
  title={Embracing the chaos: analysis and diagnosis of numerical instability in variational flows},
  author={Xu, Zuheng and Campbell, Trevor},
  journal={Advances in Neural Information Processing Systems},
  volume={36},
  pages={32360--32386},
  year={2023}
}

@article{andrade2024stable,
  title={Stable training of normalizing flows for high-dimensional variational inference},
  author={Andrade, Daniel},
  journal={arXiv preprint arXiv:2402.16408},
  year={2024}
}

@InProceedings{tan2025scalable,
  title = 	 {Scalable Equilibrium Sampling with Sequential Boltzmann Generators},
  author =       {Tan, Charlie B. and Bose, Joey and Lin, Chen and Klein, Leon and Bronstein, Michael M. and Tong, Alexander},
  booktitle = 	 {Proceedings of the 42nd International Conference on Machine Learning},
  pages = 	 {58467--58498},
  year = 	 {2025},
  editor = 	 {Singh, Aarti and Fazel, Maryam and Hsu, Daniel and Lacoste-Julien, Simon and Berkenkamp, Felix and Maharaj, Tegan and Wagstaff, Kiri and Zhu, Jerry},
  volume = 	 {267},
  series = 	 {Proceedings of Machine Learning Research},
  month = 	 {13--19 Jul},
  publisher =    {PMLR},
  url = 	 {https://proceedings.mlr.press/v267/tan25a.html}
}

@inproceedings{tan2025amortized,
title={Amortized Sampling with Transferable Normalizing Flows},
author={Charlie B. Tan and Majdi Hassan and Leon Klein and Saifuddin Syed and Dominique Beaini and Michael M. Bronstein and Alexander Tong and Kirill Neklyudov},
booktitle={The Thirty-ninth Annual Conference on Neural Information Processing Systems},
year={2025},
url={https://openreview.net/forum?id=JenfC3ovzU}
}

@article{kim2020softflow,
  title={Softflow: Probabilistic framework for normalizing flow on manifolds},
  author={Kim, Hyeongju and Lee, Hyeonseung and Kang, Woo Hyun and Lee, Joun Yeop and Kim, Nam Soo},
  journal={Advances in Neural Information Processing Systems},
  volume={33},
  pages={16388--16397},
  year={2020}
}

@article{gu2026normalizing,
  title={Normalizing Trajectory Models},
  author={Gu, Jiatao and Chen, Tianrong and Shen, Ying and Berthelot, David and Zhai, Shuangfei and Susskind, Josh},
  journal={arXiv preprint arXiv:2605.08078},
  year={2026}
}

@inproceedings{loaiza2023denoising,
  title={Denoising deep generative models},
  author={Loaiza-Ganem, Gabriel and Ross, Brendan Leigh and Wu, Luhuan and Cunningham, John Patrick and Cresswell, Jesse C and Caterini, Anthony L},
  booktitle={Proceedings on "I Can't Believe It's Not Better!  - Understanding Deep Learning Through Empirical Falsification" at NeurIPS 2022 Workshops},
  pages={41--50},
  year={2023},
  organization={PMLR}
}

@InProceedings{rombach2022high,
    author    = {Rombach, Robin and Blattmann, Andreas and Lorenz, Dominik and Esser, Patrick and Ommer, Bj\"orn},
    title     = {High-Resolution Image Synthesis With Latent Diffusion Models},
    booktitle = {Proceedings of the IEEE/CVF Conference on Computer Vision and Pattern Recognition (CVPR)},
    month     = {June},
    year      = {2022},
    pages     = {10684-10695}
}

@article{abramson2024accurate,
  title={Accurate structure prediction of biomolecular interactions with AlphaFold 3},
  author={Abramson, Josh and Adler, Jonas and Dunger, Jack and Evans, Richard and Green, Tim and Pritzel, Alexander and Ronneberger, Olaf and Willmore, Lindsay and Ballard, Andrew J and Bambrick, Joshua and others},
  journal={Nature},
  volume={630},
  number={8016},
  pages={493--500},
  year={2024},
  publisher={Nature Publishing Group UK London}
}

@article{zeni2025generative,
  title={A generative model for inorganic materials design},
  author={Zeni, Claudio and Pinsler, Robert and Z{\"u}gner, Daniel and Fowler, Andrew and Horton, Matthew and Fu, Xiang and Wang, Zilong and Shysheya, Aliaksandra and Crabb{\'e}, Jonathan and Ueda, Shoko and others},
  journal={Nature},
  volume={639},
  number={8055},
  pages={624--632},
  year={2025},
  publisher={Nature Publishing Group UK London}
}

@inproceedings{geng2026improved,
  title={Improved mean flows: On the challenges of fastforward generative models},
  author={Geng, Zhengyang and Lu, Yiyang and Wu, Zongze and Shechtman, Eli and Kolter, J Zico and He, Kaiming},
  booktitle={Proceedings of the IEEE/CVF Conference on Computer Vision and Pattern Recognition},
  pages={30467--30476},
  year={2026}
}

@inproceedings{
salimans2022progressive,
title={Progressive Distillation for Fast Sampling of Diffusion Models},
author={Tim Salimans and Jonathan Ho},
booktitle={International Conference on Learning Representations},
year={2022},
url={https://openreview.net/forum?id=TIdIXIpzhoI}
}

@article{grathwohl2018ffjord,
  title={Ffjord: Free-form continuous dynamics for scalable reversible generative models},
  author={Grathwohl, Will and Chen, Ricky TQ and Bettencourt, Jesse and Sutskever, Ilya and Duvenaud, David},
  journal={arXiv preprint arXiv:1810.01367},
  year={2018}
}

@inproceedings{huang2018neural,
  title={Neural autoregressive flows},
  author={Huang, Chin-Wei and Krueger, David and Lacoste, Alexandre and Courville, Aaron},
  booktitle={International conference on machine learning},
  pages={2078--2087},
  year={2018},
  organization={PMLR}
}

@article{teshima2020coupling,
  title={Coupling-based invertible neural networks are universal diffeomorphism approximators},
  author={Teshima, Takeshi and Ishikawa, Isao and Tojo, Koichi and Oono, Kenta and Ikeda, Masahiro and Sugiyama, Masashi},
  journal={Advances in Neural Information Processing Systems},
  volume={33},
  pages={3362--3373},
  year={2020}
}

@article{albergo2025stochastic,
  title={Stochastic interpolants: A unifying framework for flows and diffusions},
  author={Albergo, Michael and Boffi, Nicholas M and Vanden-Eijnden, Eric},
  journal={Journal of Machine Learning Research},
  volume={26},
  number={209},
  pages={1--80},
  year={2025}
}

@article{noe2019boltzmann,
  title={Boltzmann generators: Sampling equilibrium states of many-body systems with deep learning},
  author={No{\'e}, Frank and Olsson, Simon and K{\"o}hler, Jonas and Wu, Hao},
  journal={Science},
  volume={365},
  number={6457},
  pages={eaaw1147},
  year={2019},
  publisher={American Association for the Advancement of Science}
}

@misc{owen2013monte,
  title={Monte Carlo theory, methods and examples},
  author={Owen, Art B},
  year={2013},
  publisher={Stanford}
}

@inproceedings{rehman2026autoregressive,
title={Autoregressive Boltzmann Generators},
author={Danyal Rehman and Charlie B. Tan and Yoshua Bengio and Joey Bose and Alexander Tong},
booktitle={Forty-third International Conference on Machine Learning},
year={2026},
url={https://openreview.net/forum?id=75AYDsndHP}
}

@inproceedings{lu2026bidirectional,
  title={Bidirectional normalizing flow: From data to noise and back},
  author={Lu, Yiyang and Sun, Qiao and Wang, Xianbang and Jiang, Zhicheng and Zhao, Hanhong and He, Kaiming},
  booktitle={Proceedings of the IEEE/CVF Conference on Computer Vision and Pattern Recognition},
  pages={2069--2078},
  year={2026}
}

@inproceedings{lu2026onestep,
    title={One-step Latent-free Image Generation with Pixel Mean Flows},
    author={Yiyang Lu and Susie Lu and Qiao Sun and Hanhong Zhao and Zhicheng Jiang and Xianbang Wang and Tianhong Li and Zhengyang Geng and Kaiming He},
    booktitle={Forty-third International Conference on Machine Learning},
    year={2026},
    url={https://openreview.net/forum?id=WUK8JIeetF}
}

@misc{lee2026flowmaplanguagemodels,
      title={Flow Map Language Models: One-step Language Modeling via Continuous Denoising}, 
      author={Chanhyuk Lee and Jaehoon Yoo and Manan Agarwal and Sheel Shah and Jerry Huang and Aditi Raghunathan and Seunghoon Hong and Nicholas M. Boffi and Jinwoo Kim},
      year={2026},
      eprint={2602.16813},
      archivePrefix={arXiv},
      primaryClass={cs.CL},
      url={https://arxiv.org/abs/2602.16813}, 
}

@misc{li2026basicsletdenoisinggenerative,
      title={Back to Basics: Let Denoising Generative Models Denoise}, 
      author={Tianhong Li and Kaiming He},
      year={2026},
      eprint={2511.13720},
      archivePrefix={arXiv},
      primaryClass={cs.CV},
      url={https://arxiv.org/abs/2511.13720}, 
}

@inproceedings{
    sabour2025align,
    title={Align Your Flow: Scaling Continuous-Time Flow Map Distillation},
    author={Amirmojtaba Sabour and Sanja Fidler and Karsten Kreis},
    booktitle={The Thirty-ninth Annual Conference on Neural Information Processing Systems},
    year={2025},
    url={https://openreview.net/forum?id=pzHuesCvcO}
}

@inproceedings{
    karras2022elucidating,
    title={Elucidating the Design Space of Diffusion-Based Generative Models},
    author={Tero Karras and Miika Aittala and Timo Aila and Samuli Laine},
    booktitle={Advances in Neural Information Processing Systems},
    editor={Alice H. Oh and Alekh Agarwal and Danielle Belgrave and Kyunghyun Cho},
    year={2022},
    url={https://openreview.net/forum?id=k7FuTOWMOc7}
}

@InProceedings{pmlr-v139-nichol21a,
  title = 	 {Improved Denoising Diffusion Probabilistic Models},
  author =       {Nichol, Alexander Quinn and Dhariwal, Prafulla},
  booktitle = 	 {Proceedings of the 38th International Conference on Machine Learning},
  pages = 	 {8162--8171},
  year = 	 {2021},
  editor = 	 {Meila, Marina and Zhang, Tong},
  volume = 	 {139},
  series = 	 {Proceedings of Machine Learning Research},
  month = 	 {18--24 Jul},
  publisher =    {PMLR},
  url = 	 {https://proceedings.mlr.press/v139/nichol21a.html}
}

@article{williams_simple_1992,
	title = {Simple statistical gradient-following algorithms for connectionist reinforcement learning},
	volume = {8},
	issn = {1573-0565},
	url = {https://doi.org/10.1007/BF00992696},
	doi = {10.1007/BF00992696},
	number = {3},
	journal = {Machine Learning},
	author = {Williams, Ronald J.},
	month = may,
	year = {1992},
	pages = {229--256},
}

@article{tabak_density_2010,
	title = {Density estimation by dual ascent of the log-likelihood},
	volume = {8},
	issn = {15396746, 19450796},
	url = {https://link.intlpress.com/JDetail/1806265554799382529},
	doi = {10.4310/CMS.2010.v8.n1.a11},
	language = {en},
	number = {1},
	urldate = {2026-08-17},
	journal = {Communications in Mathematical Sciences},
	author = {Tabak, Esteban G. and Vanden-Eijnden, Eric},
	year = {2010},
	pages = {217--233},
}

@article{tabak_family_2013,
	title = {A {Family} of {Nonparametric} {Density} {Estimation} {Algorithms}},
	volume = {66},
	copyright = {Copyright © 2012 Wiley Periodicals, Inc.},
	issn = {1097-0312},
	url = {https://onlinelibrary.wiley.com/doi/abs/10.1002/cpa.21423},
	doi = {10.1002/cpa.21423},
	language = {en},
	number = {2},
	urldate = {2026-08-17},
	journal = {Communications on Pure and Applied Mathematics},
	author = {Tabak, E. G. and Turner, Cristina V.},
	year = {2013},
	pages = {145--164},
}

@article{tong_improving_2024,
title={Improving and generalizing flow-based generative models with minibatch optimal transport},
author={Alexander Tong and Kilian FATRAS and Nikolay Malkin and Guillaume Huguet and Yanlei Zhang and Jarrid Rector-Brooks and Guy Wolf and Yoshua Bengio},
journal={Transactions on Machine Learning Research},
issn={2835-8856},
year={2024},
url={https://openreview.net/forum?id=CD9Snc73AW},
note={Expert Certification}
}

@inproceedings{chen_iterative_2026,
    title={Normalizing Flows with Iterative Denoising},
    author={Tianrong Chen and Jiatao Gu and David Berthelot and Joshua M. Susskind and Shuangfei Zhai},
    booktitle={Forty-third International Conference on Machine Learning},
    year={2026},
    url={https://openreview.net/forum?id=oUVzFuemEn}
}

@misc{shen_starflow2_2026,
	title = {{STARFlow2}: {Bridging} {Language} {Models} and {Normalizing} {Flows} for {Unified} {Multimodal} {Generation}},
	shorttitle = {{STARFlow2}},
	url = {http://arxiv.org/abs/2605.08029},
	doi = {10.48550/arXiv.2605.08029},
	urldate = {2026-08-17},
	publisher = {arXiv},
	author = {Shen, Ying and Chen, Tianrong and Gao, Yuan and Zhang, Yizhe and Wang, Yuyang and Bautista, Miguel {\'A}ngel and Zhai, Shuangfei and Susskind, Joshua M. and Gu, Jiatao},
	month = may,
	year = {2026},
	note = {arXiv:2605.08029 [cs.CV]},
}

@inproceedings{rezende_variational_2015,
	title = {Variational {Inference} with {Normalizing} {Flows}},
	issn = {1938-7228},
	url = {https://proceedings.mlr.press/v37/rezende15.html},
	language = {en},
	urldate = {2026-08-15},
	booktitle = {Proceedings of the 32nd {International} {Conference} on {Machine} {Learning}},
	publisher = {PMLR},
	author = {Rezende, Danilo and Mohamed, Shakir},
	month = jun,
	year = {2015},
	pages = {1530--1538},
}

@inproceedings{dinh_density_2017,
title={Density estimation using Real {NVP}},
author={Laurent Dinh and Jascha Sohl-Dickstein and Samy Bengio},
booktitle={International Conference on Learning Representations},
year={2017},
url={https://openreview.net/forum?id=HkpbnH9lx}
}

@inproceedings{sohl-dickstein_deep_2015,
	title = {Deep {Unsupervised} {Learning} using {Nonequilibrium} {Thermodynamics}},
	issn = {1938-7228},
	url = {https://proceedings.mlr.press/v37/sohl-dickstein15.html},
	language = {en},
	urldate = {2026-08-16},
	booktitle = {Proceedings of the 32nd {International} {Conference} on {Machine} {Learning}},
	publisher = {PMLR},
	author = {Sohl-Dickstein, Jascha and Weiss, Eric and Maheswaranathan, Niru and Ganguli, Surya},
	month = jun,
	year = {2015},
	pages = {2256--2265},
}

@inproceedings{ho_denoising_2020,
	title = {Denoising {Diffusion} {Probabilistic} {Models}},
	volume = {33},
	url = {https://proceedings.neurips.cc/paper/2020/hash/4c5bcfec8584af0d967f1ab10179ca4b-Abstract.html},
	urldate = {2026-08-16},
	booktitle = {Advances in {Neural} {Information} {Processing} {Systems}},
	publisher = {Curran Associates, Inc.},
	author = {Ho, Jonathan and Jain, Ajay and Abbeel, Pieter},
	year = {2020},
	pages = {6840--6851},
}

@inproceedings{
    song2021scorebased,
    title={Score-Based Generative Modeling through Stochastic Differential Equations},
    author={Yang Song and Jascha Sohl-Dickstein and Diederik P Kingma and Abhishek Kumar and Stefano Ermon and Ben Poole},
    booktitle={International Conference on Learning Representations},
    year={2021},
    url={https://openreview.net/forum?id=PxTIG12RRHS}
}

@inproceedings{
    lipman2023flow,
    title={Flow Matching for Generative Modeling},
    author={Yaron Lipman and Ricky T. Q. Chen and Heli Ben-Hamu and Maximilian Nickel and Matthew Le},
    booktitle={The Eleventh International Conference on Learning Representations },
    year={2023},
    url={https://openreview.net/forum?id=PqvMRDCJT9t}
}

@inproceedings{
    liu2023flow,
    title={Flow Straight and Fast: Learning to Generate and Transfer Data with Rectified Flow},
    author={Xingchao Liu and Chengyue Gong and qiang liu},
    booktitle={The Eleventh International Conference on Learning Representations },
    year={2023},
    url={https://openreview.net/forum?id=XVjTT1nw5z}
}

@InProceedings{zhai2025normalizing,
  title = 	 {Normalizing Flows are Capable Generative Models},
  author =       {Zhai, Shuangfei and Zhang, Ruixiang and Nakkiran, Preetum and Berthelot, David and Gu, Jiatao and Zheng, Huangjie and Chen, Tianrong and Bautista, Miguel \'{A}ngel and Jaitly, Navdeep and Susskind, Joshua M.},
  booktitle = 	 {Proceedings of the 42nd International Conference on Machine Learning},
  pages = 	 {74348--74369},
  year = 	 {2025},
  editor = 	 {Singh, Aarti and Fazel, Maryam and Hsu, Daniel and Lacoste-Julien, Simon and Berkenkamp, Felix and Maharaj, Tegan and Wagstaff, Kiri and Zhu, Jerry},
  volume = 	 {267},
  series = 	 {Proceedings of Machine Learning Research},
  month = 	 {13--19 Jul},
  publisher =    {PMLR},
  url = 	 {https://proceedings.mlr.press/v267/zhai25d.html}
}

@misc{scallop_2026,
	title = {Few-{Step} {Boltzmann} {Generators} via {Scalable} {Likelihood} {Flow} {Maps}},
	url = {http://arxiv.org/abs/2606.29110},
	doi = {10.48550/arXiv.2606.29110},
	urldate = {2026-08-18},
	publisher = {arXiv},
	author = {OuYang, RuiKang and Yu, Hanlin and Ai, Xinyue and He, Yutong and Boffi, Nicholas M. and Ravikumar, Pradeep and Hernandez-Lobato, Jose Miguel and Simchowitz, Max and Miller, Benjamin Kurt and Chehab, Omar},
	month = jun,
	year = {2026},
	note = {arXiv:2606.29110 [cs.LG]},
}

@article{gu2026starflow,
  title={STARFlow: Scaling Latent Normalizing Flows for High-resolution Image Synthesis},
  author={Gu, Jiatao and Chen, Tianrong and Berthelot, David and Zheng, Huangjie and Wang, Yuyang and Zhang, Ruixiang and Dinh, Laurent and Bautista, Miguel Angel and Susskind, Joshua and Zhai, Shuangfei},
  journal={Advances in Neural Information Processing Systems},
  volume={38},
  pages={120986--121022},
  year={2026}
}

@InProceedings{gu2025starflowvideo,
    author    = {Gu, Jiatao and Shen, Ying and Chen, Tianrong and Dinh, Laurent and Wang, Yuyang and Bautista, Miguel Angel and Berthelot, David and Susskind, Josh and Zhai, Shuangfei},
    title     = {STARFlow-V: End-to-End Video Generative Modeling with Autoregressive Normalizing Flows},
    booktitle = {Proceedings of the IEEE/CVF Conference on Computer Vision and Pattern Recognition (CVPR)},
    month     = {June},
    year      = {2026},
    pages     = {9084-9094}
}

@inproceedings{
    albergo2023building,
    title={Building Normalizing Flows with Stochastic Interpolants},
    author={Michael Samuel Albergo and Eric Vanden-Eijnden},
    booktitle={The Eleventh International Conference on Learning Representations },
    year={2023},
    url={https://openreview.net/forum?id=li7qeBbCR1t}
}

@article{chenneuralode,
  title={Neural ordinary differential equations},
  author={Chen, Ricky TQ and Rubanova, Yulia and Bettencourt, Jesse and Duvenaud, David K},
  journal={Advances in neural information processing systems},
  volume={31},
  year={2018}
}

@inproceedings{
    song2024improved,
    title={Improved Techniques for Training Consistency Models},
    author={Yang Song and Prafulla Dhariwal},
    booktitle={The Twelfth International Conference on Learning Representations},
    year={2024},
    url={https://openreview.net/forum?id=WNzy9bRDvG}
}

@InProceedings{song2023consistency,
  title = 	 {Consistency Models},
  author =       {Song, Yang and Dhariwal, Prafulla and Chen, Mark and Sutskever, Ilya},
  booktitle = 	 {Proceedings of the 40th International Conference on Machine Learning},
  pages = 	 {32211--32252},
  year = 	 {2023},
  editor = 	 {Krause, Andreas and Brunskill, Emma and Cho, Kyunghyun and Engelhardt, Barbara and Sabato, Sivan and Scarlett, Jonathan},
  volume = 	 {202},
  series = 	 {Proceedings of Machine Learning Research},
  month = 	 {23--29 Jul},
  publisher =    {PMLR},
  url = 	 {https://proceedings.mlr.press/v202/song23a.html}
}

@article{lee2026reward,
  title={Reward Score Matching: Unifying Reward-based Fine-tuning for Flow and Diffusion Models},
  author={Lee, Jeongjae and Chang, Jinho and Kim, Jeongsol and Ye, Jong Chul},
  journal={arXiv preprint arXiv:2604.17415},
  year={2026}
}

@article{schulman2017proximal,
  title={Proximal policy optimization algorithms},
  author={Schulman, John and Wolski, Filip and Dhariwal, Prafulla and Radford, Alec and Klimov, Oleg},
  journal={arXiv preprint arXiv:1707.06347},
  year={2017}
}

@inproceedings{deng2009imagenet,
  title={Imagenet: A large-scale hierarchical image database},
  author={Deng, Jia and Dong, Wei and Socher, Richard and Li, Li-Jia and Li, Kai and Fei-Fei, Li},
  booktitle={2009 IEEE conference on computer vision and pattern recognition},
  pages={248--255},
  year={2009},
  organization={Ieee}
}

@article{krizhevsky2009learning,
  title={Learning multiple layers of features from tiny images},
  author={Krizhevsky, Alex},
  year={2009},
  publisher={Toronto, ON, Canada}
}

@article{heusel2017gans,
  title={Gans trained by a two time-scale update rule converge to a local nash equilibrium},
  author={Heusel, Martin and Ramsauer, Hubert and Unterthiner, Thomas and Nessler, Bernhard and Hochreiter, Sepp},
  journal={Advances in neural information processing systems},
  volume={30},
  year={2017}
}

@article{jiralerspong2023feature,
  title={Feature likelihood divergence: evaluating the generalization of generative models using samples},
  author={Jiralerspong, Marco and Bose, Joey and Gemp, Ian and Qin, Chongli and Bachrach, Yoram and Gidel, Gauthier},
  journal={Advances in Neural Information Processing Systems},
  volume={36},
  pages={33095--33119},
  year={2023}
}

@article{kynkaanniemi2024applying,
  title={Applying guidance in a limited interval improves sample and distribution quality in diffusion models},
  author={Kynk{\"a}{\"a}nniemi, Tuomas and Aittala, Miika and Karras, Tero and Laine, Samuli and Aila, Timo and Lehtinen, Jaakko},
  journal={Advances in Neural Information Processing Systems},
  volume={37},
  pages={122458--122483},
  year={2024}
}

@article{ho2022classifier,
  title={Classifier-free diffusion guidance},
  author={Ho, Jonathan and Salimans, Tim},
  journal={arXiv preprint arXiv:2207.12598},
  year={2022}
}

@article{rafailov2023direct,
  title={Direct preference optimization: Your language model is secretly a reward model},
  author={Rafailov, Rafael and Sharma, Archit and Mitchell, Eric and Manning, Christopher D and Ermon, Stefano and Finn, Chelsea},
  journal={Advances in neural information processing systems},
  volume={36},
  pages={53728--53741},
  year={2023}
}

@article{klein2023timewarp,
  title={Timewarp: Transferable acceleration of molecular dynamics by learning time-coarsened dynamics},
  author={Klein, Leon and Foong, Andrew and Fjelde, Tor and Mlodozeniec, Bruno and Brockschmidt, Marc and Nowozin, Sebastian and No{\'e}, Frank and Tomioka, Ryota},
  journal={Advances in Neural Information Processing Systems},
  volume={36},
  pages={52863--52883},
  year={2023}
}

@article{lewis2025scalable,
  title={Scalable emulation of protein equilibrium ensembles with generative deep learning},
  author={Lewis, Sarah and Hempel, Tim and Jim{\'e}nez-Luna, Jos{\'e} and Gastegger, Michael and Xie, Yu and Foong, Andrew YK and Satorras, Victor Garc{\'\i}a and Abdin, Osama and Veeling, Bastiaan S and Zaporozhets, Iryna and others},
  journal={Science},
  volume={389},
  number={6761},
  pages={eadv9817},
  year={2025},
  publisher={American Association for the Advancement of Science}
}

@InProceedings{yu2025unisim,
  title = 	 {{U}ni{S}im: A Unified Simulator for Time-Coarsened Dynamics of Biomolecules},
  author =       {Yu, Ziyang and Huang, Wenbing and Liu, Yang},
  booktitle = 	 {Proceedings of the 42nd International Conference on Machine Learning},
  pages = 	 {73097--73114},
  year = 	 {2025},
  editor = 	 {Singh, Aarti and Fazel, Maryam and Hsu, Daniel and Lacoste-Julien, Simon and Berkenkamp, Felix and Maharaj, Tegan and Wagstaff, Kiri and Zhu, Jerry},
  volume = 	 {267},
  series = 	 {Proceedings of Machine Learning Research},
  month = 	 {13--19 Jul},
  publisher =    {PMLR},
  url = 	 {https://proceedings.mlr.press/v267/yu25j.html}
}

@article{klein2023equivariant,
  title={Equivariant flow matching},
  author={Klein, Leon and Kr{\"a}mer, Andreas and No{\'e}, Frank},
  journal={Advances in Neural Information Processing Systems},
  volume={36},
  pages={59886--59910},
  year={2023}
}

@article{klein2024transferable,
  title={Transferable boltzmann generators},
  author={Klein, Leon and No{\'e}, Frank},
  journal={Advances in Neural Information Processing Systems},
  volume={37},
  pages={45281--45314},
  year={2024}
}

@article{liu2025flow,
  title={Flow-grpo: Training flow matching models via online rl},
  author={Liu, Jie and Liu, Gongye and Liang, Jiajun and Li, Yangguang and Liu, Jiaheng and Wang, Xintao and Wan, Pengfei and Zhang, Di and Ouyang, Wanli},
  journal={arXiv preprint arXiv:2505.05470},
  year={2025}
}

@inproceedings{mcallister2026flow,
  title={Flow matching policy gradients},
  author={McAllister, David and Ge, Songwei and Yi, Brent and Kim, Chung Min and Weber, Ethan and Choi, Hongsuk and Feng, Haiwen and Kanazawa, Angjoo},
  booktitle={International Conference on Learning Representations},
  volume={2026},
  pages={36352--36372},
  year={2026}
}

@inproceedings{black2024training,
  title={Training diffusion models with reinforcement learning},
  author={Black, Kevin and Janner, Michael and Du, Yilun and Kostrikov, Ilya and Levine, Sergey},
  booktitle={International Conference on Learning Representations},
  volume={2024},
  pages={4965--4987},
  year={2024}
}

@article{alder1959studies,
  title={Studies in molecular dynamics. I. General method},
  author={Alder, Berni J and Wainwright, Thomas Everett},
  journal={The Journal of Chemical Physics},
  volume={31},
  number={2},
  pages={459--466},
  year={1959},
  publisher={American Institute of Physics}
}

@book{liu2001monte,
  title={Monte Carlo strategies in scientific computing},
  author={Liu, Jun S},
  year={2001},
  publisher={Springer}
}

@article{wirnsberger2020targeted,
  title={Targeted free energy estimation via learned mappings},
  author={Wirnsberger, Peter and Ballard, Andrew J and Papamakarios, George and Abercrombie, Stuart and Racani{\`e}re, S{\'e}bastien and Pritzel, Alexander and Jimenez Rezende, Danilo and Blundell, Charles},
  journal={The Journal of Chemical Physics},
  volume={153},
  number={14},
  year={2020},
  publisher={AIP Publishing}
}

@article{weinan2002string,
  title={String method for the study of rare events},
  author={E, Weinan and Ren, Weiqing and Vanden-Eijnden, Eric},
  journal={Physical Review B},
  volume={66},
  number={5},
  pages={052301},
  year={2002},
  publisher={APS}
}

@book{frenkel1997understanding,
author = {Daan Frenkel and Berend Smit},
title = {Understanding Molecular Simulation: From Algorithms to Applications},
year = {1996},
isbn = {0122673700},
publisher = {Academic Press, Inc.},
address = {USA},
edition = {1st}
}

@article{zuckerman2011equilibrium,
  title={Equilibrium sampling in biomolecular simulations},
  author={Zuckerman, Daniel M},
  journal={Annual review of biophysics},
  volume={40},
  pages={41--62},
  year={2011},
  publisher={Annual Reviews}
}

@book{pathria2017statistical,
  title={Statistical Mechanics: International Series of Monographs in Natural Philosophy},
  author={Pathria, Raj Kumar},
  volume={45},
  year={2017},
  publisher={Elsevier}
}

@book{tuckerman2023statistical,
  title={Statistical mechanics: theory and molecular simulation},
  author={Tuckerman, Mark E},
  year={2023},
  publisher={Oxford university press}
}

@article{watson2023novo,
  title={De novo design of protein structure and function with RFdiffusion},
  author={Watson, Joseph L and Juergens, David and Bennett, Nathaniel R and Trippe, Brian L and Yim, Jason and Eisenach, Helen E and Ahern, Woody and Borst, Andrew J and Ragotte, Robert J and Milles, Lukas F and others},
  journal={Nature},
  volume={620},
  number={7976},
  pages={1089--1100},
  year={2023},
  publisher={Nature Publishing Group UK London}
}

@article{ho2022video,
  title={Video diffusion models},
  author={Ho, Jonathan and Salimans, Tim and Gritsenko, Alexey and Chan, William and Norouzi, Mohammad and Fleet, David J},
  journal={Advances in neural information processing systems},
  volume={35},
  pages={8633--8646},
  year={2022}
}

@InProceedings{liu2023audioldm,
  title = 	 {{A}udio{LDM}: Text-to-Audio Generation with Latent Diffusion Models},
  author =       {Liu, Haohe and Chen, Zehua and Yuan, Yi and Mei, Xinhao and Liu, Xubo and Mandic, Danilo and Wang, Wenwu and Plumbley, Mark D},
  booktitle = 	 {Proceedings of the 40th International Conference on Machine Learning},
  pages = 	 {21450--21474},
  year = 	 {2023},
  editor = 	 {Krause, Andreas and Brunskill, Emma and Cho, Kyunghyun and Engelhardt, Barbara and Sabato, Sivan and Scarlett, Jonathan},
  volume = 	 {202},
  series = 	 {Proceedings of Machine Learning Research},
  month = 	 {23--29 Jul},
  publisher =    {PMLR},
  url = 	 {https://proceedings.mlr.press/v202/liu23f.html}
}

@article{shu2023physics,
  title={A physics-informed diffusion model for high-fidelity flow field reconstruction},
  author={Shu, Dule and Li, Zijie and Farimani, Amir Barati},
  journal={Journal of Computational Physics},
  volume={478},
  pages={111972},
  year={2023},
  publisher={Elsevier}
}

@article{du2024conditional,
  title={Conditional neural field latent diffusion model for generating spatiotemporal turbulence},
  author={Du, Pan and Parikh, Meet Hemant and Fan, Xiantao and Liu, Xin-Yang and Wang, Jian-Xun},
  journal={Nature Communications},
  volume={15},
  number={1},
  pages={10416},
  year={2024},
  publisher={Nature Publishing Group UK London}
}

@article{prabhudesai2023aligning,
  title={Aligning text-to-image diffusion models with reward backpropagation},
  author={Prabhudesai, Mihir and Goyal, Anirudh and Pathak, Deepak and Fragkiadaki, Katerina},
  year={2023}
}

@inproceedings{clark2024directly,
  title={Directly fine-tuning diffusion models on differentiable rewards},
  author={Clark, Kevin and Vicol, Paul and Swersky, Kevin and Fleet, David},
  booktitle={International Conference on Learning Representations},
  volume={2024},
  pages={4793--4822},
  year={2024}
}

@InProceedings{dalle,
  title = 	 {Zero-Shot Text-to-Image Generation},
  author =       {Ramesh, Aditya and Pavlov, Mikhail and Goh, Gabriel and Gray, Scott and Voss, Chelsea and Radford, Alec and Chen, Mark and Sutskever, Ilya},
  booktitle = 	 {Proceedings of the 38th International Conference on Machine Learning},
  pages = 	 {8821--8831},
  year = 	 {2021},
  editor = 	 {Meila, Marina and Zhang, Tong},
  volume = 	 {139},
  series = 	 {Proceedings of Machine Learning Research},
  month = 	 {18--24 Jul},
  publisher =    {PMLR},
  url = 	 {https://proceedings.mlr.press/v139/ramesh21a.html}
}

@inproceedings{dhariwal2021guidance,
 author = {Dhariwal, Prafulla and Nichol, Alexander},
 booktitle = {Advances in Neural Information Processing Systems},
 editor = {M. Ranzato and A. Beygelzimer and Y. Dauphin and P.S. Liang and J. Wortman Vaughan},
 pages = {8780--8794},
 publisher = {Curran Associates, Inc.},
 title = {Diffusion Models Beat GANs on Image Synthesis},
 url = {https://proceedings.neurips.cc/paper_files/paper/2021/file/49ad23d1ec9fa4bd8d77d02681df5cfa-Paper.pdf},
 volume = {34},
 year = {2021}
}

@InProceedings{Bansal_2023_CVPR,
    author    = {Bansal, Arpit and Chu, Hong-Min and Schwarzschild, Avi and Sengupta, Soumyadip and Goldblum, Micah and Geiping, Jonas and Goldstein, Tom},
    title     = {Universal Guidance for Diffusion Models},
    booktitle = {Proceedings of the IEEE/CVF Conference on Computer Vision and Pattern Recognition (CVPR) Workshops},
    month     = {June},
    year      = {2023},
    pages     = {843-852}
}

@article{ma2025inference,
  title={Inference-time scaling for diffusion models beyond scaling denoising steps},
  author={Ma, Nanye and Tong, Shangyuan and Jia, Haolin and Hu, Hexiang and Su, Yu-Chuan and Zhang, Mingda and Yang, Xuan and Li, Yandong and Jaakkola, Tommi and Jia, Xuhui and others},
  journal={arXiv preprint arXiv:2501.09732},
  year={2025}
}

@article{scarpellini2026few,
  title={Few-step Cofolding with All-Atom Flow Maps},
  author={Scarpellini, Gianluca and Shprints, Ron and Holderrieth, Peter and Nam, Juno and Murugan, Pranav and G{\'o}mez-Bombarelli, Rafael and Jaakkola, Tommi and Al-Shedivat, Maruan and Boffi, Nicholas Matthew and Bose, Avishek Joey},
  journal={arXiv preprint arXiv:2606.08375},
  year={2026}
}

@article{wohlwend2025boltz,
	author = {Wohlwend, Jeremy and Corso, Gabriele and Passaro, Saro and Getz, Noah and Reveiz, Mateo and Leidal, Ken and Swiderski, Wojtek and Atkinson, Liam and Portnoi, Tally and Chinn, Itamar and Silterra, Jacob and Jaakkola, Tommi and Barzilay, Regina},
	title = {Boltz-1 Democratizing Biomolecular Interaction Modeling},
	elocation-id = {2024.11.19.624167},
	year = {2025},
	doi = {10.1101/2024.11.19.624167},
	publisher = {Cold Spring Harbor Laboratory},
	URL = {https://www.biorxiv.org/content/early/2025/05/06/2024.11.19.624167},
	eprint = {https://www.biorxiv.org/content/early/2025/05/06/2024.11.19.624167.full.pdf},
	journal = {bioRxiv}
}

@article{shazeer2020glu,
  title={Glu variants improve transformer},
  author={Shazeer, Noam},
  journal={arXiv preprint arXiv:2002.05202},
  year={2020}
}

@article{zhang2019root,
  title={Root mean square layer normalization},
  author={Zhang, Biao and Sennrich, Rico},
  journal={Advances in neural information processing systems},
  volume={32},
  year={2019}
}

@article{su2024roformer,
  title={Roformer: Enhanced transformer with rotary position embedding},
  author={Su, Jianlin and Ahmed, Murtadha and Lu, Yu and Pan, Shengfeng and Bo, Wen and Liu, Yunfeng},
  journal={Neurocomputing},
  volume={568},
  pages={127063},
  year={2024},
  publisher={Elsevier}
}

@InProceedings{peebles2023scalable,
    author    = {Peebles, William and Xie, Saining},
    title     = {Scalable Diffusion Models with Transformers},
    booktitle = {Proceedings of the IEEE/CVF International Conference on Computer Vision (ICCV)},
    month     = {October},
    year      = {2023},
    pages     = {4195-4205}
}

@article{miyasawa1961empirical,
  title={An empirical Bayes estimator of the mean of a normal population},
  author={Miyasawa, Koichi},
  journal={Bull. Inst. Internat. Statist},
  volume={38},
  pages={181--188},
  year={1961}
}

@inproceedings{karras2024analyzing,
  title={Analyzing and improving the training dynamics of diffusion models},
  author={Karras, Tero and Aittala, Miika and Lehtinen, Jaakko and Hellsten, Janne and Aila, Timo and Laine, Samuli},
  booktitle={2024 IEEE/CVF Conference on Computer Vision and Pattern Recognition (CVPR)},
  pages={24174--24184},
  year={2024},
  organization={IEEE}
}

@misc{shao2024deepseekmathpushinglimitsmathematical,
      title={DeepSeekMath: Pushing the Limits of Mathematical Reasoning in Open Language Models}, 
      author={Zhihong Shao and Peiyi Wang and Qihao Zhu and Runxin Xu and Junxiao Song and Xiao Bi and Haowei Zhang and Mingchuan Zhang and Y. K. Li and Y. Wu and Daya Guo},
      year={2024},
      eprint={2402.03300},
      archivePrefix={arXiv},
      primaryClass={cs.CL},
      url={https://arxiv.org/abs/2402.03300}, 
}

@inproceedings{behrmann2021understanding,
  title={Understanding and mitigating exploding inverses in invertible neural networks},
  author={Behrmann, Jens and Vicol, Paul and Wang, Kuan-Chieh and Grosse, Roger and Jacobsen, J{\"o}rn-Henrik},
  booktitle={International Conference on Artificial Intelligence and Statistics},
  pages={1792--1800},
  year={2021},
  organization={PMLR}
}

@misc{wang2025skyworkvlrewardeffectivereward,
      title={Skywork-VL Reward: An Effective Reward Model for Multimodal Understanding and Reasoning}, 
      author={Xiaokun Wang and Peiyu Wang and Jiangbo Pei and Wei Shen and Yi Peng and Yunzhuo Hao and Weijie Qiu and Ai Jian and Tianyidan Xie and Xuchen Song and Yang Liu and Yahui Zhou},
      year={2025},
      eprint={2505.07263},
      archivePrefix={arXiv},
      primaryClass={cs.CV},
      url={https://arxiv.org/abs/2505.07263}, 
}

@inproceedings{
midgley2023flow,
title={Flow Annealed Importance Sampling Bootstrap},
author={Laurence Illing Midgley and Vincent Stimper and Gregor N. C. Simm and Bernhard Sch{\"o}lkopf and Jos{\'e} Miguel Hern{\'a}ndez-Lobato},
booktitle={The Eleventh International Conference on Learning Representations },
year={2023},
url={https://openreview.net/forum?id=XCTVFJwS9LJ}
}

@article{kingma2021variational,
  title={Variational diffusion models},
  author={Kingma, Diederik and Salimans, Tim and Poole, Ben and Ho, Jonathan},
  journal={Advances in neural information processing systems},
  volume={34},
  pages={21696--21707},
  year={2021}
}

@article{jaynes1957,
    author = "Jaynes, E. T.",
    title = "{Information Theory and Statistical Mechanics}",
    doi = "10.1103/PhysRev.106.620",
    journal = "Phys. Rev.",
    volume = "106",
    pages = "620--630",
    year = "1957"
}

@inproceedings{wallace2024diffusion,
  title={Diffusion model alignment using direct preference optimization},
  author={Wallace, Bram and Dang, Meihua and Rafailov, Rafael and Zhou, Linqi and Lou, Aaron and Purushwalkam, Senthil and Ermon, Stefano and Xiong, Caiming and Joty, Shafiq and Naik, Nikhil},
  booktitle={2024 IEEE/CVF Conference on Computer Vision and Pattern Recognition (CVPR)},
  pages={8228--8238},
  year={2024},
  organization={IEEE}
}

@inproceedings{domingo2025adjoint,
  title={Adjoint matching: Fine-tuning flow and diffusion generative models with memoryless stochastic optimal control},
  author={Domingo i Enrich, Carles and Drozdzal, Michal and Karrer, Brian and Chen, Ricky TQ},
  booktitle={International Conference on Learning Representations},
  volume={2025},
  pages={53791--53846},
  year={2025}
}

@inproceedings{liu2025value,
title={Value Gradient Guidance for Flow Matching Alignment},
author={Zhen Liu and Tim Z. Xiao and Carles Domingo-Enrich and Weiyang Liu and Dinghuai Zhang},
booktitle={The Thirty-ninth Annual Conference on Neural Information Processing Systems},
year={2025},
url={https://openreview.net/forum?id=6MmOy2Ji8V}
}

@inproceedings{zheng2026diffusionnft,
  title={Diffusionnft: Online diffusion reinforcement with forward process},
  author={Zheng, Kaiwen and Chen, Huayu and Ye, Haotian and Wang, Haoxiang and Zhang, Qinsheng and Jiang, Kai and Su, Hang and Ermon, Stefano and Zhu, Jun and Liu, Ming-Yu},
  booktitle={International Conference on Learning Representations},
  volume={2026},
  pages={134129--134150},
  year={2026}
}

@inproceedings{blasingame2025greed,
title={Greed is Good: A Unifying Perspective on Guided Generation},
author={Zander W. Blasingame and Chen Liu},
booktitle={The Thirty-ninth Annual Conference on Neural Information Processing Systems},
year={2025},
url={https://openreview.net/forum?id=s14pdQgoLb}
}

@InProceedings{ben-hamu2024dflow,
  title = 	 {D-Flow: Differentiating through Flows for Controlled Generation},
  author =       {Ben-Hamu, Heli and Puny, Omri and Gat, Itai and Karrer, Brian and Singer, Uriel and Lipman, Yaron},
  booktitle = 	 {Proceedings of the 41st International Conference on Machine Learning},
  pages = 	 {3462--3483},
  year = 	 {2024},
  editor = 	 {Salakhutdinov, Ruslan and Kolter, Zico and Heller, Katherine and Weller, Adrian and Oliver, Nuria and Scarlett, Jonathan and Berkenkamp, Felix},
  volume = 	 {235},
  series = 	 {Proceedings of Machine Learning Research},
  month = 	 {21--27 Jul},
  publisher =    {PMLR},
  url = 	 {https://proceedings.mlr.press/v235/ben-hamu24a.html}
}

@inproceedings{imagereward,
author = {Xu, Jiazheng and Liu, Xiao and Wu, Yuchen and Tong, Yuxuan and Li, Qinkai and Ding, Ming and Tang, Jie and Dong, Yuxiao},
title = {ImageReward: learning and evaluating human preferences for text-to-image generation},
year = {2023},
publisher = {Curran Associates Inc.},
address = {Red Hook, NY, USA},
booktitle = {Proceedings of the 37th International Conference on Neural Information Processing Systems},
articleno = {700},
numpages = {33},
location = {New Orleans, LA, USA},
series = {NIPS '23}
}

@inproceedings{pickscore,
 author = {Kirstain, Yuval and Polyak, Adam and Singer, Uriel and Matiana, Shahbuland and Penna, Joe and Levy, Omer},
 booktitle = {Advances in Neural Information Processing Systems},
 doi = {10.52202/075280-1594},
 editor = {A. Oh and T. Naumann and A. Globerson and K. Saenko and M. Hardt and S. Levine},
 pages = {36652--36663},
 publisher = {Curran Associates, Inc.},
 title = {Pick-a-Pic: An Open Dataset of User Preferences for Text-to-Image Generation},
 url = {https://proceedings.neurips.cc/paper_files/paper/2023/file/73aacd8b3b05b4b503d58310b523553c-Paper-Conference.pdf},
 volume = {36},
 year = {2023}
}

@InProceedings{pmlr-v80-haarnoja18b,
  title = 	 {Soft Actor-Critic: Off-Policy Maximum Entropy Deep Reinforcement Learning with a Stochastic Actor},
  author =       {Haarnoja, Tuomas and Zhou, Aurick and Abbeel, Pieter and Levine, Sergey},
  booktitle = 	 {Proceedings of the 35th International Conference on Machine Learning},
  pages = 	 {1861--1870},
  year = 	 {2018},
  editor = 	 {Dy, Jennifer and Krause, Andreas},
  volume = 	 {80},
  series = 	 {Proceedings of Machine Learning Research},
  month = 	 {10--15 Jul},
  publisher =    {PMLR},
  url = 	 {https://proceedings.mlr.press/v80/haarnoja18b.html}
}

@article{wu2023human,
  title={Human preference score v2: A solid benchmark for evaluating human preferences of text-to-image synthesis},
  author={Wu, Xiaoshi and Hao, Yiming and Sun, Keqiang and Chen, Yixiong and Zhu, Feng and Zhao, Rui and Li, Hongsheng},
  journal={arXiv preprint arXiv:2306.09341},
  year={2023}
}

@misc{mcintoshsmith2024isambardai,
  author        = {McIntosh-Smith, S. and Alam, S. R. and Woods, C.},
  title         = {{Isambard-AI}: A Leadership Class Supercomputer Optimised Specifically for {Artificial Intelligence}},
  year          = {2024},
  eprint        = {2410.11199},
  archivePrefix = {arXiv},
  doi           = {10.48550/arXiv.2410.11199},
  url           = {https://doi.org/10.48550/arXiv.2410.11199}
}
\bibliographystyle{iclr2027_conference}

\newpage
\appendix

\startcontents[appendix]

{\LARGE\bfseries Appendix\par}
\vspace{2em}

\printcontents[appendix]{}{1}{
    \setcounter{tocdepth}{2}
}

\clearpage
\section{Theoretical Results}
\label{app:theory}

\subsection{\nfmr Proofs}
\label{app:imf_deriv_and_eulerian_connection}

\looseness=-1
In the following, we adapt~\citet{boffi2025flow, boffi2025how} settings to improved Mean Flow.

For the linear interpolant in \cref{eq:flow_matching},
\[
x_s=I_s(x_0,x_1)=(1-s)x_0+sx_1,
\qquad (x_0,x_1)\sim\pi,
\qquad x_s\sim p_s.
\]
The conditional and marginal velocities are, respectively,
\begin{equation}
v_s^*(x_s\mid x_0,x_1)=x_1-x_0,
\qquad
v_s(x)=\mathbb E\!\left[
v_s^*(x_s\mid x_0,x_1)\mid x_s=x
\right].
\label{eq:nfmr_true_velocity}
\end{equation}
From the denoiser parametrization in \cref{eq:nfm_denoiser_param},
\begin{equation}
f^\theta_{s,t}(x)=x+(t-s)v^\theta_{s,t}(x),
\qquad
f^\theta_{s,0}(x)=\delta^\theta_{s,0}(x)=x-sv^\theta_{s,0}(x),
\label{eq:nfmr_endpoint_map_app}
\end{equation}
and $p^\theta_{1\to0}=(f^\theta_{1,0})_\#p_1$.

\xhdr{Regression loss and Eulerian consistency}
\looseness=-1
Let $r^\theta_{s,t}$ denote the regression residual in
\cref{eq:imf}, evaluated on the interpolant sample $x_s$:
\begin{equation}
\begin{aligned}
r^\theta_{s,t}
:={}&v_s^*(x_s\mid x_0,x_1)-v^\theta_{s,t}(x_s)\\
&+(t-s)\left(
\partial_s v^\theta_{s,t}(x_s)
+\nabla_xv^\theta_{s,t}(x_s)v^\theta_{s,s}(x_s)
\right).
\end{aligned}
\label{eq:nfmr_regression_residual}
\end{equation}
The derivative $\partial_s$ acts on the first time argument while
holding the second fixed. Stop-gradient changes the training gradient
but not the loss value, so it is omitted here.
We introduce
\begin{equation}
\mathcal L_{\rm R}(s,t)
=\mathbb E_{(x_0,x_1)\sim\pi}
\left[\|r^\theta_{s,t}\|^2\right],
\qquad
\mathcal L_{\rm R}(\theta)
=\mathbb E_{(s,t)\sim\gT}
\left[\mathcal L_{\rm R}(s,t)\right].
\label{eq:nfmr_fixed_time_loss_app}
\end{equation}

\looseness=-1
At $t=s$, the residual reduces to
$v_s^*(x_s\mid x_0,x_1)-v^\theta_{s,s}(x_s)$.
Identifying the learned velocity $v_s^\theta$ in
\cref{eq:flow_matching} with $v^\theta_{s,s}$, define
\begin{align}
\mathcal L_{\rm FM}
&:=\mathbb E_{s\sim \nu(s)}
\left[\mathcal L_{\rm R}(s,s)\right],
\label{eq:nfmr_fm_loss_app}\\
\mathcal L_{\rm distill}
&:=\mathbb E_{(s,t)\sim\gT,\,s\ne t}
\left[\mathcal L_{\rm R}(s,t)\right],
\label{eq:nfmr_distill_loss_app}\\
\mathcal L_{\rm end}
&:=\int_0^1\mathcal L_{\rm R}(s,0)\,ds.
\label{eq:nfmr_endpoint_loss_app}
\end{align}
Here $\nu(s)$ is the time-sampling density from
\cref{eq:flow_matching}, and the expectation defining
$\mathcal L_{\rm distill}$ is conditional on an off-diagonal pair.
For the formal statement below, suppose the time-sampling law in
\cref{eq:imf} assigns probability $\eta\in(0,1)$ to diagonal pairs
$(s,s)$, with $s\sim \nu(s)$, and probability $1-\eta$ to
off-diagonal pairs. Consequently,
\begin{equation}
\mathcal L_{\rm R}(\theta)
=\eta\mathcal L_{\rm FM}
+(1-\eta)\mathcal L_{\rm distill}.
\label{eq:nfmr_loss_decomposition}
\end{equation}
The diagonal is understood as the continuous extension of the
admissible time pairs used in \cref{eq:imf}.

\begin{mdframed}[style=MyFrame2]
\begin{restatable}[Endpoint control for pointwise distillation]
{proposition}{imfwasserstein}
\label{prop:nfmr_wasserstein_formal}
\looseness=-1
Consider the loss and time-sampling law in
\cref{eq:nfmr_fixed_time_loss_app,eq:nfmr_loss_decomposition}.
Assume $p_0,p_1$ have finite second moments and the marginal velocity
$v_s$ admits a probability flow $X_{s,t}$: for $x_1\sim p_1$,
$s\mapsto X_{1,s}(x_1)$ is almost surely absolutely continuous,
\[
\frac{d}{ds}X_{1,s}(x_1)=v_s\!\left(X_{1,s}(x_1)\right)
\quad\text{a.e.},
\qquad
(X_{1,s})_\#p_1=p_s,
\qquad X_{1,1}=\mathrm{Id}.
\]
Assume the derivatives in \cref{eq:nfmr_regression_residual} exist,
$f^\theta_{s,0}$ is continuously differentiable in $(s,x)$ up to the
time boundaries, $f^\theta_{0,0}=\mathrm{Id}$, and
\[
\int_0^1\mathbb E_{x_s\sim 
\nu_s}
\left[\left\|
\partial_s f^\theta_{s,0}(x_s)
+\nabla_xf^\theta_{s,0}(x_s)v_s(x_s)
\right\|^2\right]ds<\infty.
\]
Suppose $\nu(s)\geq \nu_{\min}>0$ almost everywhere on $(0,1)$ and,
for finite constants $\widehat L,C_{\rm tr}\geq0$ independent of
$\theta$,
\[
\sup_{s\in[0,1],\,x\in\R^d}
\left\|I-\nabla_xf^\theta_{s,0}(x)\right\|_{\rm op}
\leq\widehat L,
\qquad
\mathcal L_{\rm end}\leq C_{\rm tr}\mathcal L_{\rm distill}.
\]
Then
\[
W_2^2\!\left(p_0,p^\theta_{1\to0}\right)
\leq
\left(
\frac{\widehat L^2}{\eta \nu_{\min}}
+\frac{C_{\rm tr}}{1-\eta}
\right)\mathcal L_{\rm R}(\theta).
\]
\end{restatable}
\end{mdframed}

\begin{proof}
Define the endpoint Eulerian residual using the marginal velocity
$v_s$ from \cref{eq:nfmr_true_velocity}:
\begin{equation}
R^\theta_s(x)
:=\partial_s f^\theta_{s,0}(x)
+\nabla_x f^\theta_{s,0}(x)v_s(x).
\label{eq:nfmr_true_endpoint_residual}
\end{equation}
Fix a reference sample $x_1\sim p_1$ and consider its true
probability-flow trajectory $s\mapsto X_{1,s}(x_1)$. By definition,
\[
\frac{d}{ds}X_{1,s}(x_1)
=v_s\!\left(X_{1,s}(x_1)\right),
\qquad X_{1,1}(x_1)=x_1.
\]
Applying the chain rule to the predicted endpoint along this
trajectory gives, almost everywhere in $s$,
\begin{align*}
\frac{d}{ds}f^\theta_{s,0}\!\left(X_{1,s}(x_1)\right)
&=\partial_s f^\theta_{s,0}\!\left(X_{1,s}(x_1)\right)
+\nabla_x f^\theta_{s,0}\!\left(X_{1,s}(x_1)\right)
  \frac{d}{ds}X_{1,s}(x_1)\\
&=\partial_s f^\theta_{s,0}\!\left(X_{1,s}(x_1)\right)
+\nabla_x f^\theta_{s,0}\!\left(X_{1,s}(x_1)\right)
  v_s\!\left(X_{1,s}(x_1)\right)\\
&=R^\theta_s\!\left(X_{1,s}(x_1)\right).
\end{align*}
Integrating from $s=0$ to $s=1$ therefore yields
\[
f^\theta_{1,0}\!\left(X_{1,1}(x_1)\right)
-f^\theta_{0,0}\!\left(X_{1,0}(x_1)\right)
=\int_0^1 R^\theta_s\!\left(X_{1,s}(x_1)\right)\,ds.
\]
Using $X_{1,1}=\mathrm{Id}$ and $f^\theta_{0,0}=\mathrm{Id}$,
this becomes
\begin{equation}
f^\theta_{1,0}(x_1)-X_{1,0}(x_1)
=\int_0^1 R^\theta_s\!\left(X_{1,s}(x_1)\right)\,ds.
\label{eq:nfmr_trajectory_endpoint_difference}
\end{equation}

Now $(X_{1,0})_\#p_1=p_0$ and
$(f^\theta_{1,0})_\#p_1=p^\theta_{1\to0}$.
Thus evaluating both maps on the same $x_1\sim p_1$ gives a
coupling of $p_0$ and $p^\theta_{1\to0}$. Since the squared
Wasserstein distance is the infimum of the expected squared
distance over all such couplings,
\begin{align*}
W_2^2\!\left(p_0,p^\theta_{1\to0}\right)
&\leq
\mathbb E_{x_1\sim p_1}
\left[\left\|f^\theta_{1,0}(x_1)-X_{1,0}(x_1)\right\|^2\right]\\
&=
\mathbb E_{x_1\sim p_1}
\left[\left\|
\int_0^1R^\theta_s\!\left(X_{1,s}(x_1)\right)\,ds
\right\|^2\right].
\end{align*}
For almost every $x_1$, the triangle inequality followed by
Cauchy--Schwarz in the time variable gives
\begin{align*}
\left\|\int_0^1R^\theta_s\!\left(X_{1,s}(x_1)\right)\,ds\right\|^2
&\leq
\left(\int_0^1
\left\|R^\theta_s\!\left(X_{1,s}(x_1)\right)\right\|\,ds\right)^2\\
&\leq
\left(\int_0^1 1\,ds\right)
\left(\int_0^1
\left\|R^\theta_s\!\left(X_{1,s}(x_1)\right)\right\|^2\,ds\right)\\
&=\int_0^1
\left\|R^\theta_s\!\left(X_{1,s}(x_1)\right)\right\|^2\,ds.
\end{align*}
Taking expectations and interchanging the nonnegative integrals
by Tonelli's theorem, we obtain
\begin{equation}
\begin{aligned}
W_2^2\!\left(p_0,p^\theta_{1\to0}\right)
&\leq
\int_0^1\mathbb E_{x_1\sim p_1}
\left[\left\|R^\theta_s\!\left(X_{1,s}(x_1)\right)\right\|^2\right]ds\\
&=
\int_0^1\mathbb E_{x_s\sim p_s}
\left[\|R^\theta_s(x_s)\|^2\right]ds.
\end{aligned}
\label{eq:nfmr_endpoint_eulerian}
\end{equation}
The last equality uses $(X_{1,s})_\#p_1=p_s$ for each fixed $s$.
It identifies the one-time distributions of the probability flow
and the stochastic interpolant, not their sample paths.

Differentiating $f^\theta_{s,0}(x)=x-sv^\theta_{s,0}(x)$ and using
\cref{eq:nfmr_regression_residual} at $t=0$ yields
\begin{equation}
\begin{aligned}
R^\theta_s(x_s)
={}&\mathbb E\!\left[r^\theta_{s,0}\mid x_s\right]\\
&+\left(I-\nabla_x f^\theta_{s,0}(x_s)\right)
\left(v^\theta_{s,s}(x_s)-v_s(x_s)\right).
\end{aligned}
\label{eq:nfmr_true_residual_decomp}
\end{equation}
Indeed, conditioning the regression residual on $x_s$ replaces
$v_s^*(x_s\mid x_0,x_1)$ by $v_s(x_s)$, while
$I-\nabla_x f^\theta_{s,0}=s\nabla_xv^\theta_{s,0}$.
Conditional Jensen gives
\begin{align*}
\mathbb E\left[
\left\|\mathbb E[r^\theta_{s,0}\mid x_s]\right\|^2
\right]
&\leq\mathcal L_{\rm R}(s,0),\\
\mathbb E\left[
\left\|v^\theta_{s,s}(x_s)-v_s(x_s)\right\|^2
\right]
&\leq\mathcal L_{\rm R}(s,s).
\end{align*}
Moreover, $\nu(s)\geq \nu_{\min}$ implies
\[
\int_0^1\mathcal L_{\rm R}(s,s)\,ds
\leq\frac{\mathcal L_{\rm FM}}{\nu_{\min}}.
\]
Applying the triangle inequality in $L^2(ds\,p_s(dx))$ to
\cref{eq:nfmr_true_residual_decomp} therefore gives
\begin{equation}
\left(
\int_0^1\mathbb E\left[\|R^\theta_s(x_s)\|^2\right]ds
\right)^{1/2}
\leq\sqrt{\mathcal L_{\rm end}}
+\frac{\widehat L}{\sqrt{\nu_{\min}}}
\sqrt{\mathcal L_{\rm FM}}.
\label{eq:nfmr_endpoint_residual_control}
\end{equation}
Combining \cref{eq:nfmr_endpoint_eulerian,eq:nfmr_endpoint_residual_control}
with the endpoint-loss assumption and Cauchy--Schwarz yields
\begin{align*}
W_2^2\!\left(p_0,p^\theta_{1\to0}\right)
&\leq\left(
\sqrt{C_{\rm tr}\mathcal L_{\rm distill}}
+\frac{\widehat L}{\sqrt{\nu_{\min}}}
\sqrt{\mathcal L_{\rm FM}}
\right)^2\\
&\leq\left(
\frac{C_{\rm tr}}{1-\eta}
+\frac{\widehat L^2}{\eta \nu_{\min}}
\right)
\left((1-\eta)\mathcal L_{\rm distill}
+\eta\mathcal L_{\rm FM}\right)\\
&=\left(
\frac{C_{\rm tr}}{1-\eta}
+\frac{\widehat L^2}{\eta \nu_{\min}}
\right)\mathcal L_{\rm R}(\theta).
\end{align*}
\end{proof}

\begin{remark}[Loss normalization]
If the objective is instead defined as the unweighted sum
$\mathcal L_{\rm R}=\mathcal L_{\rm FM}+\mathcal L_{\rm distill}$,
the same proof gives
\[
W_2^2\!\left(p_0,p^\theta_{1\to0}\right)
\leq\left(C_{\rm tr}+\frac{\widehat L^2}{\nu_{\min}}\right)
\mathcal L_{\rm R}.
\]
For uniform source-time sampling, $\nu(s)=1$, this reduces to
$(C_{\rm tr}+\widehat L^2)\mathcal L_{\rm R}$.
\end{remark}

\subsection{\nfmd Proofs}

Here we provide the proofs and derivations for the results stated in
\S\ref{sec:nfmd}.


\xhdr{Conditional and marginal \nfmd}

\begin{mdframed}[style=MyFrame2]
\conditionalnfmd*
\end{mdframed}

\begin{proof}
Fix $t\in[0,1]$. By inserting the marginal density $p_t(x_t)$ into the
conditional KL, we obtain
\begin{align}
\mathbb{E}_{x_0}
\left[
\KL\left(
p_t(\cdot\mid x_0)
\,\middle\|\,
p^\theta_{1\to t}
\right)
\right]
&=
\mathbb{E}_{x_0,x_t}
\left[
\log
\frac{p_t(x_t\mid x_0)}
     {p^\theta_{1\to t}(x_t)}
\right]
\nonumber\\
&=
\mathbb{E}_{x_0,x_t}
\left[
\log
\frac{p_t(x_t\mid x_0)}
     {p_t(x_t)}
\right]
+
\mathbb{E}_{x_0,x_t}
\left[
\log
\frac{p_t(x_t)}
     {p^\theta_{1\to t}(x_t)}
\right].
\end{align}
The first term is the mutual information
$\mathcal I(x_0;x_t)$. Since the second integrand depends only on $x_t$,
marginalizing over $x_0$ gives
\[
\mathbb{E}_{x_0,x_t}
\left[
\log
\frac{p_t(x_t)}
     {p^\theta_{1\to t}(x_t)}
\right]
=
\KL\left(
p_t
\,\middle\|\,
p^\theta_{1\to t}
\right).
\]
Hence, for every $t$,
\[
\mathbb{E}_{x_0}
\left[
\KL\left(
p_t(\cdot\mid x_0)
\,\middle\|\,
p^\theta_{1\to t}
\right)
\right]
=
\KL\left(
p_t
\,\middle\|\,
p^\theta_{1\to t}
\right)
+
\mathcal I(x_0;x_t).
\]
Averaging over the time distribution yields
\[
\mathcal L_{\mathrm D}
=
\E_{\gT_{\mathrm S}}
\left[
\KL\left(
p_t
\,\middle\|\,
p^\theta_{1\to t}
\right)
\right]
+
\mathbb E_t
\left[
\mathcal I(x_0;x_t)
\right].
\]
The upper bound follows from non-negativity of mutual information.
Moreover, $\mathbb E_t[\mathcal I(x_0;x_t)]$ is determined entirely by
the prescribed corruption process and is independent of $\theta$.
Whenever it is finite, adding this term therefore does not change the
set of minimizers.
\end{proof}


\xhdr{Empirical conditional \nfmd is bounded below}

We next show that the likelihood-form objective in
\cref{eq:nfmd_single_conditional} remains bounded below when evaluated on
the empirical data distribution. Let
\[
\widehat p_0
=
\frac{1}{n}\sum_{i=1}^n\delta_{x_i},
\]
and consider the linear stochastic interpolant
\[
x_t=(1-t)x_0+t x_1,
\qquad
x_0\sim\widehat p_0,
\qquad
x_1\sim p_1=\mathcal N(0,I).
\]
For $t>0$, its conditional distribution is
\[
p_t(\cdot\mid x_0)
=
\mathcal N\left((1-t)x_0,t^2I\right).
\]

\begin{mdframed}[style=MyFrame2]
\begin{proposition}[Empirical conditional \nfmd is bounded below]
\label{prop:nfmd_empirical_bounded}
Suppose the time-sampling density $\nu$ associated with
$\gT_{\mathrm S}$ satisfies
\[
\int_0^1 \nu(t)|\log t|\,dt<\infty.
\]
Then the empirical likelihood-form objective corresponding to
\cref{eq:nfmd_single_conditional},
\[
\widehat{\mathcal L}^{\mathrm{lik}}_{\mathrm D}(\theta)
\coloneqq
\mathbb E_{\substack{t\sim\nu,\,
                     x_0\sim\widehat p_0,\,
                     x_t\sim p_t(\cdot\mid x_0)}}
\left[
\frac{1}{2}
\left\|
(f^\theta_{1,t})^{-1}(x_t)
\right\|^2
-
\log\left|
\det
\frac{\partial(f^\theta_{1,t})^{-1}(x_t)}
     {\partial x_t}
\right|
\right]
\]
satisfies
\[
\widehat{\mathcal L}^{\mathrm{lik}}_{\mathrm D}(\theta)
\geq
\frac{d}{2}
+
d\int_0^1\nu(t)\log t\,dt
>
-\infty
\]
for every $\theta$. Consequently,
\[
\inf_\theta
\widehat{\mathcal L}^{\mathrm{lik}}_{\mathrm C,1}(\theta)
>
-\infty.
\]
\end{proposition}
\end{mdframed}

\begin{proof}
For fixed $t>0$ and $x_0$, expanding the conditional KL gives
\begin{align}
&\mathbb E_{x_t\sim p_t(\cdot\mid x_0)}
\left[
\frac{1}{2}
\left\|
(f^\theta_{1,t})^{-1}(x_t)
\right\|^2
-
\log\left|
\det
\frac{\partial(f^\theta_{1,t})^{-1}(x_t)}
     {\partial x_t}
\right|
\right]
\nonumber\\
&\qquad =
\KL\left(
p_t(\cdot\mid x_0)
\,\middle\|\,
p^\theta_{1\to t}
\right)
+
\mathcal H\left(p_t(\cdot\mid x_0)\right)
-
\frac{d}{2}\log(2\pi).
\end{align}
Under the linear interpolant,
\[
p_t(\cdot\mid x_0)
=
\mathcal N\left((1-t)x_0,t^2I\right),
\]
and therefore
\[
\mathcal H\left(p_t(\cdot\mid x_0)\right)
=
\frac{d}{2}\log(2\pi e t^2).
\]
Hence, by non-negativity of the KL divergence,
\[
\mathbb E_{x_t\sim p_t(\cdot\mid x_0)}
\left[
\frac{1}{2}
\left\|
(f^\theta_{1,t})^{-1}(x_t)
\right\|^2
-
\log\left|
\det
\frac{\partial(f^\theta_{1,t})^{-1}(x_t)}
     {\partial x_t}
\right|
\right]
\geq
\frac{d}{2}+d\log t.
\]
Averaging over $x_0\sim\widehat p_0$ and $t\sim\nu$ yields
\[
\widehat{\mathcal L}^{\mathrm{lik}}_{\mathrm D}(\theta)
\geq
\frac{d}{2}
+
d\int_0^1\nu(t)\log t\,dt.
\]
The right-hand side is finite under the stated assumption and independent
of $\theta$, proving the result.
\end{proof}

For uniform time sampling, $\nu(t)=1$ on $(0,1)$,
\[
\int_0^1\log t\,dt=-1,
\]
so
\[
\widehat{\mathcal L}^{\mathrm{lik}}_{\mathrm D}(\theta)
\geq
-\frac{d}{2}.
\]


\xhdr{Single-time \nfmd endpoint control}

We first establish a slightly more general bound.

\begin{mdframed}[style=MyFrame2]
\begin{lemma}[General endpoint bound]
\label{lem:nfmd_endpoint_general}
Let
\[
D_\theta(t)
\coloneqq
\KL\left(
p_t
\,\middle\|\,
p^\theta_{1\to t}
\right),
\qquad
\overline D_\theta
\coloneqq
\int_0^1 D_\theta(t)\nu(t)\,dt.
\]
Assume that there exist $C<\infty$ and $t_0>0$ such that
\[
D_\theta(t)
\geq
D_\theta(0)-Ct,
\qquad
t\in[0,t_0],
\]
and that
\[
\nu(t)\geq\nu_{\min}>0,
\qquad
t\in[0,t_0].
\]
Then
\[
D_\theta(0)
\leq
\inf_{0<h\leq t_0}
\left\{
\frac{\overline D_\theta}
     {\nu_{\min}h}
+
\frac{C}{2}h
\right\}.
\]
\end{lemma}
\end{mdframed}

\begin{proof}
For any $h\in(0,t_0]$,
\begin{align}
\overline D_\theta
&=
\int_0^1 D_\theta(t)\nu(t)\,dt
\nonumber\\
&\geq
\int_0^h D_\theta(t)\nu(t)\,dt
\nonumber\\
&\geq
\nu_{\min}
\int_0^h
\left(
D_\theta(0)-Ct
\right)\,dt
\nonumber\\
&=
\nu_{\min}
\left(
hD_\theta(0)-\frac{C}{2}h^2
\right).
\end{align}
Rearranging gives
\[
D_\theta(0)
\leq
\frac{\overline D_\theta}
     {\nu_{\min}h}
+
\frac{C}{2}h.
\]
Taking the infimum over $h\in(0,t_0]$ proves the result.
\end{proof}

\begin{mdframed}[style=MyFrame2]
\nfmdendpointcontrol*
\end{mdframed}

\begin{proof}
The local Lipschitz assumption implies
\[
D_\theta(t)
\geq
D_\theta(0)-Ct,
\qquad
t\in[0,t_0],
\]
so Lemma~\ref{lem:nfmd_endpoint_general} applies. The unconstrained minimizer of
\[
\frac{\overline D_\theta}
     {\nu_{\min}h}
+
\frac{C}{2}h
\]
is
\[
h^\star
=
\sqrt{
\frac{2\overline D_\theta}
     {C\nu_{\min}}
}.
\]
The assumption
\[
\overline D_\theta
\leq
\frac{C\nu_{\min}t_0^2}{2}
\]
ensures $h^\star\leq t_0$. Substituting $h^\star$ into
Lemma~\ref{lem:nfmd_endpoint_general} yields
\[
D_\theta(0)
\leq
\sqrt{
\frac{2C}{\nu_{\min}}
\overline D_\theta
}.
\]
Finally, using
\[
D_\theta(0)
=
\KL\left(
p_0
\,\middle\|\,
p^\theta_{1\to0}
\right),
\qquad
\overline D_\theta
=
\E_{\gT_{\mathrm S}}
\left[
\KL\left(
p_t
\,\middle\|\,
p^\theta_{1\to t}
\right)
\right],
\]
gives the claimed bound.
\end{proof}


\xhdr{Two-time extension of \nfmd}

For an intermediate source time $s$, the marginal density $p_s$ is
generally intractable. We therefore approximate it using a frozen
single-time \nfmd model,
\[
\bar p_s
\coloneqq
p^{\bar\theta}_{1\to s},
\]
where $\bar\theta$ is held fixed while optimizing the two-time map
$f^\theta_{s,t}$. The corresponding surrogate pushforward is
\[
p^{\theta, \bar \theta}_{s\to t}
\coloneqq
(f^\theta_{s,t})_\#\bar p_s.
\]

For $x_t\sim p_t$, define
\[
\widetilde x_s
\coloneqq
(f^\theta_{s,t})^{-1}(x_t).
\]
By change of variables,
\begin{equation}
\label{eq:nfmd_two_time_density}
p^{\theta, \bar \theta}_{s\to t}(x_t)
=
\bar p_s(\widetilde x_s)
\left|
\det
\frac{\partial(f^\theta_{s,t})^{-1}(x_t)}
     {\partial x_t}
\right|.
\end{equation}
Since
\[
\bar p_s
=
(f^{\bar\theta}_{1,s})_\#p_1,
\]
we can further write
\[
\bar p_s(\widetilde x_s)
=
p_1
\left(
(f^{\bar\theta}_{1,s})^{-1}(\widetilde x_s)
\right)
\left|
\det
\frac{\partial(f^{\bar\theta}_{1,s})^{-1}(x_s)}
     {\partial x_s}
\bigg|_{x_s=\widetilde x_s}
\right|.
\]
Defining
\[
\Phi
\coloneqq
(f^{\bar\theta}_{1,s})^{-1}
\circ
(f^\theta_{s,t})^{-1},
\]
we therefore obtain
\begin{align}
\log p^{\theta, \bar \theta}_{s\to t}(x_t)
=
&
\log p_1\!\left(\Phi(x_t)\right)
\nonumber\\
&
+
\left.
\log
\left|
\det
\frac{\partial(f^{\bar\theta}_{1,s})^{-1}(x_s)}
     {\partial x_s}
\right|_{x_s=\widetilde x_s}
\right|
\nonumber\\
&
+
\log
\left|
\det
\frac{\partial(f^\theta_{s,t})^{-1}(x_t)}
     {\partial x_t}
\right|.
\end{align}

Since $p_1=\mathcal N(0,I)$, the corresponding negative log-likelihood,
up to constants independent of $\theta$, is
\begin{align}
\mathcal L_{\rm D, 2}
=
\mathbb E_{s,t,x_t}
\Bigg[
&
\frac{1}{2}
\left\|
\Phi(x_t)
\right\|^2
\nonumber\\
&
-
\left.
\log
\left|
\det
\frac{\partial(f^{\bar\theta}_{1,s})^{-1}(x_s)}
     {\partial x_s}
\right|_{x_s=\widetilde x_s}
\right|
\nonumber\\
&
-
\log
\left|
\det
\frac{\partial(f^\theta_{s,t})^{-1}(x_t)}
     {\partial x_t}
\right|
\Bigg].
\label{eq:nfmd_two_time}
\end{align}

If the frozen single-time model is exact, such that
\[
\bar p_s=p_s,
\]
then the surrogate pushforward coincides with the ideal pushforward
\[
p^\theta_{s\to t}
=
(f^\theta_{s,t})_\#p_s.
\]
In this case,
\begin{align}
\mathbb E_{x_t\sim p_t}
\left[
-\log p^\theta_{s\to t}(x_t)
\right]
&=
\KL\left(
p_t
\,\middle\|\,
p^\theta_{s\to t}
\right)
+
\mathcal H(p_t),
\end{align}
where $\mathcal H(p_t)$ is independent of $\theta$.

When $\bar p_s$ only approximates $p_s$, the frozen single-time model
instead supplies a tractable surrogate source density. Although
$\bar\theta$ is held fixed, gradients with respect to the input of
$(f^{\bar\theta}_{1,s})^{-1}$ are retained, allowing the loss to update
$f^\theta_{s,t}$ through the full composition.

In practice, the single- and two-time maps are evaluations of the same
time-conditioned architecture at different time pairs. We jointly
optimize
\begin{equation}
\mathcal L_{\mathrm{total}}(\theta;\bar\theta)
=
\mathcal L_{\mathrm D}(\theta)
+
\lambda
\mathcal L_{\rm D,2}(\theta;\bar\theta),
\end{equation}
where the single-time objective learns the marginal probability path
from the tractable reference $p_1$, while the two-time objective extends
this supervision to transport between intermediate times using the
tractable surrogate source density $\bar p_s$ supplied by the frozen
single-time model.

\clearpage
\section{Exact likelihood without Tweedie denoising}


In~\Cref{fig:mnist_mle_comparison}, we compare guided MNIST
samples from an iTARFlow trained via MLE without additive
Gaussian noise and an iTARFlow trained with Gaussian noise
augmentation and Tweedie denoising at inference time.
In this experiment, training without noise produces visibly
poor samples, despite retaining the flow's exact invertibility condition. Noise augmentation followed by denoising
substantially improves sample quality. However, the flow's tractable density describes its outputs
before denoising. The subsequent Tweedie transformation is
not guaranteed to be invertible, and its Jacobian determinant
is not accounted for by the flow's likelihood calculation.
Consequently, the likelihood of the denoised output
distribution is not provided by the original flow.
This comparison illustrates the practical difficulty, in our
setting, of obtaining both high-quality samples and tractable
exact likelihoods for the same output distribution.

\label{app:likelihood_tweedie}
\begin{figure}[h]
    \centering

    \begin{subfigure}[t]{0.32\linewidth}
        \centering
        \includegraphics[
            width=\linewidth
        ]{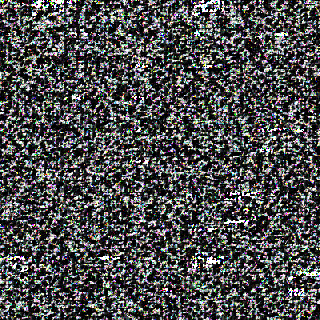}
        \caption{MLE without noise.}
        \label{fig:mnist_mle_no_noise}
    \end{subfigure}
    \hspace{0.04\linewidth}
    \begin{subfigure}[t]{0.32\linewidth}
        \centering
        \includegraphics[
            width=\linewidth
        ]{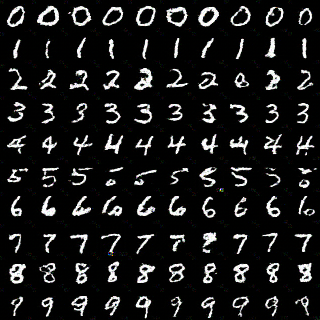}
        \caption{MLE with noise and Tweedie denoising.}
        \label{fig:mnist_mle_tweedie}
    \end{subfigure}

    \caption{Guided MNIST samples from MLE-trained iTARFlow.}
    \label{fig:mnist_mle_comparison}
\end{figure}


\clearpage
\section{Extended background on Boltzmann Generators}
\label{app:BGs}

Sampling from equilibrium distributions is a fundamental problem in statistical
physics~\citep{frenkel1997understanding, tuckerman2023statistical}, since equilibrium configurations provide access to thermodynamic
observables and characterize the macroscopic behavior of physical systems~\citep{pathria2017statistical}. For a configuration $x$ with energy $\mathcal{E}(x)$ at temperature
$T$, the equilibrium distribution is given by the Boltzmann distribution
\begin{equation}
    \mu_{\mathrm{target}}(x)
    =
    \frac{1}{\mathcal{Z}}
    \exp\left(-\beta \mathcal{E}(x)\right),
    \qquad
    \mathcal{Z}
    =
    \int
    \exp\left(-\beta \mathcal{E}(x)\right)\,dx,
    \qquad
    \beta = \frac{1}{k_B T},
\end{equation}
where $\mathcal{Z}$ is the partition function, which is generally intractable
for high-dimensional systems.

A particularly important instance of this problem is the sampling of molecular
configurations~\citep{zuckerman2011equilibrium}. Equilibrium ensembles are required to characterize
thermodynamic and conformational properties of molecular systems and are
relevant to applications including biomolecular modeling and drug discovery~\citep{frenkel1997understanding}. Sampling such systems is challenging because their energy landscapes are
typically high-dimensional and multimodal, with metastable regions separated by
large free-energy barriers. As a consequence, local sampling methods such as
molecular dynamics (MD)~\citep{alder1959studies} and Markov chain Monte Carlo can
exhibit long autocorrelation times and slow transitions between metastable
states~\citep{wirnsberger2020targeted, weinan2002string}.

Boltzmann Generators (BGs)~\citep{noe2019boltzmann} address this challenge by learning a tractable
generative model that approximates the target Boltzmann distribution. In
particular, the model defines a distribution $p_\theta(x)$ such that
\begin{equation}
    p^\theta(x) \approx \mu_{\mathrm{target}}(x),
\end{equation}
while providing both a mechanism for efficiently drawing samples
$x \sim p^\theta$ and a tractable evaluation of their density
$p^\theta(x)$. The former enables direct generation of global proposals in
configuration space, while the latter enables the use of Monte Carlo correction
methods to account for discrepancies between the learned proposal distribution
and the target distribution.

One such correction is self-normalized importance sampling (SNIS)~\citep{liu2001monte}. For an
observable $\phi(x)$, its expectation under the target distribution can be
written as
\begin{equation}
    \mathbb{E}_{\mu_{\text{target}}}[\phi(x)]
    =
    \frac{
        \mathbb{E}_{p^\theta}
        \left[
            w(x)\phi(x)
        \right]
    }{
        \mathbb{E}_{p^\theta}
        \left[
            w(x)
        \right]
    },
    \qquad
    w(x)
    =
    \frac{
        \exp\left(-\beta \mathcal{E}(x)\right)
    }{
        p^\theta(x)
    }.
\end{equation}
The unknown partition function $\mathcal{Z}$ cancels in the ratio. Given
$K$ independent samples $x^1,\ldots,x^K \sim p^\theta$, the expectation can be
approximated as
\begin{equation}
    \mathbb{E}_{\mu_{\text{target}}}[\phi(x)]
    \approx
    \sum_{i=1}^{K}
    \bar{w}^i \phi(x_i),
    \qquad
    \bar{w}^i
    =
    \frac{
        w(x^i)
    }{
        \sum_{j=1}^{K} w(x^j)
    }.
\end{equation}

Beyond importance sampling, the learned model can also be combined with other
Monte Carlo correction schemes. In particular, sequential Monte Carlo methods~\citep{tan2025scalable}
can use the generative model as a proposal while progressively bridging between
the learned and target distributions through intermediate distributions and
resampling steps. We refer to the appendix of~\citep{tan2025amortized} for further discussion.

\subsection{Baselines}
\label{app:BGs_baselines}
We compare \nameshort against a broad set of baselines covering both
single-system and transferable molecular generative modeling. In the
single-peptide setting, we include equivariant CNFs~\citep{klein2023equivariant, klein2024transferable} (ECNF++~\citep{tan2025scalable}),
RegFlow~\citep{regflow}, flow map based variants, FALCON~\citep{rehman2026falcon}, F2D2~\citep{ai2026joint}, and
SCALLOP~\citep{scallop_2026}. We further compare against models trained by maximum likelihood,
including SBG~\citep{tan2025scalable} and ArBG~\citep{rehman2026autoregressive}. In the transferable setting, we benchmark against Timewarp~\citep{klein2023timewarp}, BioEmu~\citep{lewis2025scalable}, UniSim~\citep{yu2025unisim}, TarFlow~\citep{zhai2025normalizing, tan2025amortized}, Prose~\citep{tan2025amortized}, and Robin~\citep{rehman2026autoregressive}.

\subsection{Metrics}
\label{app:BGs_metrics}
We evaluate \nameshort using three complementary metrics that compare generated samples against reference samples obtained from MD. The energy-based metric probes local geometric fidelity, while the torus- and TICA-based metrics assess larger-scale conformational structure and global distributional agreement.

\xhdr{Energy 2-Wasserstein Distance ($\mathcal{E}-\mathcal{W}_2$)}
To compare the energy distributions induced by generated and reference
configurations, we compute the 2-Wasserstein distance between their respective
energy marginals. Let $\rho,\nu \in \mathcal{P}(\mathbb{R})$ denote the
distributions of energy values under the generated and reference ensembles,
respectively, and let $\Pi(\rho,\nu)$ denote the set of couplings with marginals
$\rho$ and $\nu$. We define
\begin{equation}
    \mathcal{E}\text{-}\mathcal{W}_2(\rho,\nu)
    :=
    \left(
    \inf_{\pi \in \Pi(\rho,\nu)}
    \int_{\mathbb{R}\times\mathbb{R}}
    |u-v|^2 \, d\pi(u,v)
    \right)^{1/2}.
\end{equation}
This metric measures discrepancies in the distribution of molecular energies
and is therefore sensitive to local structural distortions that affect the
potential energy, such as changes in bond lengths, bond angles, or short-range
interactions.

\xhdr{Torus 2-Wasserstein Distance ($\mathbb{T}$-$\mathcal{W}_2$)}
To compare the global conformational structure of generated and reference
ensembles, we measure discrepancies in the space of backbone dihedral angles.
For a peptide with $L$ residues, we represent a configuration $x$ by its vector of dihedral angles
\begin{equation}
    \text{Dihedrals}(x)
    =
    \big(
        \phi_1,\psi_1,\ldots,\phi_{L-1},\psi_{L-1}
    \big),
\end{equation}
where $\phi_i$ and $\psi_i$ denote the backbone torsion angles. Since these
variables are periodic, Euclidean distances between angles are not appropriate.
We therefore define the squared distance between two configurations using the
shortest angular displacement on the torus,
\begin{equation}
    c_{\mathbb{T}}(x,y)^2
    =
    \sum_{i=1}^{2(L-1)}
    \left[
        \big(\text{Dihedrals}(x)_i-\text{Dihedrals}(y)_i+\pi\big)\bmod 2\pi-\pi
    \right]^2.
\end{equation}

Let $\rho$ and $\nu$ denote the distributions of backbone dihedral angles
induced by the generated and reference ensembles, respectively. Their torus
2-Wasserstein distance is defined as
\begin{equation}
    \mathbb{T}\text{-}\mathcal{W}_2(\rho,\nu)
    =
    \left(
    \inf_{\pi\in\Pi(\rho,\nu)}
    \int
    c_{\mathbb{T}}(x,y)^2
    \,d\pi(x,y)
    \right)^{1/2}.
\end{equation}
By comparing configurations in torsional space, this metric captures differences in large-scale conformational structure and is particularly informative about mode coverage of the generated
ensemble.

\xhdr{TICA 2-Wasserstein Distance ($\mathrm{TICA}$-$\mathcal{W}_2$)}
To compare the generated and reference ensembles along their slow collective
coordinates, we project configurations onto the leading components obtained by
time-lagged independent component analysis (TICA). TICA identifies
directions of maximal time-lagged autocorrelation from a reference trajectory.

Let $w_1,w_2$ denote the first two TICA projection vectors. We define the
squared distance between two configurations in TICA space as
\begin{equation}
    c_{\mathrm{TICA}}(x,y)^2
    =
    \sum_{j=1}^{2}
    \left(
        w_j^\top x - w_j^\top y
    \right)^2.
\end{equation}
The corresponding 2-Wasserstein distance between the generated and reference
distributions $\rho$ and $\nu$ is
\begin{equation}
    \mathrm{TICA}\text{-}\mathcal{W}_2(\rho,\nu)
    =
    \left(
    \inf_{\pi\in\Pi(\rho,\nu)}
    \int
    c_{\mathrm{TICA}}(x,y)^2
    \,d\pi(x,y)
    \right)^{1/2}.
\end{equation}
This metric measures agreement in the slow conformational subspace and is
therefore informative about large-scale metastable structure.
\clearpage
\section{Implementation details}
\label{app:c}

\subsection{Transformer Autoregressive Flows}
\label{app:tarflows}

\nameshorts are compatible with any NF architecture, however, we focus on Transformer Autoregressive Flows (TARFlows) as they have been established as the state-of-the-art NF architectue in terms of sample fidelity and scalability. TARFlows construct exactly-invertible networks by composing affine transformations parametrized by causal transformers. Except for the first block, each block reverses the sequence, alternating the autoregressive direction. In the interest of concision we omit these permutations below; within block $b$, entry $z_j$'s updated value is given by:
\begin{equation}
z^{b+1}_j = (z^{b}_j-\mu^{b}_{\theta}(z^{b}_{<j}))\odot\exp(-\alpha^{b}_{\theta}(z^{b}_{<j}))
\end{equation}
where $\mu^{b}_{\theta}(z^{b}_j)$ and $\alpha^{b}_{\theta}(z^{b}_{<j})$ are obtained by processing the sequence elements proceeding $z_j$ with a causal transformer. It follows immediately that this transformation may be reversed by autoregressively rolling out,
\begin{equation}
    z^{b}_{j}=z^{b+1}_j\odot\exp(\alpha^{b}_{\theta}(z^{b}_{<j}))+\mu^{b}_{\theta}(z^{b}_{<j})
\end{equation}

\subsection{NFM Architecture}
\label{app:architecture}
Both \nfmd and \nfmr adopt modified versions of the iTARFlow architecture~\citep{lu2026bidirectional}. \nfmd introduces time-conditioning via a sinusoidal timestep embedding module and injects this conditioning along with the class embedding into both the attention and the adaptive layernorm~\citep{peebles2023scalable} modules within each block. 

\nfmr adapts iTARFlow to more closely follow the iMF architecture. The major changes are completely eliminating adaptive layer norm modules in favor of in-context conditioning and 
adopting general-purpose Transformer improvements: SwiGLU~\citep{shazeer2020glu}, RMSnorm~\citep{zhang2019root}, RoPE~\citep{su2024roformer} and QK-normalization.

\xhdr{In-context conditioning}
Following iMF we abandon adaLN-zero~\citep{peebles2023scalable} in favor of in-context conditioning with \emph{multiple} tokens for each condition. This change is particularly helpful given the need for our image models to condition on class, two timesteps and several CFG related conditions (c.f.~\S~\ref{app:imf_deriv_and_eulerian_connection}), something which is better handled by introducing all conditioning into attention than simply summing their embeddings in an adaLN-zero module. Furthermore, adaLN-zero is known to be parameter-heavy~\citep{geng2026improved,lu2026bidirectional}; so this change potentially enables \nfmr to be more parameter efficient. Key to making in-context conditioning competitive with adaLN-zero is the use of multiple tokens per condition. Given a condition $\alpha\in\mathbb{R}$, we jointly learn its embedding map $g_{\theta}:\mathbb{R}\mapsto \mathbb{R}^{n}$ and $m$ embedding offsets $\{\beta_i\}^{m}_{i=1},\beta_i\in\mathbb{R}^n$. Then $\alpha$'s in-context conditioning tokens are given by

\begin{equation}
    \{g_{\theta}(\alpha)+\beta_1, \dots, g_{\theta}(\alpha)+\beta_m\}
\end{equation}

\subsection{Algorithmic aspects}
\label{app:algorithmic_aspects}
\looseness=-1
\subsubsection{Improved Mean Flow objective}
The Mean Flow~\citep{geng2025mean}, improved Mean Flow~\citep{geng2026improved} and pixel Mean Flow~\citep{lu2026onestep} works all derive their objectives from the Mean Flow identity:
\begin{equation}
\label{eq:mean_flow_identity}
v_{s,t}(x_s)=v_{s}(x_s)+(t-s)\frac{d}{ds}v_{s,t}(x_s),    
\end{equation}
where $v_{s,t}(x_s)$ has the usual definition of average velocity $v_{s,t}(x_s)=\frac{1}{t-s}\int_{s}^{t}v_\tau(x_\tau) d\tau$. The identity can be derived by beginning with this definition and applying standard rules of Calculus,
\begin{align}
    &(t-s)v_{s,t}(x_s) = \int_{s}^{t}v_\tau(x_\tau) d\tau\\
    &\frac{d}{ds}(t-s)v_{s,t}(x_s)=\frac{d}{ds}\int_{s}^{t}v_\tau(x_\tau) d\tau\\
    &-v_{s,t}(x_s)+(t-s)\frac{d}{ds}v_{s,t}(x_s)=-v_s(x_s)\\
    &v_{s,t}(x_s)=v_{s}(x_s)+(t-s)\frac{d}{ds}v_{s,t}(x_s)
\end{align}
Expanding $\frac{d}{ds}v_{s,t}(x_s)$ we obtain,
\begin{equation}
    \label{eq:expanded_mean_flow_identity}
    v_{s,t}(x_s)=v_{s}(x_s)+(t-s)(\partial_sv_{s,t}(x_s)+\nabla v_{s,t}(x_s)v_{s}(x_s))
\end{equation}
Equation~\ref{eq:expanded_mean_flow_identity} suggests that we may learn $v_{s,t}(x_s)$ by using stopgrads and matching the RHS. However, when training from scratch we do not have access to the underlying $v_s(x_s)$. In improved Mean Flow the outer $v_s(x_s)$ is replaced by the \emph{sample-conditional} instantaneous velocity\footnote{We adopt a linear interpolant in this work but this objective supports any stochastic interpolant.} and the inner $v_s(x_s)$ by $v^{\theta}_{s,s}(x)$ following the boundary condition,
\begin{equation}
\mathcal{L}_{\rm{iMF}}=\mathbb{E}_{s,t,x_0,x_1}\left[\lVert(x_1-x_0)+\mathrm{sg}((t-s)(\partial_s v^{\theta}_{s,t}(x_s)+\nabla v^{\theta}_{s,t}(x_s)v^{\theta}_{s,s}(x_s))) -v^{\theta}_{s,t}(x_s)\rVert^{2}\right]
\end{equation}

\looseness=-1
\subsubsection{Conditioning and guidance} \nameshorts can be made conditional by incorporating additional arguments (e.g., class label). \nfmd follows TARFlow-style~\citep{zhai2025normalizing} guidance; training with conditioning-dropout, and at inference performs \emph{online} guidance \emph{directly on the sample} (\S~\ref{app:nfmd_guidance}). In contrast, \nfmr utilizes iMF-style guidance wherein the average velocity $v_{s,t}$ is conditional on the classifier-free guidance (CFG) interval~\citep{ho2022classifier,kynkaanniemi2024applying} $[t_{\rm{min}},t_{\rm{max}}]$ and guidance strength $\omega$. For training, this entails adjusting $\mathcal{L}_R$ from \S\ref{sec:distillation_losses}  with guidance of varying $\omega$, $t_{\rm{min}}$ and $t_{\rm{max}}$ (\S~\ref{app:nfmr_guidance}). At inference time, this enables few-step sampling with flexible CFG settings at no additional cost.

\xhdr{NFM-R guidance}
\label{app:nfmr_guidance}
\nfmr utilizes the flexible guidance introduced in iMF. The original Mean Flow~\citep{geng2025mean} derived a relationship between the average velocity and the guided instantaneous velocity typically employed in diffusion and flow matching work~\citep{ho2022classifier}. We begin by defining the guided instantaneous velocity for fixed $\omega$:
\begin{equation}
    v_t^{\rm{cfg}}(x_t|c)\coloneqq \omega v_t(x_t|c) + (1-\omega)v^{u}(x_t,t),
\end{equation}
which is a linear combination of class-conditional and class-unconditional instantaneous velocities:
\begin{equation}
    v_t(x_t|c)\coloneqq\mathbb{E}_{p(x_0,x_1|x_t,c)}[x_1-x_0]\text{ and } v^{u}_t(x_t)\coloneqq \mathbb{E}_c[v_t(x_t|c)]
\end{equation}
Next, from the MeanFlow Identity (Equation~\ref{eq:mean_flow_identity}), it follows that the guided average velocity satisfies,
\begin{equation}
    \label{eq:guided_mean_flow_identity}
    v^{\rm{cfg}}_{s,t}(x_s|c)=v^{\rm{cfg}}_s(x_s|c)+(t-s)\frac{d}{ds}v^{\rm{cfg}}_{s,t}(x_s|c) 
\end{equation}
Hence, it follows that $v_{s,s}^{\rm{cfg}}(x_s|c)=v^{\rm{cfg}}_{s}(x_s|c)$. Next, define $v^{\rm{cfg}}_{s,t}(x_s)\coloneqq\mathbb{E}_c[v^{\rm{cfg}}_{s,t}(x_s|c)]$. Then by leveraging $v^{u}_s(x_s)=\mathbb{E}_c[v^{\rm{cfg}}_s(x_s|c)]$, we have that $v^{\rm{cfg}}_{s,s}(x_s)=v^{u}_{s}(x_s)$. Therefore,

\begin{align}
    &v^{\rm{cfg}}_s(x_s|c)-v^{\rm{cfg}}_{s,s}(x_s|c)=v^{u}_{s}(x_s)-v^{\rm{cfg}}_{s,s}(x_s)\\
    &v^{\rm{cfg}}_{s}(x_s|c)-v^{u}_{s}(x_s)=v^{\rm{cfg}}_{s,s}(x_s|c)-v^{\rm{cfg}}_{s,s}(x_s)\\
    &\omega(v_s(x_s|c)-v^{u}_s(x_s,s))=v^{\rm{cfg}}_{s,s}(x_s|c)-v^{\rm{cfg}}_{s,s}(x_s)\\
    &v_s(x_s|c)-v^{u}_s(x_s,s)=\frac{1}{\omega}(v^{\rm{cfg}}_{s,s}(x_s|c)-v^{\rm{cfg}}_{s,s}(x_s))\\
    &v^{u}_t(x_s,s)=v_s(x_s|c)-\frac{1}{\omega}(v^{\rm{cfg}}_{s,s}(x_s|c)-v^{\rm{cfg}}_{s,s}(x_s))\\
    &v^{u}_s(x_s,s)+(v^{\rm{cfg}}_{s,s}(x_s|c)-v^{\rm{cfg}}_{s,s}(x_s))=v_s(x_s|c)+(1-\frac{1}{\omega})(v^{\rm{cfg}}_{s,s}(x_s|c)-v^{\rm{cfg}}_{s,s}(x_s))\\
    &v^{u}_s(x_s,s)+(v^{\rm{cfg}}_{s}(x_s|c)-v^{u}_s(x_s,s))=v_t(x_s|c)+(1-\frac{1}{\omega})(v^{\rm{cfg}}_{s,s}(x_s|c)-v^{\rm{cfg}}_{s,s}(x_s))\\
    \label{eq:v_cfg_in_terms_of_v_avg}
    &v^{\rm{cfg}}_s(x_s|c)=v_s(x_s|c)+(1-\frac{1}{\omega})(v^{\rm{cfg}}_{s,s}(x_s|c)-v^{\rm{cfg}}_{s,s}(x_s))
\end{align}
Equation~\ref{eq:v_cfg_in_terms_of_v_avg} allows the guided \emph{marginal} instantaneous velocity to be targeted by utilizing the model's own current estimates of $v^{\rm{cfg}}_{s,s}(x_s|c)$ and $v^{\rm{cfg}}_{s,s}(x_s)$, where $v^{\rm{cfg}}_{s,s}(x_s)\approx v^{\theta}_{s,s}(x_s|\emptyset)$.  Furthermore, it can be beneficial to only apply guidance within a specific time interval $[s_{\rm{min}}, s_{\rm{max}}]$~\citep{kynkaanniemi2024applying},
\begin{equation}
    v_s^{\rm{cfg}}(x_s|c) = 
    \begin{cases}
    v_s(x_s|c)+(1-\frac{1}{\omega})(v^{\rm{cfg}}_{s,s}(x_s|c)-v^{\rm{cfg}}_{s,s}(x_s)), & s\in[s_{\rm{min}},s_{\rm{max}}]\\
    v_{s}(x_s|c),&\text{otherwise}.
    \end{cases}
\end{equation}

In order to enable flexible guidance at inference time, iMF makes the network conditional on $\omega$, $s_{\rm{min}}$ and $s_{\rm{max}}$. During training $\omega$ is sampled from a power distribution $\omega\sim p(\omega)\propto \omega^{-\beta}$, $s_{\rm{min}}\sim \mathcal{U}[0.0, 0.5]$ and $s_{\rm{max}}\sim\mathcal{U}[0.5,1.0]$.  For brevity we use $\Omega\coloneqq\{\omega, s_{\rm{min}}, s_{\rm{max}}\}$ to denote all the cfg conditioning. From Equations~\ref{eq:guided_mean_flow_identity} and \ref{eq:v_cfg_in_terms_of_v_avg} we can derive
\begin{align}
    &v^{\rm{cfg}}_{s}(x_s|c,\Omega)-v^{\rm{cfg}}_{s,t}(x_s|c,\Omega)+(t-s)\frac{d}{ds}v^{\rm{cfg}}_{s,t}(x_s|c,\Omega)=0,\\
    &v_s(x_s|c,\Omega)+(1-\frac{1}{\omega})(v^{\rm{cfg}}_{s,s}(x_s|c,\Omega)-v^{\rm{cfg}}_{s,s}(x_s))-v^{\rm{cfg}}_{s,t}(x_s|c,\Omega)+(t-s)\frac{d}{ds}v^{\rm{cfg}}_{s,t}(x_s|c,\Omega)=0.
\end{align}
Replacing $v_s(x_s|c,\Omega)$ with the \emph{sample-conditional} instantaneous velocity and $v_{s,s}(x_s|c,\Omega)$ as in Appendix~\ref{app:imf_deriv_and_eulerian_connection} and applying stopgrads we obtain the guided objective,
\begin{align}
\mathcal{L}^{\rm{cfg}}_{\rm{iMF}}=\mathbb{E}_{s,t,x_0, x_1, c,\Omega}[\lVert(x_1-x_0)&+(1-\frac{1}{\omega}){\rm{sg}}((v^{\theta}_{s,s}(x_s|c,\Omega)-v^{\theta}_{s,s}(x_s))\\
&+{\rm{sg}}((t-s)\partial_s v^{\theta}_{s,t}(x_s|c,\Omega)\\
&+(t-s)\nabla v^{\theta}_{s,t}(x_s|c,\Omega)v^{\theta}_{s,s}(x_s|c,\Omega))\\
&-v^{\theta}_{s,t}(x_s|c,\Omega)\rVert^2]
\end{align}

\xhdr{NFM-D Guidance}
\label{app:nfmd_guidance}
\nfmd utilizes the guidance introduced in TARFlow where the parameters of each affine transformation are a linear combination of guided and unguided $\mu$ and $\alpha$. Given conditioning $c$, null conditioning $\emptyset$ and guidance strength $w$, the guided $\tilde{z}^{t}_{i}$ (block $t$, sequence position $i$) is computed via
\begin{equation}
    \tilde{z}_{j}^{b}= z_j^{b+1}\odot\exp(\tilde{\alpha}^{b}_{j}(\tilde{z}^{b}_{<j};c,w))+\tilde{\mu}_j^{b}(\tilde{z}^b_{<j};c,w).
\end{equation}
Where the guided affine parameters $\tilde{\mu}_{i}^t$ and $\tilde{\alpha}_i^t$ are given by, 
\begin{align}
    &\tilde{\mu}^{b}_{j}(\tilde{z}^{b}_{<j};c,w)=(1+w)\mu^{b}_j(\tilde{z}^{t}_{<j};c)-w\mu_{j}^{b}(\tilde{z}^{b}_{<j};\emptyset),\\
    &\tilde{\alpha}^{b}_j(\tilde{z}^{b}_{<j};c,w)=(1+w)\alpha^b_{j}(\tilde{z}^{b}_{<j};c) - w\alpha^{b}_{j}(\tilde{z}^{b}_{<j};\emptyset).
\end{align}

\subsubsection{Numerical invertibility}
\looseness=-1
\xhdr{Noise annealing and \texttt{FP64} inference}
Theoretically, \nfmr should be exactly invertible; however, in practice, its regression objective combined with the denoiser parametrization can lead to a loss of \emph{numerical} invertibility. As a mitigation strategy, we apply additive Gaussian noise to $x_0$ before computing the interpolant and anneal its contribution to zero over the course of training (see \S\ref{subsec:pseudocode}). Through this strategy, it is possible to recover numerical invertibility in \texttt{FP64} at inference. This strategy enabled the \nfmr \cifar to be finetuned with REINFORCE to preferentially produce green frogs (\S~\ref{app:rl_results}). 

\cut{
\subsubsection{\nameshort parametrization (C): Gated}
\label{app:gated_param}
When the base NF is composed of affine coupling transformations as in TARFlows, the identity condition can be enforced while retaining invertibility for every time pair through gating with $(s-t)$, 

\begin{equation}
\label{eq:nfm_gated_param}
(\mathrm{C}):\qquad
\left[f_{s,t}^{\theta,b}\!\left(z^b\right)\right]_j
=
\left(
z_j^b
-
(s-t)\mu_\theta^b\!\left(z_{<j}^b; s,t\right)
\right)
\odot
\exp\!\left(
-(s-t)\alpha_\theta^b\!\left(z_{<j}^b; s,t\right)
\right).
\end{equation}
\looseness=-1
We term this the \emph{gated dual-time NFM}, as each layer $i$ reduces to the identity at $s=t$ while remaining explicitly invertible.
}
\subsection{Pseudocode}
\label{subsec:pseudocode}
\definecolor{nfmCodeText}{RGB}{25,35,65}
\definecolor{nfmCodeComment}{RGB}{95,155,170}
\definecolor{nfmCodeAccent}{RGB}{245,45,140}
\definecolor{nfmCodeOperator}{RGB}{80,45,255}

\lstdefinestyle{nfmPython}{
    language=Python,
    basicstyle=\ttfamily\footnotesize\color{nfmCodeText},
    keywordstyle=\color{nfmCodeAccent},
    commentstyle=\color{nfmCodeComment},
    stringstyle=\color{nfmCodeAccent},
    morekeywords={sample_reflow_batch,sample_times,sample_cfg,
                  randn_like,jvp,stopgrad,reweight,auxloss, freeze, gamma, aux_perceptual_losses},
    columns=fullflexible,
    keepspaces=true,
    showstringspaces=false,
    numbers=none,
    frame=none,
    breaklines=true,
    tabsize=4,
    aboveskip=0.3em,
    belowskip=0.3em,
    xleftmargin=0.3em,
    literate=*{+}{{{\color{nfmCodeOperator}+}}}1
              {-}{{{\color{nfmCodeOperator}-}}}1
              {*}{{{\color{nfmCodeOperator}*}}}1
              {/}{{{\color{nfmCodeOperator}/}}}1
              {=}{{{\color{nfmCodeOperator}=}}}1
}

\begin{algorithm}[h!]
    \caption[NFM-R: training loss]{NFM-R: training loss.\newline
        \normalfont\footnotesize
        Note: \texttt{jvp} returns the function output and its JVP.}
    \label{alg:nfmr_training}

\begin{lstlisting}[style=nfmPython]
# net: denoiser NF network
# x,c: training and condition batch


# ReFlow-paired data x0 and noise x1
x0, x1, c = sample_reflow_batch()
sigma = sigma_init
if anneal:
    frac = min((step // steps_per_epoch) / anneal_epochs, 1)
    sigma = (1 - frac) * (sigma_init - sigma_min) + sigma_min
x0 = x0 + sigma * randn_like(x0)

s, t, diag = sample_times()
w, a, b = sample_cfg(diag)
z = (1 - s) * x0 + s * x1
v_pair = x1 - x0

# Average velocity from data prediction
def v_fn(z, s, t, w, c, a, b):
    return (z - net(z, s, t, w, c, a, b)) / s

# Unfiltered and interval-filtered instantaneous velocities
v_c = v_fn(z, s, s, w, c, a, b)
v_u = v_fn(z, s, s, w, None, 0, 1)
w_s = where((a <= s) & (s <= b), w, 1)
v_s = v_fn(z, s, s, w_s, c, 0, 1)
v_target = where(diag, v_c, v_s)
v_g = v_pair + (1 - 1 / w) * (v_target - v_u)

# Class dropout; retain the conditional JVP tangent
n_drop = sum(bernoulli(p_drop, len(c)))
drop = arange(len(c)) < n_drop
c_train = where(drop, null, c)
v_g = where(drop, v_pair, v_g)

# JVP in z and s; hold t, w, c_train, a, b fixed
v, dv_ds = jvp(v_fn, (z, s, t, w, c_train, a, b),(v_s, 1, 0, 0, 0, 0, 0))
V = v + (s - t) * stopgrad(dv_ds)
loss =  metric(V, stopgrad(v_g)) + aux_perceptual_losses(z,s,s,w,c,a,b)

\end{lstlisting}
\end{algorithm}

\begin{algorithm}[t]
    \caption[One-time NFM-D: training loss]{One-time NFM-D: training loss.\newline
        \normalfont\footnotesize
        Note: the encoding destination is fixed at $s=1$ (noise).}
    \label{alg:nfmd_one_time_training}

\begin{lstlisting}[style=nfmPython,
    morekeywords={sample_batch,uniform,bernoulli,mean}]
# net: invertible nf.
# x,c: training and condition batch

# Data batch and independent Gaussian noise
x0, c = sample_batch()
x1 = randn_like(x0)
t = uniform(0, t_max, len(c))
x_t = (1 - t) * x0 + t * x1

# Label dropout for classifier-free guidance
mask = bernoulli(p_drop, len(c))
c = where(mask, null, c)

# Map the interpolant to the noise distribution at s=1.
z, logdet = net.invert(x_t, 1, t, c)

# Gaussian prior term and log-Jacobian correction
loss = 0.5 * mean(z**2) - mean(logdet)
\end{lstlisting}
\end{algorithm}

\cut{
\begin{algorithm}[t]
    \caption[Two-time NFM-D: training loss]{Two-time NFM-D: training loss.}
    \label{alg:nfmd_two_time_training}

\begin{lstlisting}[style=nfmPython,
    morekeywords={sample_batch,uniform,bernoulli,mean}]
# net: invertible nf.
# x,c: training and condition batch
# gamma: one-time loss weighting

# Data batch and independent Gaussian noise
x0, c = sample_batch()
x1 = randn_like(x0)
s,t = sample_times()
x_s = (1 - s) * x0 + s * x1

# Label dropout for classifier-free guidance
mask = bernoulli(p_drop, len(c))
c = where(mask, null, c)


#Map from s to 1:
z, log_det_s_1 = net.invert(x_s, 1, s, c)
one_time_loss = 0.5 * mean(z**2) - mean(log_det_s_1)

#Map from s to t:
pred_x_t, log_det_s_t = net(x_s, s, t, c)

#Freeze and map t to 1:
frozen_net = freeze(net)
y, log_det_t_1 = frozen_net.invert(pred_x_t, 1, t, c)
two_time_loss = 0.5 * mean(y**2) - mean(log_det_t_1) - mean(log_det_s_t)

loss = gamma*one_time_loss + two_time_loss
\end{lstlisting}
\end{algorithm}
}

\begin{algorithm}[t]
    \caption[Two-time NFM-D: training loss]{Two-time NFM-D: training loss.}
    \label{alg:nfmd_two_time_training}

\begin{lstlisting}[style=nfmPython,
    morekeywords={sample_batch,uniform,bernoulli,mean}]
# net: invertible nf.
# x,c: training and condition batch
# gamma: one-time loss weighting

# Data batch and label dropout for classifier-free guidance
x0, c = sample_batch()
mask = bernoulli(p_drop, len(c))
c = where(mask, null, c)

# Direct objective
u = uniform(0, 1, len(c))
x_u = (1 - u) * x0 + u * randn_like(x0)
z, logdet_u1 = net.invert(x_u, 1, u, c, params=theta)
one_time_loss = 0.5 * mean(z**2) - mean(logdet_u1)

# Composed objective
s,t = sample_times() # s < t
x_s = (1 - s) * x0 + s * randn_like(x0)
x_t, logdet_st = net.invert(x_s, t, s, c, params=theta)

# Detach current parameters; retain input gradients.
theta_ref = stopgrad_parameters(theta)
z, logdet_t1 = net.invert(x_t, 1, t, c, params=theta_ref)
two_time_loss = (
    0.5 * mean(z**2) - mean(logdet_st) - mean(logdet_t1)
)

loss = 0.5 * one_time_loss + 0.5 * two_time_loss
\end{lstlisting}
\end{algorithm}

\clearpage

\subsection{Reinforcement Learning Finetuning}
\label{app:rl_finetuning}
We perform \emph{on-policy} reinforcement learning finetuning for both \nfmd and \nfmr. When \nameshort models are numerically invertible, we are able to train \emph{without} reward gradients and with an exact KL-penalty computed against the untuned model. This is achieved via a KL-penalized REINFORCE with a leave-one-out baseline (RLOO), thus relying only on $\nabla_{\theta}\log p_{1\to 0}^{\theta}(x)$:
\begin{align}
    \label{eq:rloo_reinforce}
    &J^{\rm{RLOO}}(\theta)=\frac{1}{N}\sum_{i=1}^{N} \log p_{1\to0}^{\theta}(x_i)\rm{sg}(A_i)\\
    &\text{with } A_i = H_i - \frac{1}{n-1}\sum_{j\neq i}H_j\\
    &\text{and } H_i = R(x_i) - \beta(\log p_{1\to 0}^{\theta}(x_i)-\log p^{\theta}_{\mathrm{ref}}(x_i)),\,\beta>0.
\end{align}

When \nfmr becomes numerically non-invertible, the use of the score and KL-penalties becomes infeasible. Instead, we utilize reward gradients and replace the KL-constraint with maximum entropy regularization,
\begin{align}
    &J^{\rm{max ent}}(\theta)=\frac{1}{N}\sum_{i=1}^{N}R(x_i)-\beta\log p_{1\to0}^{\theta}(x_i),\,\beta>0\\
    &\nabla_{\theta}J^{\rm{maxent}}(\theta)=\frac{1}{N}\sum_{i=1}^{N}\nabla_{\theta}[R(x_i)-\beta\log p_{1\to 0}^{\theta}(x_i)]
\end{align}\

\xhdr{CIFAR-10}
For our \cifar RL experiments we train \nfmr and \nfmd to preferentially produce green frogs. We use the $J^{\rm{RLOO}}(\theta)$ objective to train the final bias in each iTARFlow block for 5 steps. Noise-annealed \nfmr uses learning rate $2\times 10^{-2}$, $\beta=1\times 10^{-6}$ and \texttt{FP64} precision and trains without guidance. \nfmd uses learning rate $1\times 10^{-2}$, $\beta=1$, \texttt{FP32} precision and optimizes the guided policy directly. The reward used to incentivize green frogs is:
\begin{equation}
    r_{\rm{green}}(x)=\rm{sigmoid}\left(\frac{g_{\mathrm{score}}(x)-\mu_g}{\sigma_g}\right)
\end{equation}
where $g_{\mathrm{score}}(x)$ center-crops the image and computes the 'green score' via the pixel average of 
\begin{equation}
    2\cdot x_{\rm{green\_channel}}-x_{\rm{red\_channel}}-x_{\rm{blue\_channel}}.
\end{equation}
Finally, $\mu_g$ and $\sigma_g$ are the empirical average and standard deviation of the green score before post-training. See \Cref{fig:green-frogs} for before and after.

\xhdr{ImageNet-64}
The \imnet experiments use a mixture of different rewards; \nfmr uses $R(x)=\rm{Skywork}(x) + 0.1\times\rm{ImageReward}(x)$ and \nfmd uses $R(x)=\rm{ImageReward}(x)+0.1\times\rm{Aesthetic Score}(x) +0.2\times\rm{PickScore}(x)$. \nfmr uses $J^{\rm{max ent}}(\theta)$ with $\beta=0.2$, learning rate $3\times10^{-6}$, batch size $256$, \texttt{FP32} precision and trains without guidance for 40 steps. \nfmd uses $J^{\rm{RLOO}}(\theta)$ with $\beta=0.01$, learning rate $1\times 10^{-6}$, batch size $448$, \texttt{FP32} and trains without guidance for 20 steps. See \Cref{fig:before_and_after_rl} for before and after, and \Cref{fig:nfmd_b2_dog_metrics,fig:nfmr_rl_sheep_metrics} for metrics over the course of training.

\clearpage
\section{Additional Results}
\label{app:results}

\xhdr{Hardware}
Experiments were conducted on NVIDIA L40S, H100, GH200, RTX Pro 6000 Blackwell, and RTX Pro 5000 Blackwell.

\subsection{Image Generation}
\label{app:image_generation}
\xhdr{Datasets} We evaluate our models on both CIFAR-10 and ImageNet-64$\times$64. For \nfmr we use ReFlow couplings from EDM~\citep{karras2022elucidating} for \cifar and from EDM 2~\citep{karras2024analyzing} for \imnet.

\xhdr{Architecture} \nameshorts implement modified iTARFlow architectures (\S~\ref{app:architecture}). All model size configurations are based off the original iTARFlow and the \nfmr recipe follows pixel Mean Flow (pMF)~\citep{lu2026onestep}.

\begin{table}[h]
\centering
\caption{Architectural hyperparameters used for the image experiments.}
\label{tab:image_architecture_configurations}
\begin{tabular}{lccccc}
\toprule
System & Patch Size& Blocks & Layers & Channels & Parameters (M) \\
\midrule
\cifar      &2& 8 & 6 & 256 & $~40$M\\
\midrule
\imnet      &2& 8 & 8 & 384 & $~120$M\\
\bottomrule
\end{tabular}
\end{table}

\xhdr{Training configuration} Following pMF, we use the Muon optimizer with Muon and AdamW weight decays of 0.1 and $1\times10^{-4}$ respectively, and AdamW parameters $(\beta_1,\beta_2)=(0.9,0.95)$. For MLE iTARFlow we ablated the choice of optimizer~\Cref{tab:mle_itarflow_muon_cifar} on \cifar and used AdamW on both datasets. We train the MLE iTARFlow, \nfmd and noise annealed \nfmr \cifar models for 960 epochs adopting a constant learning rate of $4\times 10^{-4}$. The noise annealed \nfmr begins with $\sigma=0.5$ noise injected into the data, and this is annealed to $0$ over the first 300 epochs. The numerically non-invertible \nfmr model follows a different training strategy; first we train solely with a flow matching loss (equivalently, only sampling $s=t$ in \Cref{alg:nfmr_training}) for 450 epochs at learning rate of $4\times10^{-4}$, next we resume training from EMA weights with the full iMF loss for 200 epochs with the same learning rate. This strategy is substantially cheaper and produces better samples. 

Similarly for \imnet, \nfmr begins with flow matching training at a learning rate of $4\times 10^{-4}$ for 200 epochs, and resumes training with the full iMF loss for just 25 epochs at a lower learning rate of $1\times 10^{-4}$. On \imnet both MLE iTARFlow and \nfmd train for 300 epochs with a linear warmup of 25 000 steps up to a learning rate of $4\times 10^{-4}$. For \nfmd we found it beneficial to bring the learning rate down to $1\times 10^{-4}$ with a cosine decay from epoch 150 to 200, after the decay this learning rate is maintained until the end.

One-step \nfmd is trained following~\Cref{alg:nfmd_one_time_training}, where $t_{max}=0.1$ for \cifar and $t_{max} =0.15$ for \imnet.

\xhdr{Evaluation} During training we periodically evaluate the models; for MLE iTARFlow and \nfmd we use a guidance of 2.8 and for \nfmr we use the default guidance settings from pMF. After training, we select the best checkpoint and sweep the guidance setting to find each model's best performance (\Cref{fig:nfmr_cifar_10_guidance_sweeps},~\Cref{fig:nfmd_vs_mle_cifar_10_guidance}).

\xhdr{Baselines} We primarily compare against iTARFlow~\citep{lu2026bidirectional}. To the best of our knowledge, no existing TARFlow works report conditional CIFAR-10 results. Hence, to ensure a fair comparison we ablated improved TARFlow over several learning rates, EMA decays, noise augmentation levels and guidance levels (Table~~\ref{tab:mle_itarflow_cifar}). Furthermore, we also investigated whether the iTARFlow recipe could similarly benefit from the Muon optimizer over AdamW (Table~\ref{tab:mle_itarflow_muon_cifar}). Finally, for ImageNet $64\times 64$ we followed the noise level of $0.05$ reported in the original TARFlow~\citep{zhai2025normalizing} paper and performed two experiments with learning rates of $1\times10^{-4}$ and $4\times 10^{-4}$.

\begin{table}[h]
\centering
\caption{MLE iTARFlow AdamW results on conditional CIFAR-10, results are selected across the learning rates $\{10^{-4}, 4\times 10^{-4}, 10^{-3}\}$ and EMAs $\{0.9999, 0.9995, 0.999\}$. Results computed with the default MLE iTARFlow guidance of 2.8.}
\label{tab:mle_itarflow_cifar}

\resizebox{\linewidth}{!}{%
\begin{tabular}{llcccccc}
\toprule
Setting & Variant & FID $\downarrow$ & FLD $\downarrow$ & KID $\downarrow$ & Prec.\ $\uparrow$ & Rec.\ $\uparrow$ & Invertible? $\downarrow$ \\
\midrule
\multirow{2}{*}{No noise} & Guidance & 40.83 & 23.21 & 2.758 & 0.744 & 0.025 & \checkmark \\
 & No guidance & 55.52 & 26.90 & 4.249 & 0.835 & 0.017 & \checkmark \\
\midrule
\multirow{4}{*}{0.05 noise} & Guidance and Denoise & 6.72 & 8.22 & 0.457 & 0.595 & 0.293 & \xmark \\
 & No guidance & 12.04 & 12.00 & 1.277 & 0.627 & 0.340 & \xmark \\
 & No denoise & 48.34 & 10.05 & 0.923 & 0.676 & 0.122 & \checkmark \\
 & No guid., no denoise & 71.48 & 15.16 & 2.256 & 0.734 & 0.129 & \checkmark \\
\midrule
\multirow{4}{*}{0.1 noise} & Guidance and Denoise & 7.59 & 8.87 & 0.531 & 0.598 & 0.233 & \xmark \\
 & No guidance & 13.99 & 13.22 & 1.513 & 0.625 & 0.273 & \xmark \\
 & No denoise & 79.60 & 12.46 & 1.778 & 0.773 & 0.039 &\checkmark \\
 & No guid., no denoise & 112.64 & 18.39 & 3.886 & 0.862 & 0.023 & \checkmark \\
\midrule
\multirow{4}{*}{0.2 noise} & Guidance and Denoise & 11.46 & 9.89 & 0.800 & 0.631 & 0.159 & \xmark \\
 & No guidance & 17.19 & 14.08 & 1.782 & 0.654 & 0.199 &\xmark \\
 & No denoise & 122.25 & 17.74 & 4.027 & 0.886 & 0.003 & \checkmark\\
 & No guid., no denoise & 163.44 & 24.66 & 6.864 & 0.948 & 0.001 & \checkmark \\
\midrule
\multirow{4}{*}{0.3 noise} & Guidance and Denoise & 15.17 & 11.13 & 1.008 & 0.628 & 0.126 & \xmark \\
 & No guidance & 22.22 & 15.54 & 2.096 & 0.658 & 0.149 & \xmark \\
 & No denoise & 165.36 & 23.71 & 6.915 & 0.938 & 0.000 & \checkmark \\
 & No guid., no denoise & 198.75 & 29.67 & 9.531 & 0.967 & 0.000 & \checkmark \\
\bottomrule
\end{tabular}%
}
\end{table}

\begin{table}[h]
\centering
\caption{MLE iTARFlow results on conditional CIFAR-10 at $0.05$ input noise, comparing the AdamW and Muon optimizers. The AdamW row is selected across the learning rates $\{10^{-4}, 4\times 10^{-4}, 10^{-3}\}$; the Muon row across the same learning rates and Muon weight decays $\{0, 0.1\}$. Both select over EMAs $\{0.9999, 0.9995, 0.999\}$ and steps by FID.  Results computed with the default MLE iTARFlow guidance of 2.8.}
\label{tab:mle_itarflow_muon_cifar}
\resizebox{\linewidth}{!}{%
\begin{tabular}{llcccccc}
\toprule
Optimizer & Variant & FID $\downarrow$ & FLD $\downarrow$ & KID $\downarrow$ & Prec.\ $\uparrow$ & Rec.\ $\uparrow$ &Invertible? \\
\midrule
\multirow{4}{*}{AdamW} & Guidance and Denoise & 6.72 & 8.22 & 0.457 & 0.595 & 0.293 & \xmark \\
 & No guidance & 12.04 & 12.00 & 1.277 & 0.627 & 0.340 & \xmark \\
 & No denoise & 48.34 & 10.05 & 0.923 & 0.676 & 0.122 & \checkmark \\
 & No guid., no denoise & 71.48 & 15.16 & 2.256 & 0.734 & 0.129 & \checkmark \\
\midrule
\multirow{4}{*}{Muon} & Guidance and Denoise & 7.54 & 7.69 & 0.394 & 0.622 & 0.275 & \xmark \\
 & No guidance & 11.62 & 11.35 & 1.177 & 0.630 & 0.322 & \xmark \\
 & No denoise & 44.69 & 9.34 & 0.854 & 0.682 & 0.116 & \checkmark \\
 & No guid., no denoise & 66.66 & 14.18 & 2.075 & 0.732 & 0.121 & \checkmark \\
\bottomrule
\end{tabular}%
}
\end{table}

\clearpage
\begin{figure}[t]
    \centering

    \begin{subfigure}[t]{0.48\textwidth}
        \centering
        \includegraphics[width=\linewidth]
        {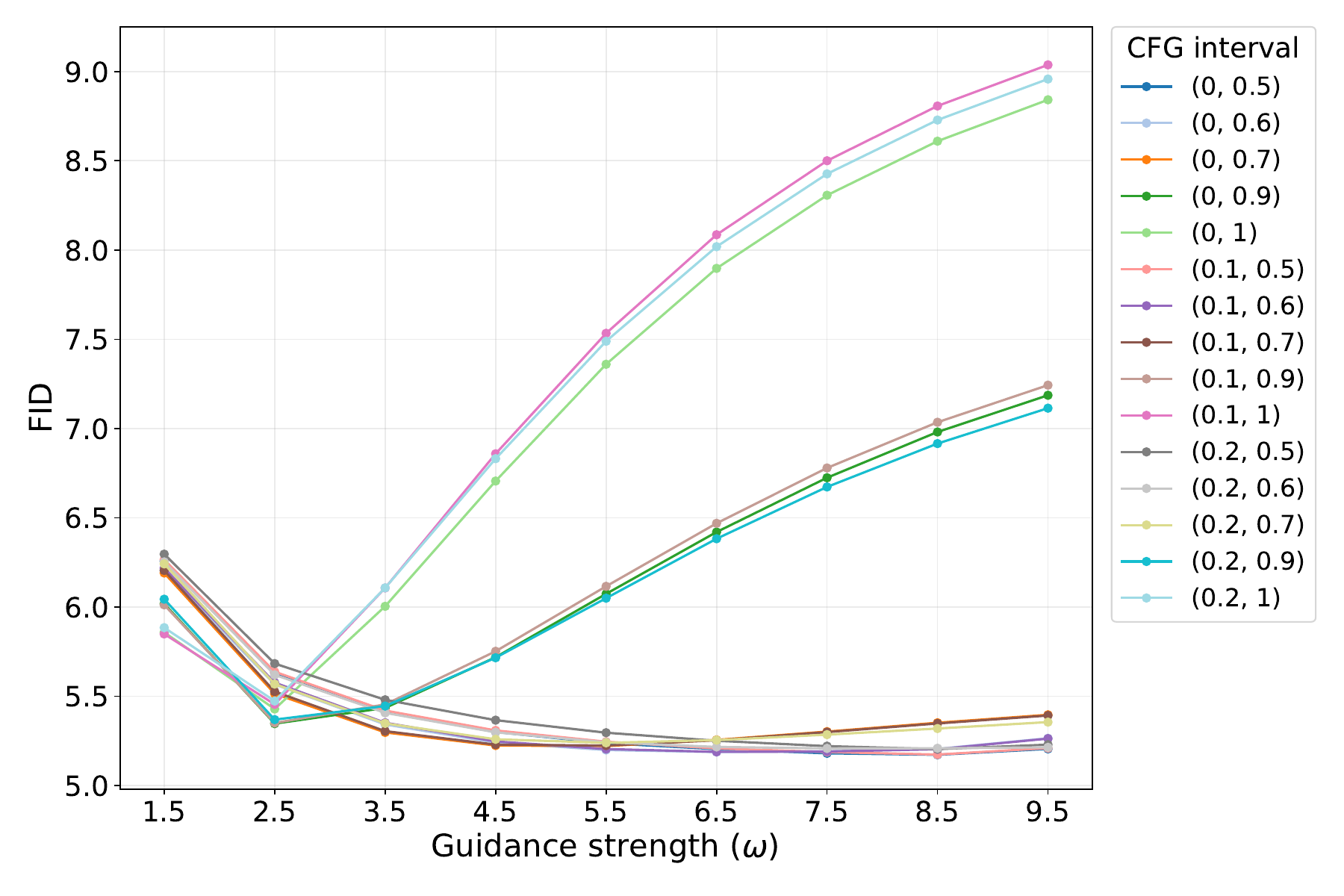}
        \subcaption{1 sampling step.}
        \label{fig:nfmr_cifar_10_k1_guidance}
    \end{subfigure}
    \hfill
    \begin{subfigure}[t]{0.48\textwidth}
        \centering
        \includegraphics[width=\linewidth]
        {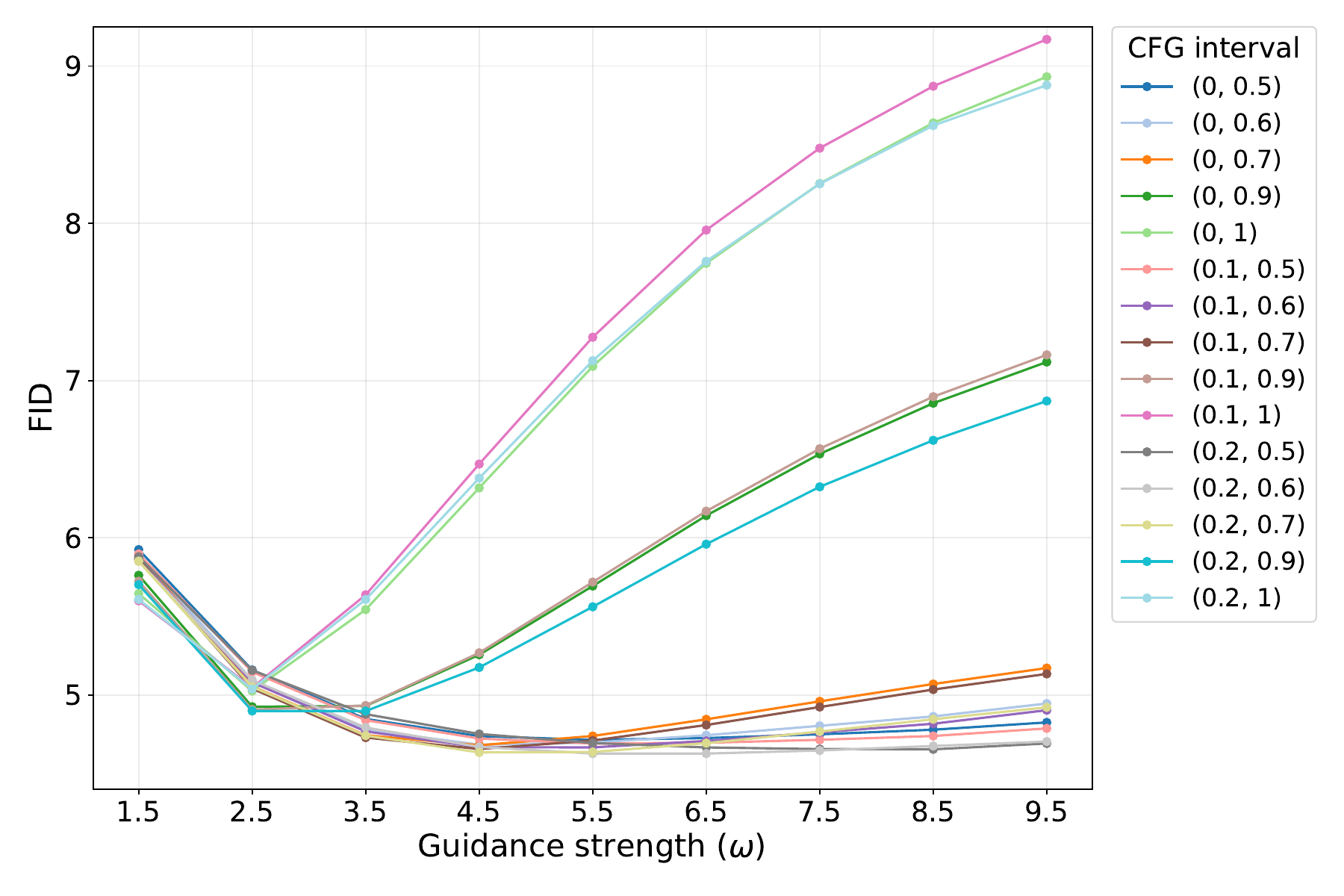}
        \subcaption{2 sampling steps.}
        \label{fig:nfmr_cifar_10_k2_guidance}
    \end{subfigure}

    \caption{\nfmr (\emph{noise annealed}) CIFAR-10 FID as the CFG strength ($\omega$) and the CFG interval
    $[t_{\rm min}, t_{\rm max}]$ vary.}
    \label{fig:nfmr_cifar_10_guidance_sweeps}
\end{figure}

\begin{figure}[t]
    \centering
    \includegraphics[width=0.6\textwidth]
    {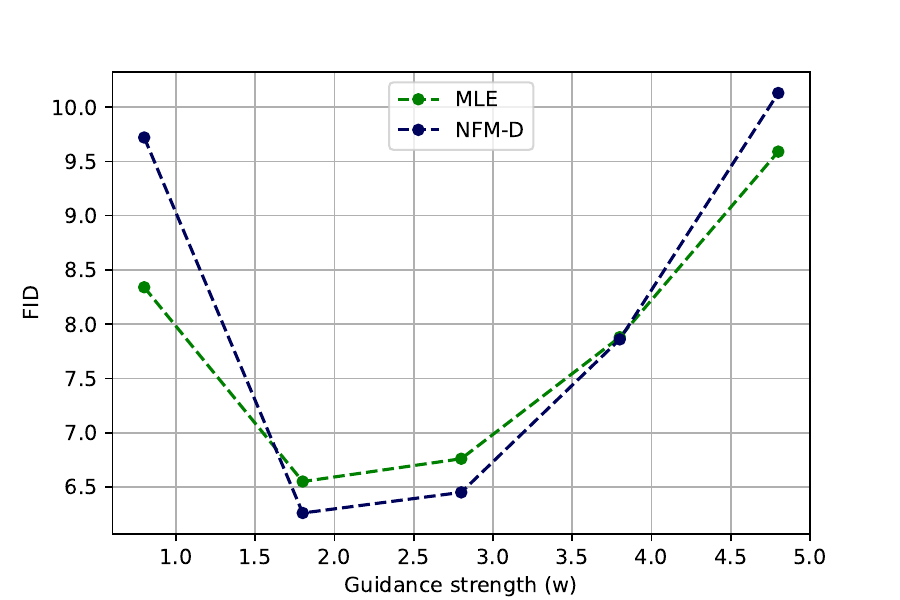}
    \caption{\nfmd and iTARFlow MLE guidance scale (w) vs FID on CIFAR-10}
    \label{fig:nfmd_vs_mle_cifar_10_guidance}
\end{figure}

\xhdr{Multi-step \nfmd}
\looseness=-1 For two-step generation on \cifar, we train the gated dual-time parametrization using equally weighted one-time
and composed likelihood objectives,
\[
\mathcal L
=
\tfrac12\mathcal L_{\mathrm D}
+
\tfrac12\mathcal L_{\mathrm D,2}.
\]
\looseness=-1 Both objectives are evaluated on the full minibatch from the start of training. In the composed objective, the terminal intermediate-to-prior map uses detached current parameters, while retaining gradients through its input. For the one-time objective, we sample $u\sim\mathcal U[0,1]$. For the composed objective, we independently draw two uniform times and order them as $0\leq t<s\leq1$. The two objectives share the same class-dropout mask, with labels removed independently with probability $0.1$.

\looseness=-1 At inference, we use the final checkpoint after $1000$ training epochs with EMA decay $0.9999$.
Sampling follows $1\rightarrow0.75\rightarrow0$: the first step is unconditional, while the second uses
a CFG scale of $3.2$.

\clearpage

\subsubsection{Reinforcement Learning experiments}
\label{app:rl_results}
\xhdr{CIFAR-10: Frogs experiment}
\begin{figure*}[h]
    \centering
    \begin{subfigure}[t]{0.18\textwidth}
        \centering
        \includegraphics[width=\linewidth]{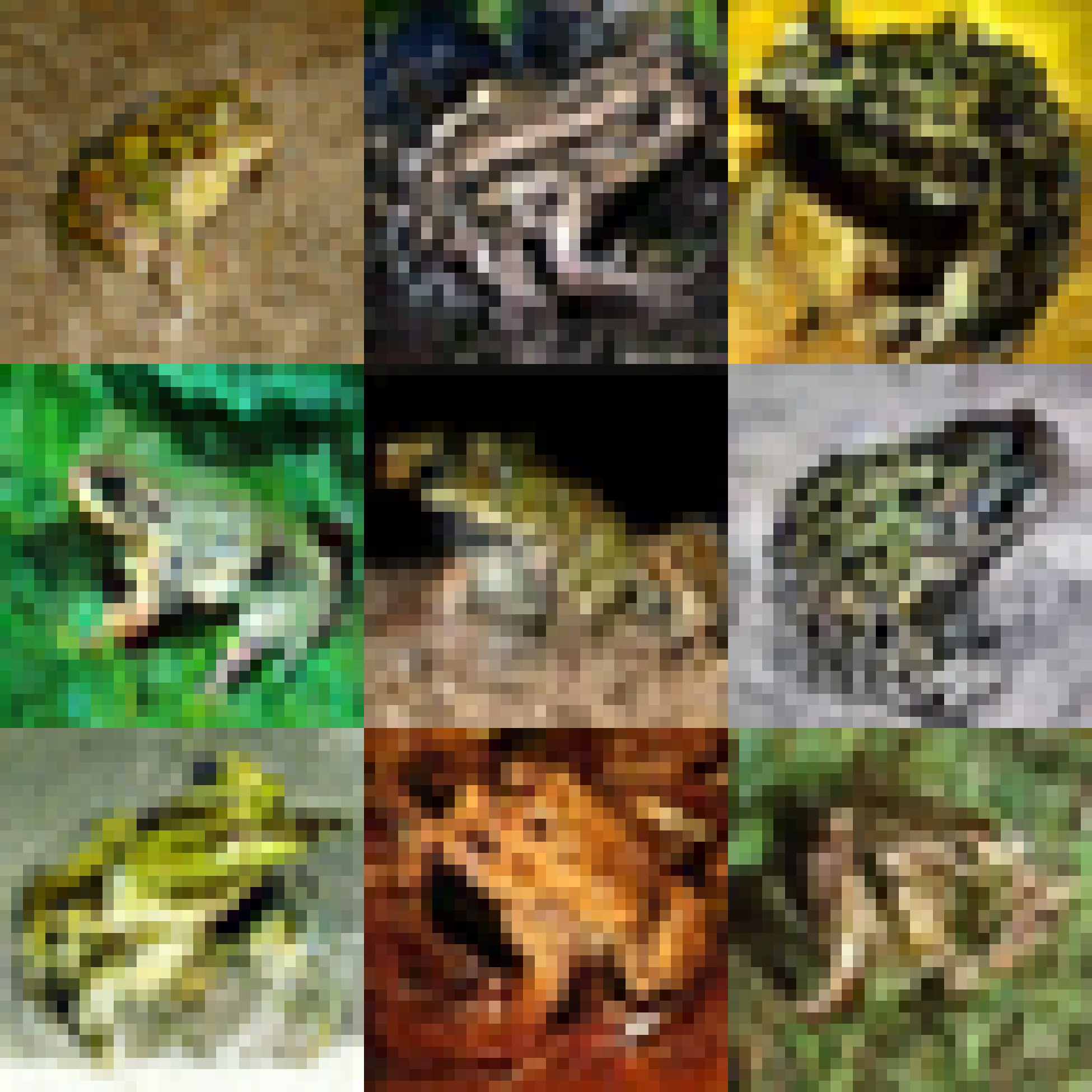}
        \caption{NFM-D, before.}
        \label{fig:green-frogs-nfmd-before}
    \end{subfigure}
    \hfill
    \begin{subfigure}[t]{0.18\textwidth}
        \centering
        \includegraphics[width=\linewidth]{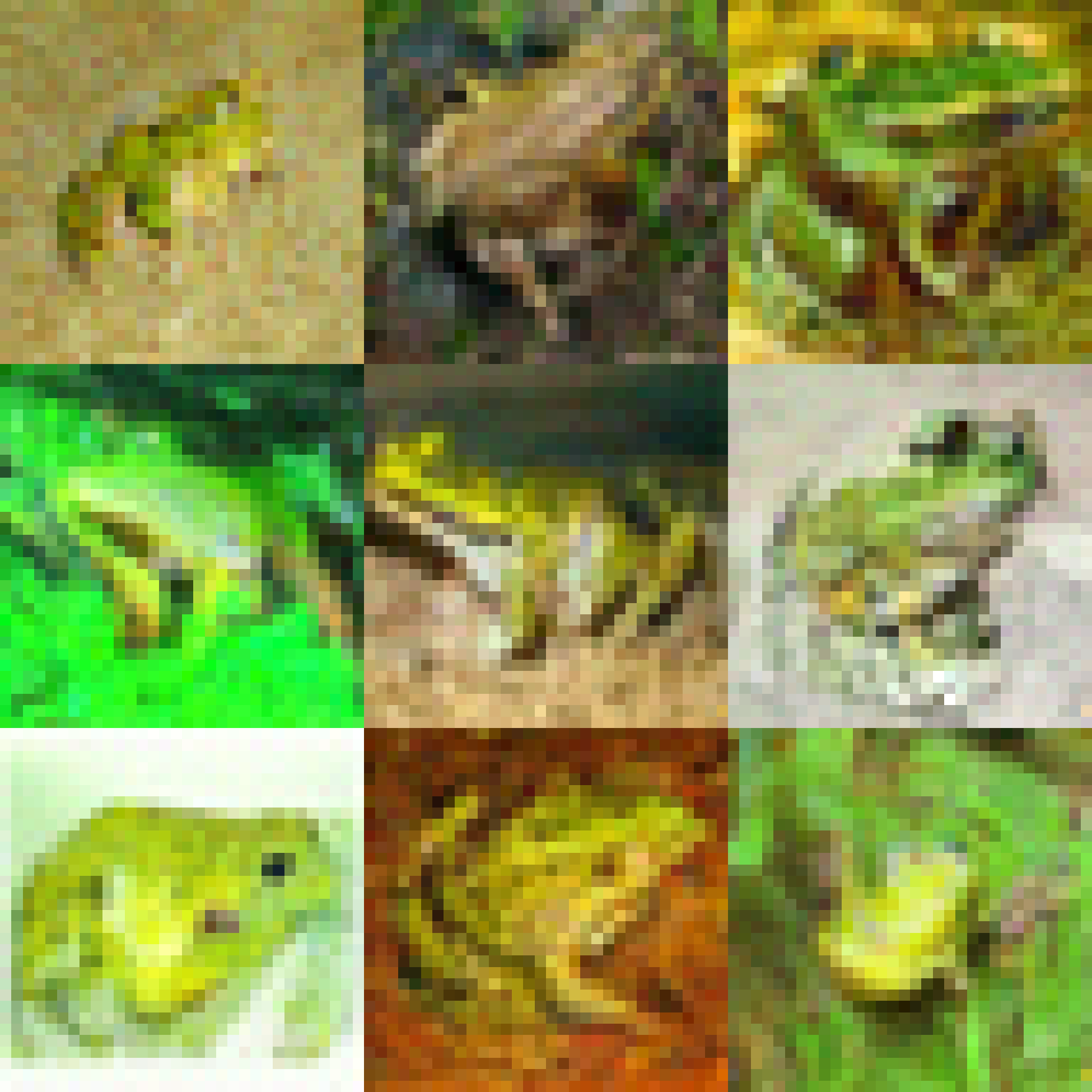}
        \caption{NFM-D, after.}
        \label{fig:green-frogs-nfmd-after}
    \end{subfigure}
    \hfill
    \begin{subfigure}[t]{0.18\textwidth}
        \centering
        \includegraphics[width=\linewidth]{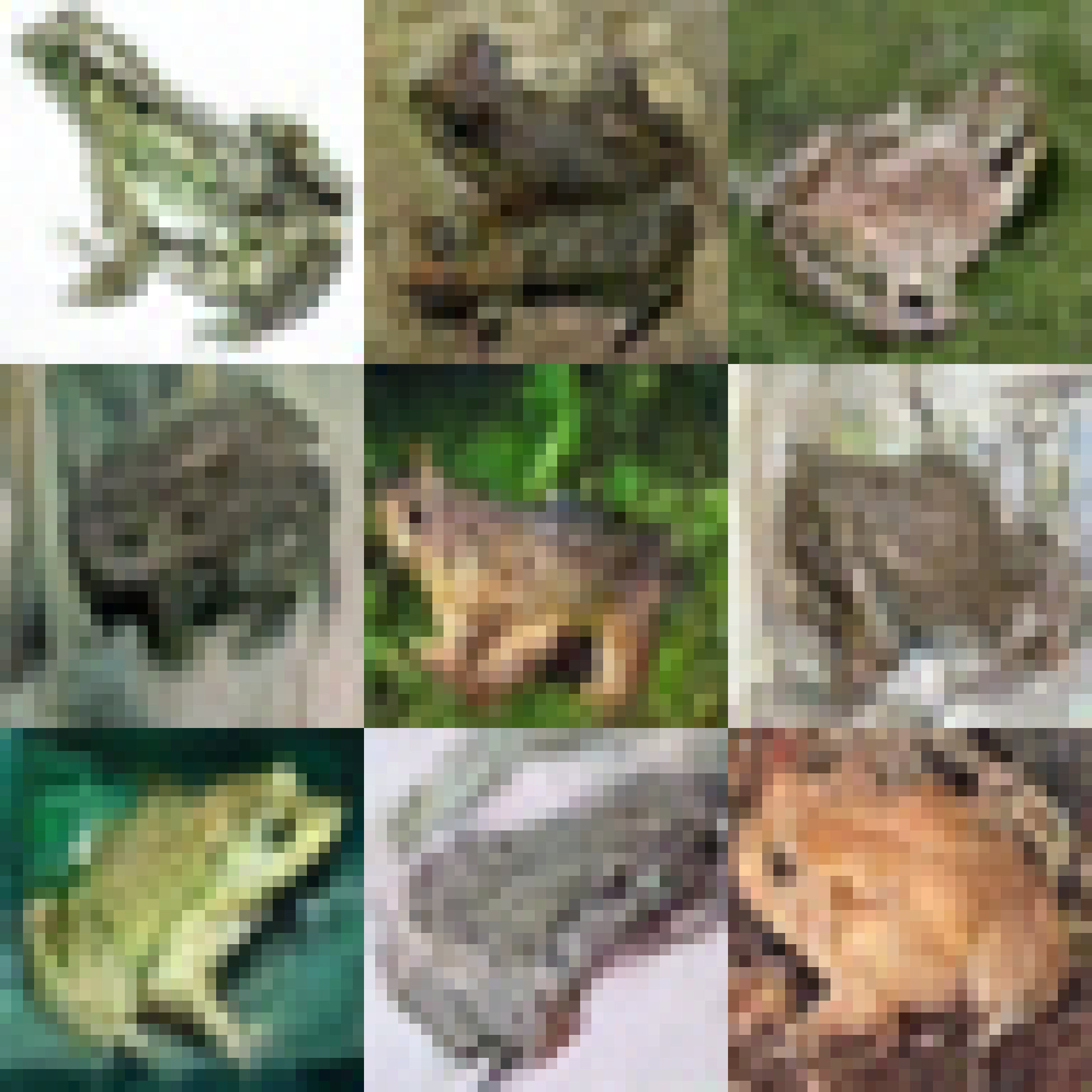}
        \caption{NFM-R, before.}
        \label{fig:green-frogs-nfmr-before}
    \end{subfigure}
    \hfill
    \begin{subfigure}[t]{0.18\textwidth}
        \centering
        \includegraphics[width=\linewidth]{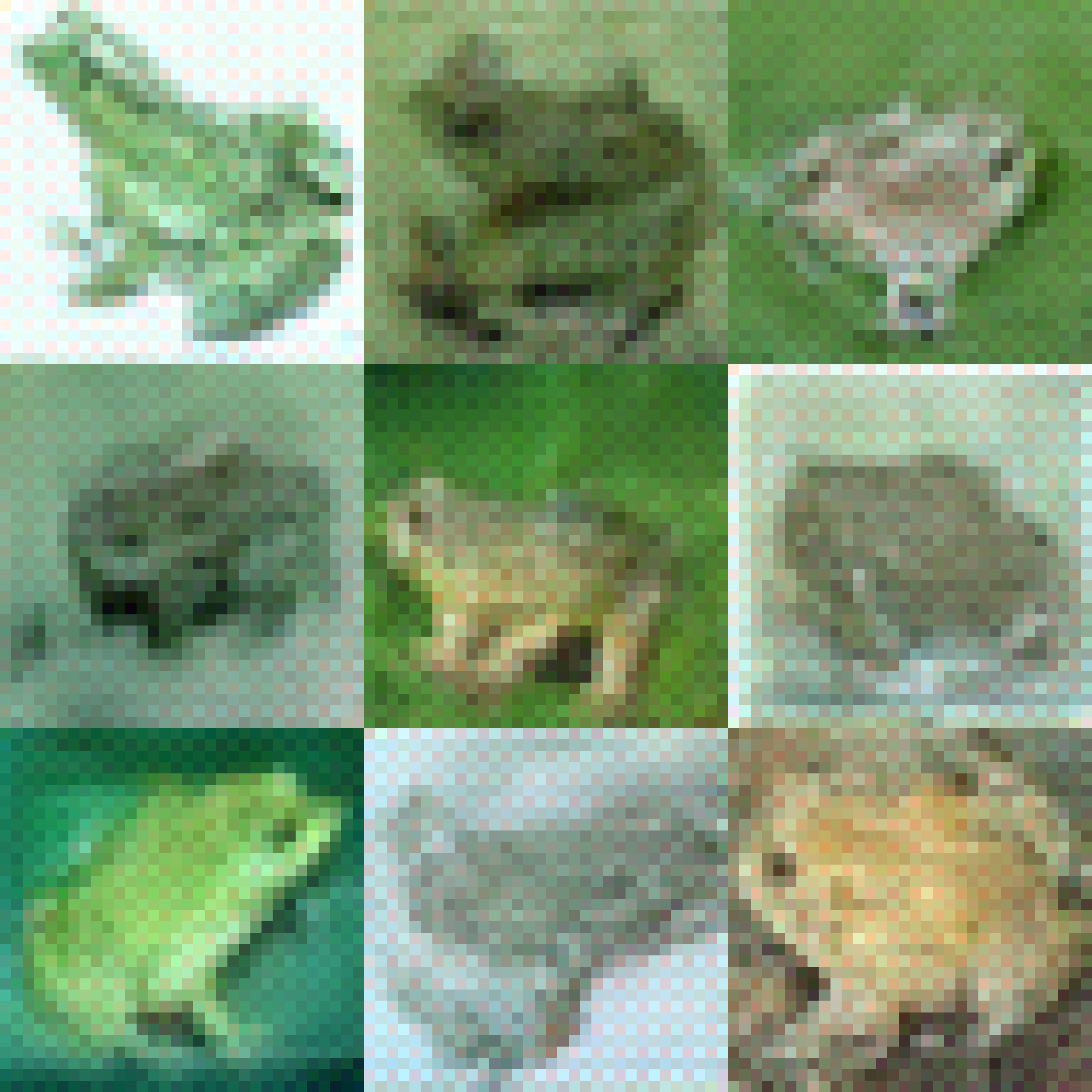}
        \caption{NFM-R, after.}
        \label{fig:green-frogs-nfmr-after}
    \end{subfigure}
    \vspace{-5pt}
    \caption{
        \small \looseness=-1 Samples from \nameshorts pre- and post-REINFORCE fine-tuning
        with a reward favouring green frogs.
    }
    \label{fig:green-frogs}
\end{figure*}

In \Cref{fig:green-frogs} we sample frogs under a fixed noise seed before and after fine-tuning. NFM-D demonstrates a clear shift in model outputs, while NFM-R a subtle adjustment.

\xhdr{ImageNet-64: Text-alignment experiments}

\begin{figure}[h]
  \centering
  \includegraphics[width=0.5\linewidth]{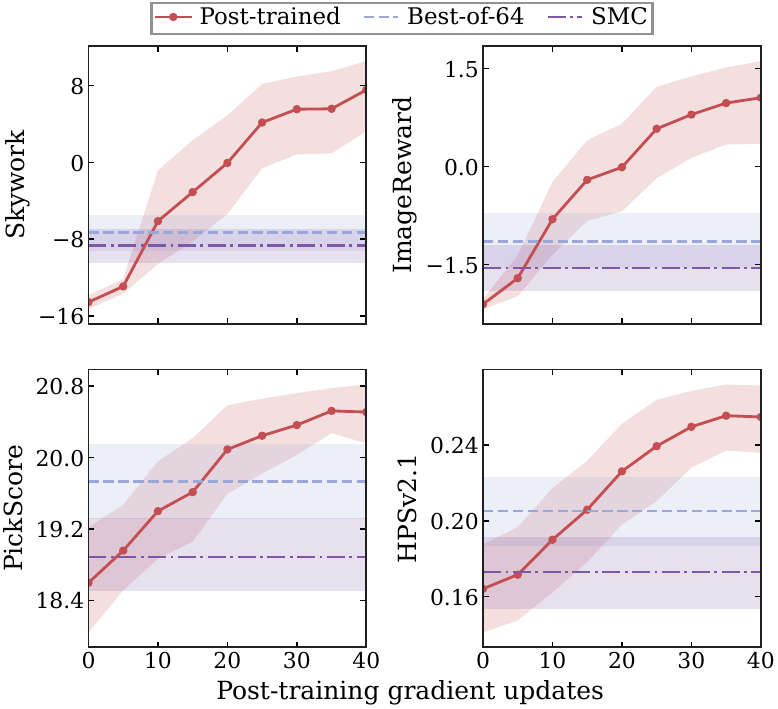}
  \caption{\small \looseness=-1 Post-trained \nfmr vs. best-of-$64$ and SMC baselines. Prompt: \emph{bighorn sheep standing on a sandy beach beside the sea}.}
  \label{fig:nfmr_rl_sheep_metrics}
\end{figure}

In \Cref{fig:nfmr_rl_sheep_metrics} we plot the performance of \nfmr against best-of-N and SMC baselines over 10 seeds and observe a clear improvement over these baselines across all four rewards.

\clearpage

\subsection{Boltzmann Generators}
\label{app:results_BGs}

\xhdr{Datasets}
We evaluate our models in both system-specific and transferable settings. For the system-specific setting, we consider tri-alanine (AL3), tetra-alanine (AL4), and hexa-alanine (AL6), comprising three, four, and six residues, respectively. We additionally evaluate on Chignolin (sequence GYDPETGTWG), a 10-residue miniprotein. For the transferable setting, we use ManyPeptidesMD.

\xhdr{Architecture}
We use the iTARFlow architecture for \nameshort, with modifications to TARFlow detailed in \S~\ref{app:architecture}. Table~\ref{tab:architecture} reports the configurations used in our experiments. The system-specific configurations follow SBG~\citep{tan2025scalable}, except that AL4 uses the smaller AL3 configuration and Chignolin uses the smaller AL6 configuration. The transferable configuration follows transferable TARFlow~\citep{tan2025amortized}.

\begin{table}[h]
\centering
\caption{Architectural hyperparameters used for \nameshort.}
\label{tab:architecture}
\begin{tabular}{lcccc}
\toprule
System & Layers & Blocks & Channels & Parameters (M) \\
\midrule
AL3              & 6 & 6 & 256 & 29.5\\
AL4              & 6 & 6 & 256 & 29.5\\
AL6              & 6 & 6 & 384 & 66.5\\
Chignolin        & 6 & 6 & 384 & 66.5\\
\midrule
Transferable      & 8 & 8 & 384 & 118.6\\
\bottomrule
\end{tabular}
\end{table}

\xhdr{Training configuration}
Following ArBG~\citep{rehman2026autoregressive}, we use the Muon optimizer, departing from the original SBG training configuration. We use a learning rate of $2 \times 10^{-4}$, Nesterov momentum $\mu=0.95$, Muon and AdamW weight decays of $0.1$ and $1\times10^{-4}$, respectively, and AdamW parameters $(\beta_1,\beta_2)=(0.9,0.999)$. We adopt the same cosine-decay learning rate schedule with warmup as SBG~\citep{tan2025scalable}. For the system-specific models, we sweep over 2,000, 2,500, and 3,000 training epochs, while the transferable model is trained for 1,000 epochs. To improve training throughput, we additionally increase the batch size to 2,048. We ablate these training choices in this section.

\subsubsection{Single peptides}
\xhdr{AL3, AL4, and AL6}
Table~\ref{tab:single_peptides} compares \nfmd to different models on AL3, AL4, and AL6. First, \nfmd outperforms SBG. Additionally, it is strong on every baseline, closing the gap to ArBG. We believe there is the need to introduce more single system datasets, including much larger systems (up to 20 residues), that contains equilibrium distributions, as many results are statistically indistinguishable (mostly regarding $\mathbb{T}$-$\mathcal{W}_2$).

\cut{
\begin{table*}[h]
\centering
\caption{Benchmark results on AL3, AL4, and AL6. Reference baselines are shown in gray. Bold values indicate the best result between SBG and \nfmd.}
\resizebox{\textwidth}{!}{%
\begin{tabular}{llcccccc}
\toprule
&
& \multicolumn{2}{c}{Tri-alanine (AL3)}
& \multicolumn{2}{c}{Tetrapeptide (AL4)}
& \multicolumn{2}{c}{Hexa-alanine (AL6)} \\
\cmidrule(lr){3-4}
\cmidrule(lr){5-6}
\cmidrule(lr){7-8}

& Algorithm $\downarrow$
& $\mathcal{E}$-$\mathcal{W}_2 \downarrow$ & $\mathbb{T}$-$\mathcal{W}_2 \downarrow$
& $\mathcal{E}$-$\mathcal{W}_2 \downarrow$ & $\mathbb{T}$-$\mathcal{W}_2 \downarrow$
& $\mathcal{E}$-$\mathcal{W}_2 \downarrow$ & $\mathbb{T}$-$\mathcal{W}_2 \downarrow$ \\
\midrule

\multirow{6}{*}{\rotatebox[origin=c]{90}{\textcolor{refgray}{Regression}}}
& \textcolor{refgray}{ECNF++~\citep{tan2025scalable}}
& \textcolor{refgray}{$2.206 \pm 0.813$}
& \textcolor{refgray}{$0.962 \pm 0.253$}
& \textcolor{refgray}{$5.638 \pm 0.483$}
& \textcolor{refgray}{$1.002 \pm 0.061$}
& \textcolor{refgray}{$10.668 \pm 0.285$}
& \textcolor{refgray}{$1.902 \pm 0.055$} \\

& \textcolor{refgray}{RegFlow~\citep{regflow}}
& \textcolor{refgray}{$0.853 \pm 0.105$}
& \textcolor{refgray}{$1.577 \pm 0.140$}
& \textcolor{refgray}{$3.277 \pm 0.546$}
& \textcolor{refgray}{$2.342 \pm 0.102$}
& \textcolor{refgray}{---}
& \textcolor{refgray}{---} \\

& \textcolor{refgray}{FALCON-A~\citep{rehman2026falcon}}
& \textcolor{refgray}{$1.385 \pm 0.182$}
& \textcolor{refgray}{$0.343 \pm 0.004$}
& \textcolor{refgray}{$2.929 \pm 0.068$}
& \textcolor{refgray}{$1.094 \pm 0.034$}
& \textcolor{refgray}{$1.211 \pm 0.105$}
& \textcolor{refgray}{$1.163 \pm 0.112$} \\

& \textcolor{refgray}{FALCON~\citep{rehman2026falcon}}
& \textcolor{refgray}{$0.544 \pm 0.013$}
& \textcolor{refgray}{$0.452 \pm 0.011$}
& \textcolor{refgray}{$0.686 \pm 0.047$}
& \textcolor{refgray}{$0.858 \pm 0.077$}
& \textcolor{refgray}{$0.892 \pm 0.311$}
& \textcolor{refgray}{$1.256 \pm 0.132$} \\

& \textcolor{refgray}{F2D2~\citep{ai2026joint}}
& \textcolor{refgray}{$2.155 \pm 0.033$}
& \textcolor{refgray}{$0.219 \pm 0.001$}
& \textcolor{refgray}{$4.209 \pm 0.066$}
& \textcolor{refgray}{$1.435 \pm 0.023$}
& \textcolor{refgray}{$10.140 \pm 0.065$}
& \textcolor{refgray}{$1.143 \pm 0.039$} \\

& \textcolor{refgray}{SCALLOP~\citep{scallop_2026}}
& \textcolor{refgray}{$1.596 \pm 0.028$}
& \textcolor{refgray}{$0.212 \pm 0.001$}
& \textcolor{refgray}{$1.545 \pm 0.033$}
& \textcolor{refgray}{$1.399 \pm 0.039$}
& \textcolor{refgray}{$5.845 \pm 0.101$}
& \textcolor{refgray}{$1.120 \pm 0.014$} \\

\midrule

\multirow{3}{*}{\rotatebox[origin=c]{90}{MLE}}
& \textcolor{refgray}{ArBG~\citep{rehman2026autoregressive}}
& \textcolor{refgray}{$0.271 \pm 0.113$}
& \textcolor{refgray}{$0.311 \pm 0.009$}
& \textcolor{refgray}{$0.886 \pm 0.076$}
& \textcolor{refgray}{$0.593 \pm 0.008$}
& \textcolor{refgray}{$0.722 \pm 0.063$}
& \textcolor{refgray}{$1.085 \pm 0.069$} \\

\cmidrule(lr){2-8}

& SBG~\citep{tan2025scalable}
& $0.598 \pm 0.084$
& $0.503 \pm 0.029$
& $1.007 \pm 0.382$
& $1.039 \pm 0.069$
& $1.189 \pm 0.357$
& $1.444 \pm 0.140$ \\

& \textbf{\nfmd (Ours)}
& $\mathbf{0.166 \pm 0.012}$
& $\mathbf{0.428 \pm 0.025}$
& $\mathbf{0.651 \pm 0.145}$
& $\mathbf{0.894 \pm 0.091}$
& $\mathbf{0.325 \pm 0.027}$
& $\mathbf{1.135 \pm 0.004}$ \\

\bottomrule
\end{tabular}%
}
\label{tab:single_peptides}
\end{table*}
}

\begin{table*}[h]
\centering
\caption{Benchmark results for continuous flows on AL3, AL4, and AL6. We report ArBG in gray as an autoregressive model reference. Bold values indicate the best result between SBG and \nfmd.}
\label{tab:single_peptides}
\small
\setlength{\tabcolsep}{3pt}
\resizebox{\textwidth}{!}{%
\begin{tabular}{lcccccc}
\toprule
& \multicolumn{2}{c}{Tri-alanine (AL3)}
& \multicolumn{2}{c}{Tetrapeptide (AL4)}
& \multicolumn{2}{c}{Hexa-alanine (AL6)} \\
\cmidrule(lr){2-3}
\cmidrule(lr){4-5}
\cmidrule(lr){6-7}
Algorithm
& $\mathcal{E}$-$\mathcal{W}_2 \downarrow$
& $\mathbb{T}$-$\mathcal{W}_2 \downarrow$
& $\mathcal{E}$-$\mathcal{W}_2 \downarrow$
& $\mathbb{T}$-$\mathcal{W}_2 \downarrow$
& $\mathcal{E}$-$\mathcal{W}_2 \downarrow$
& $\mathbb{T}$-$\mathcal{W}_2 \downarrow$ \\
\midrule

\textcolor{refgray}{ArBG~\citep{rehman2026autoregressive}}
& \textcolor{refgray}{$0.271 \pm 0.113$}
& \textcolor{refgray}{$0.311 \pm 0.009$}
& \textcolor{refgray}{$0.886 \pm 0.076$}
& \textcolor{refgray}{$0.593 \pm 0.008$}
& \textcolor{refgray}{$0.722 \pm 0.063$}
& \textcolor{refgray}{$1.085 \pm 0.069$} \\

\midrule

ECNF++~\citep{tan2025scalable}
& $2.206 \pm 0.813$
& $0.962 \pm 0.253$
& $5.638 \pm 0.483$
& $1.002 \pm 0.061$
& $10.668 \pm 0.285$
& $1.902 \pm 0.055$ \\

RegFlow~\citep{regflow}
& $0.853 \pm 0.105$
& $1.577 \pm 0.140$
& $3.277 \pm 0.546$
& $2.342 \pm 0.102$
& \textemdash
& \textemdash \\

FALCON-A~\citep{rehman2026falcon}
& $1.385 \pm 0.182$
& $0.343 \pm 0.004$
& $2.929 \pm 0.068$
& $1.094 \pm 0.034$
& $1.211 \pm 0.105$
& $\mathbf{1.163 \pm 0.112}$ \\

FALCON~\citep{rehman2026falcon}
& $\underline{0.544 \pm 0.013}$
& $0.452 \pm 0.011$
& $\mathbf{0.686 \pm 0.047}$
& $\mathbf{0.858 \pm 0.077}$
& $\underline{0.892 \pm 0.311}$
& $\mathbf{1.256 \pm 0.132}$ \\

F2D2~\citep{ai2026joint}
& $2.155 \pm 0.033$
& $\underline{0.219 \pm 0.001}$
& $4.209 \pm 0.066$
& $1.435 \pm 0.023$
& $10.140 \pm 0.065$
& $\mathbf{1.143 \pm 0.039}$ \\

SCALLOP~\citep{scallop_2026}
& $1.596 \pm 0.028$
& $\mathbf{0.212 \pm 0.001}$
& $1.545 \pm 0.033$
& $1.399 \pm 0.039$
& $5.845 \pm 0.101$
& $\mathbf{1.120 \pm 0.014}$ \\

\midrule

SBG~\citep{tan2025scalable}
& $\underline{0.598 \pm 0.084}$
& $0.503 \pm 0.029$
& $1.007 \pm 0.382$
& $1.039 \pm 0.069$
& $1.189 \pm 0.357$
& $1.444 \pm 0.140$ \\

\textbf{\nfmd (Ours)}
& $\mathbf{0.166 \pm 0.012}$
& $0.428 \pm 0.025$
& $\mathbf{0.651 \pm 0.145}$
& $\mathbf{0.894 \pm 0.091}$
& $\mathbf{0.325 \pm 0.027}$
& $\mathbf{1.135 \pm 0.004}$ \\
\bottomrule
\end{tabular}%
}
\end{table*}

\cut{
Table~\ref{tab:sbg_imle} reports a more detailed comparison between SBG and \nfmd. First, we notice our architecture and training configuration improves SBG. Then we find that \nfmd always improves SBG $\mathcal{E}$-$\mathcal{W}_2$, while leaving the  $\mathbb{T}$-$\mathcal{W}_2$ statistically the same to SBG.

\begin{table}[h]
\centering
\caption{Comparison of SBG and \nfmd on AL3, AL4, and AL6. SBG-iTarFlow refers to our reimplementation of SBG, used to control for architectural, optimizer, and training-recipe changes. Best results are in \textbf{bold}. Ties are \underline{underlined}.}
\resizebox{\textwidth}{!}{%
\begin{tabular}{lcccccc}
\toprule
& \multicolumn{2}{c}{Tri-alanine (AL3)}
& \multicolumn{2}{c}{Tetrapeptide (AL4)}
& \multicolumn{2}{c}{Hexa-alanine (AL6)} \\
\cmidrule(lr){2-3}
\cmidrule(lr){4-5}
\cmidrule(lr){6-7}
Algorithm $\downarrow$
& $\mathcal{E}$-$W_2 \downarrow$ & T-$W_2 \downarrow$
& $\mathcal{E}$-$W_2 \downarrow$ & T-$W_2 \downarrow$
& $\mathcal{E}$-$W_2 \downarrow$ & T-$W_2 \downarrow$ \\
\midrule

SBG IS
& $0.758 \pm 0.506$ & $0.502 \pm 0.016$
& $1.068 \pm 0.495$ & $0.969 \pm 0.067$
& $1.021 \pm 0.239$ & $1.431 \pm 0.085$ \\

SBG SMC
& $0.598 \pm 0.084$ & $0.503 \pm 0.029$
& $1.007 \pm 0.382$ & $1.039 \pm 0.069$
& $1.189 \pm 0.357$ & $1.444 \pm 0.140$ \\

\midrule

SBG-iTARFlow IS
& $0.243 \pm 0.023$ & $0.417 \pm 0.021$
& $0.877 \pm 0.127$ & $0.750 \pm 0.056$
& $0.690 \pm 0.127$ & $1.175 \pm 0.048$ \\

\nfmd IS
& $\mathbf{0.166 \pm 0.012}$ & $0.428 \pm 0.025$
& $0.651 \pm 0.145$ & $0.894 \pm 0.091$
& $\mathbf{0.325 \pm 0.027}$ & $1.135 \pm 0.004$ \\

\bottomrule
\end{tabular}%
}
\label{tab:sbg_imle}
\end{table}
}

Finally, we report energy histograms and Ramachandran plots for AL3, AL4 and AL6 in~\cref{fig:energy},~\cref{fig:al3_rama},~\cref{fig:al4_rama}, and~\cref{fig:al6_rama}.
\begin{figure}[h]
    \centering
    \includegraphics[width=1.0\linewidth]{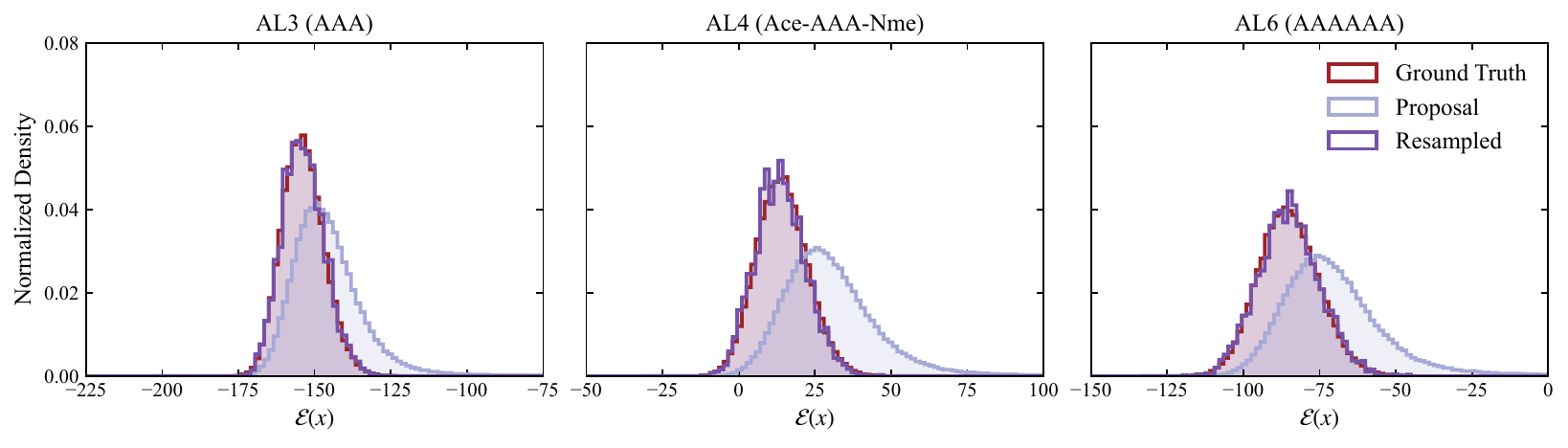} 
    \caption{Energy histograms for the true ground truth MD data, the proposal, and the resampled distribution for AL3, AL4, and AL6.}
    \label{fig:energy}
\end{figure}
\begin{figure}[h]
    \centering
    \includegraphics[width=1.0\linewidth]{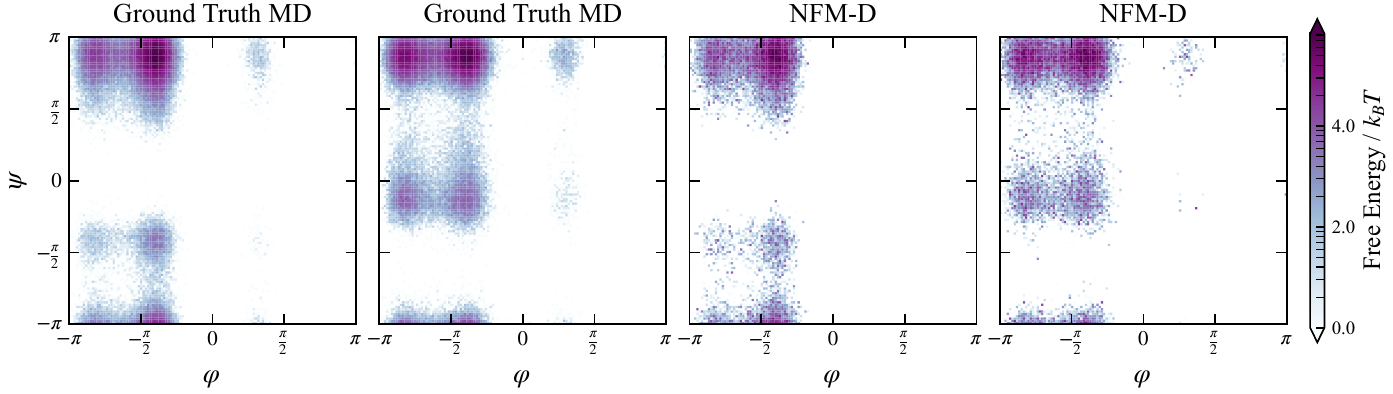}
\caption{Ramachandran plots for tri-alanine (AL3). Left: reference MD data; right: angular distributions generated by NFM-D.}    \label{fig:al3_rama}
\end{figure}
\begin{figure}[h]
    \centering
    \includegraphics[width=1.0\linewidth]{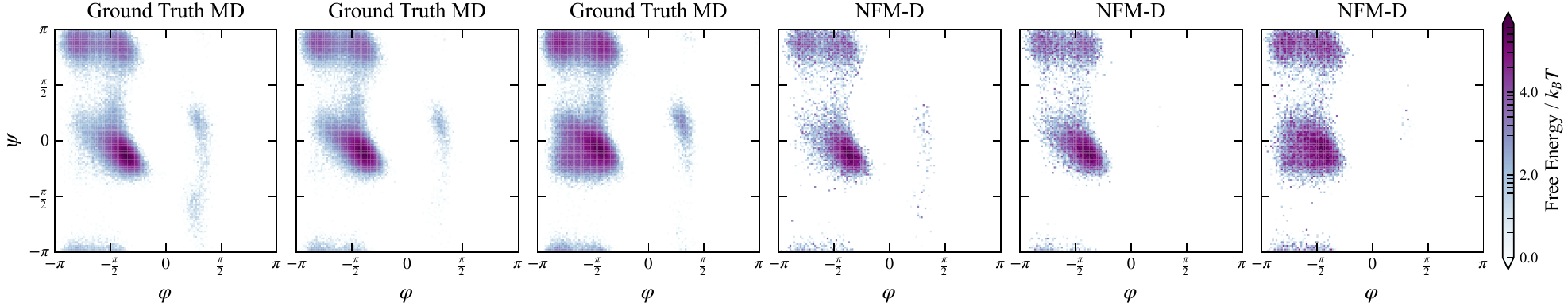}
\caption{Ramachandran plots for tetra-alanine (AL4). Left: reference MD data; right: angular distributions generated by NFM-D.}    \label{fig:al4_rama}
\end{figure}
\begin{figure}[h]
    \centering
    \includegraphics[width=1.0\linewidth]{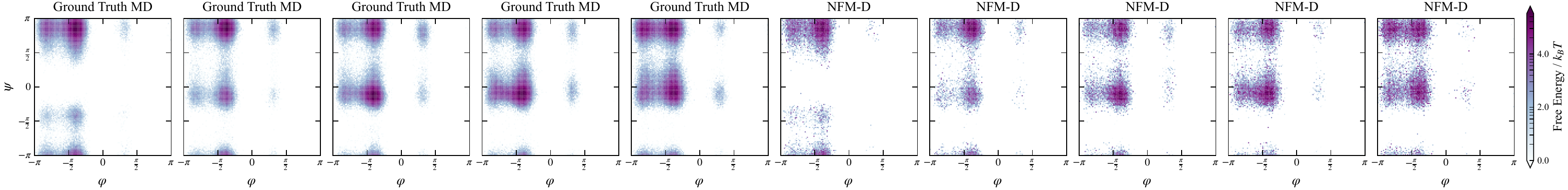}
\caption{Ramachandran plots for hexa-alanine (AL6). Left: reference MD data; right: angular distributions generated by NFM-D.}    \label{fig:al6_rama}
\end{figure}

\clearpage
\xhdr{Comparison between \nfmd and \nfmr}
\Cref{tab:al3_nfm} compares \nfmr and \nfmd on AL3. We observe that the raw metrics are comparable, showing that both \nameshort variants can successfully train NFs. After resampling, however, \nfmd substantially outperforms \nfmr. We attribute this improvement to the better likelihood estimates produced by \nfmd, which enable more effective resampling.

\begin{table}[h]
\centering
\small
\setlength{\tabcolsep}{5pt}
\caption{Comparison of the best \nfmr and \nfmd models on AL3 using raw and resampled metrics. Best results are in \textbf{bold}.}
\label{tab:al3_nfm}
\begin{tabular}{@{}lcccc@{}}
\toprule
& \multicolumn{2}{c}{Raw} & \multicolumn{2}{c}{Resampled} \\
\cmidrule(lr){2-3}\cmidrule(lr){4-5}
Algorithm
& $\mathcal{E}$-$\mathcal{W}_2\downarrow$ & $\mathbb{T}$-$\mathcal{W}_2\downarrow$
& $\mathcal{E}$-$\mathcal{W}_2\downarrow$ & $\mathbb{T}$-$\mathcal{W}_2\downarrow$ \\
\midrule
\nfmr & $>10^9$ & $\mathbf{0.31}$ & $0.54$ & $0.51$ \\
\nfmd & $>10^9$ & $0.32$ & $\mathbf{0.15}$ & $\mathbf{0.40}$ \\
\bottomrule
\end{tabular}
\end{table}

\xhdr{Scaling to decapeptides} Table~\ref{tab:chignolin_additional} compares SBG and \nfmd on Chignolin. The effect of the architecture and training setting is ablated. We train a smaller model compared to SBG to avoid overfitting.

\begin{table}[h]
\centering
\caption{Comparison of SBG and \nfmd on Chignolin. SBG-iTARFlow refers to our reimplementation of SBG, used to control for architectural, optimizer, and training-recipe changes. Best results are in \textbf{bold}; second-best results are \underline{underlined}.}
\begin{tabular}{lccc}
\toprule
Algorithm 
& $\mathcal{E}$-$\mathcal{W}_2 \downarrow$ 
& $\mathbb{T}$-$W_2 \downarrow$
& Parameters (M) \\
\midrule

SBG SMC
& $10.819 \pm 7.206$ & $3.778 \pm 0.440$ & 114 \\

SBG-iTARFlow IS
& $17.015 \pm 11.781$ & $2.845 \pm 0.069$ & $66.5$ \\

SBG-iTARFlow SMC
& $10.874 \pm 10.179$ & $3.558 \pm 0.918$ & $66.5$ \\

\textbf{\nfmd IS (Ours)}
& $\underline{5.118 \pm 0.262}$ & $\mathbf{2.712 \pm 0.050}$ & $66.5$ \\

\textbf{\nfmd SMC (Ours)}
&  $\mathbf{3.419 \pm 0.268}$ & $\underline{2.809 \pm 0.119}$ & $66.5$ \\

\bottomrule
\end{tabular}
\label{tab:chignolin_additional}
\end{table}

\subsection{Transferable Samplers}
\xhdr{Comparison with MLE training}
Table~\ref{tab:ablations_transf} reports the effects of different training settings, architectural choices, and loss functions on transferable TARFlow. Comparing the third row with the preceding baselines, we find that training with Muon and a larger batch size yields modest improvements across most metrics. Replacing TARFlow with the iTARFlow architecture leads to a substantially larger gain, while switching to the \nfmd loss further improves the quality of the learned energy representation.

\begin{table}[h]
\centering
\small
\setlength{\tabcolsep}{2.5pt}
\caption{Quantitative results on unseen 4- and 8-residue peptides, with a budget of $10^4$ energy evaluations at $T=1$. Rows marked with $^\dagger$ denote our reimplementation, used to control for differences in optimizer and training settings. Best results are in \textbf{bold}; second-best results are \underline{underlined}.}
\label{tab:ablations_transf}

\begin{tabular}{lccc|ccc}
\toprule
\multicolumn{1}{l}{Sequence length $\rightarrow$}
& \multicolumn{3}{c|}{4AA $_{\text{(30 systems)}}$}
& \multicolumn{3}{c}{8AA $_{\text{(30 systems)}}$}
\\
\cmidrule(lr){2-4}
\cmidrule(lr){5-7}

Model
& $\mathcal{E}$-$\mathcal{W}_2$ $\downarrow$
& $\mathbb{T}$-$\mathcal{W}_2$ $\downarrow$
& TICA-$\mathcal{W}_2$ $\downarrow$
& $\mathcal{E}$-$\mathcal{W}_2$ $\downarrow$
& $\mathbb{T}$-$\mathcal{W}_2$ $\downarrow$
& TICA-$\mathcal{W}_2$ $\downarrow$
\\
\midrule

MLE-TARFlow (original)
& $1.260$
& $0.924$
& $0.492$
& $11.298$
& $2.733$
& $1.087$
\\

MLE-TARFlow (rerun)
& $1.414$
& $0.933$
& $0.585$
& $12.126$
& $2.808$
& $1.107$
\\

\midrule

MLE-TARFlow$^\dagger$
& $1.362$
& $0.843$
& $0.424$
& $12.723$
& $2.654$
& $0.994$
\\

MLE-iTARFlow
& $\underline{1.068}$
& $\underline{0.757}$
& $\mathbf{0.378}$
& $\underline{8.569}$
& $\underline{2.379}$
& $\underline{0.944}$
\\

\textbf{\nfmd (Ours)}
& $\mathbf{0.907}$
& $\mathbf{0.750}$
& $\underline{0.387}$
& $\mathbf{7.396}$
& $\mathbf{2.315}$
& $\mathbf{0.890}$
\\



\bottomrule
\end{tabular}%
\end{table}

\xhdr{TICA plots for additional unseen octapeptides}
We complement Figure~\ref{fig:energy_distributoions} by reporting additional TICA plots for unseen octapeptides in Figure~\ref{fig:appendix_tica}. This furher demonstrate the zero-shot performance of \nfmd.
\begin{figure}[h]
    \centering
    \includegraphics[width=1.0\linewidth]{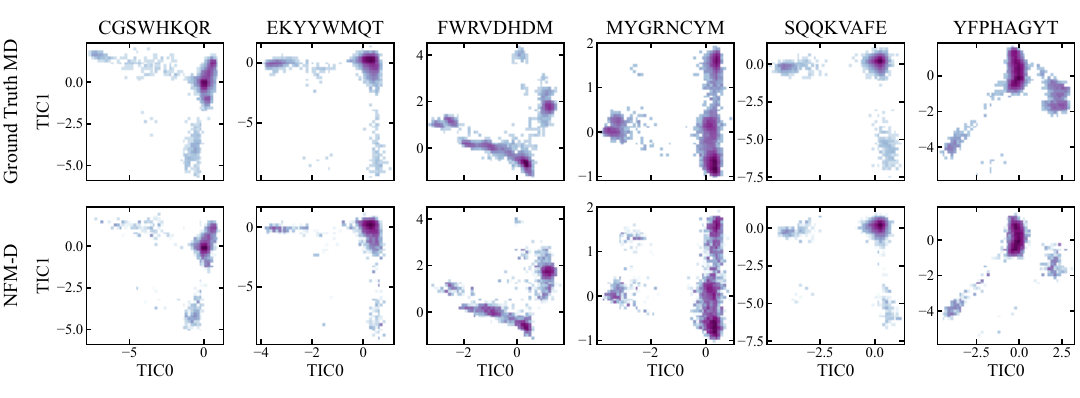}
    \caption{TICA projections for the ground truth MD, and \nfmd on six unseen octapeptides (reported in the titles).}
    \label{fig:appendix_tica}
\end{figure}

\xhdr{Peptide-level performance}
For each 4AA and 8AA peptide in the ManyPeptidesMD test set, we report
$\mathcal{E}$-$\mathcal{W}_2$, $\mathbb{T}$-$\mathcal{W}_2$, and
TICA-$\mathcal{W}_2$. NFM-D generally achieves lower discrepancies than
TARFlow, although the size of the improvement varies across sequences and
metrics.

\begin{figure}[h]
    \centering
    \includegraphics[width=0.8\linewidth]{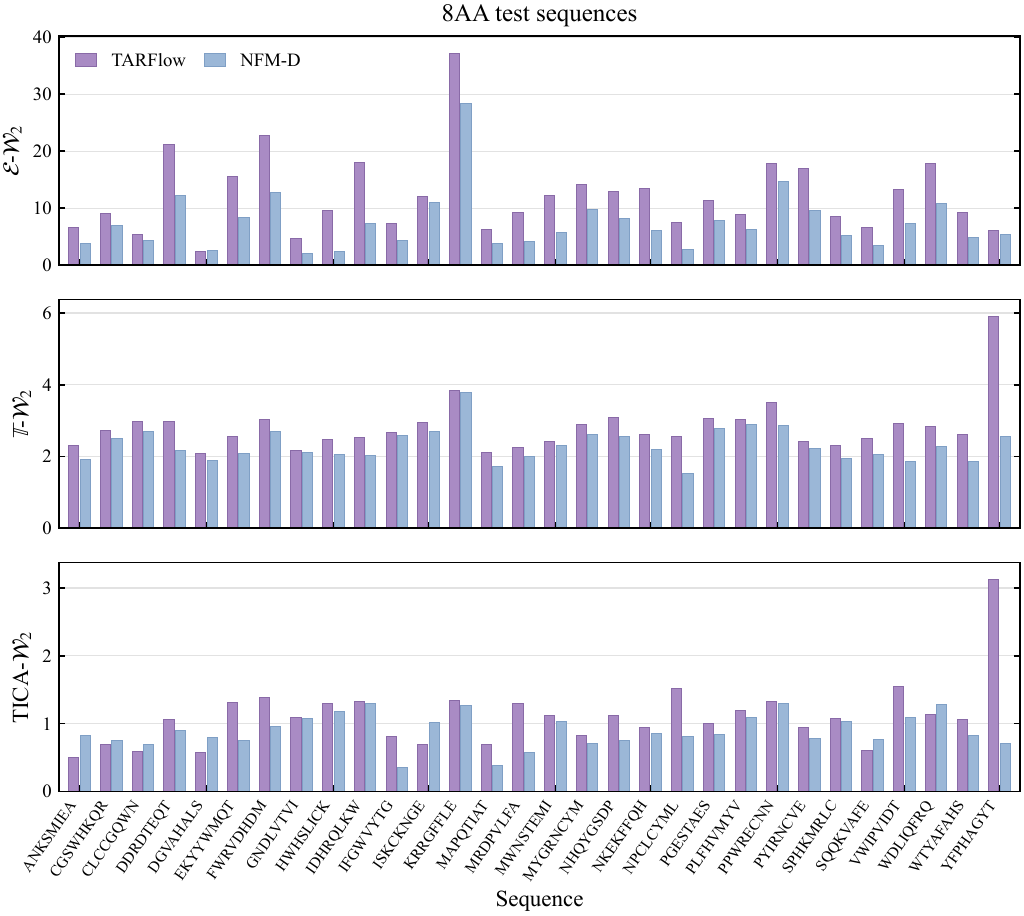}
    \caption{Per-sequence $\mathcal{E}$-$\mathcal{W}_2$,
$\mathbb{T}$-$\mathcal{W}_2$, and TICA-$\mathcal{W}_2$ on the 8AA
ManyPeptidesMD test set. Bars compare TARFlow and NFM-D; lower values
indicate better agreement with the reference distribution.}
    \label{fig:8aa_appendix}
\end{figure}

\begin{figure}[h]
    \centering
    \includegraphics[width=0.8\linewidth]{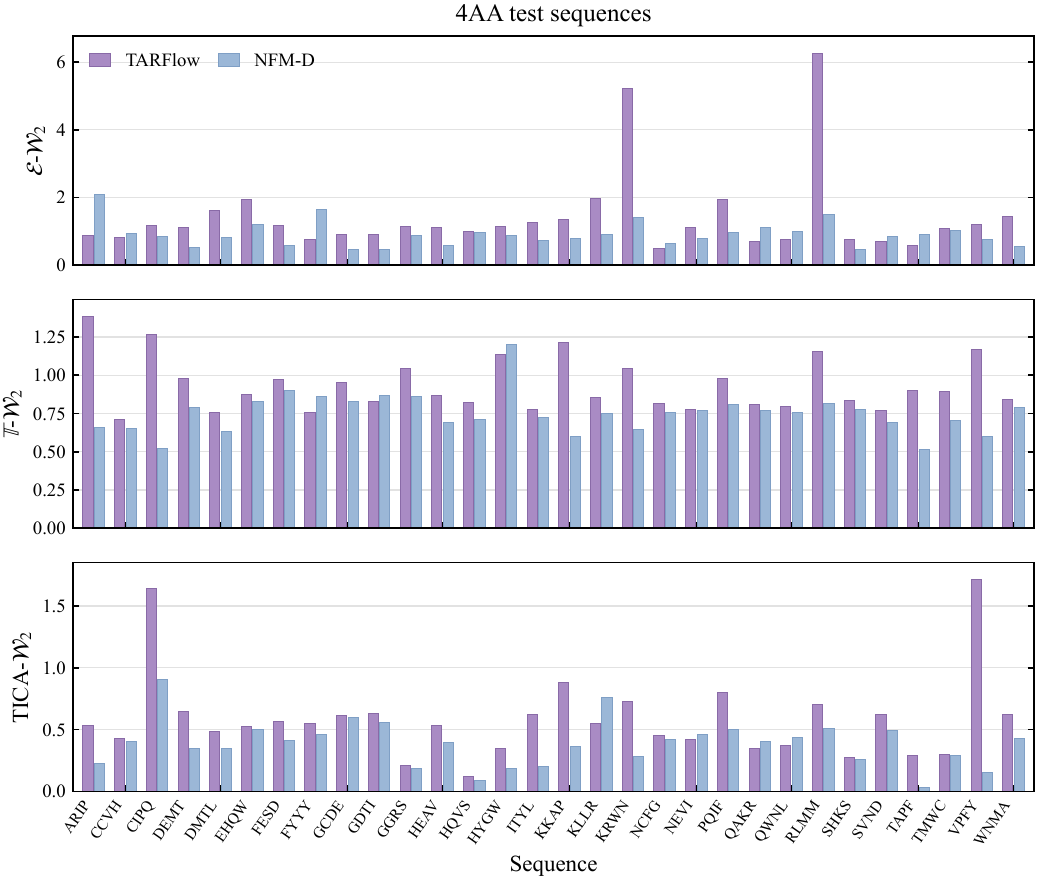}
    \caption{Per-sequence $\mathcal{E}$-$\mathcal{W}_2$,
$\mathbb{T}$-$\mathcal{W}_2$, and TICA-$\mathcal{W}_2$ on the 4AA
ManyPeptidesMD test set. Bars compare TARFlow and NFM-D; lower values
indicate better agreement with the reference distribution.}
    \label{fig:4aa_appendix}
\end{figure}

\clearpage
\section{Additional \imnet Samples}
\label{app:imagenet64-additional-samples}

\begingroup
\newcommand{\NFMAppendixGrid}[3]{%
  \begin{minipage}[t]{0.485\textwidth}
    \centering
    \includegraphics[width=\linewidth]{images/uncurated_samples/#1_final/#1_#2.png}\par
    \vspace{2pt}
    {\small #3\par}
  \end{minipage}%
}
\newcommand{\NFMAppendixPage}[3]{%
  \subsection{Uncurated #2 B/2 Samples}
  \label{app:samples-#1}
  \noindent
  \NFMAppendixGrid{#1}{081}{class 81: ptarmigan}\hfill%
  \NFMAppendixGrid{#1}{108}{class 108: sea anemone}\par\medskip
  \noindent
  \NFMAppendixGrid{#1}{323}{class 323: monarch butterfly}\hfill%
  \NFMAppendixGrid{#1}{327}{class 327: starfish}\par\medskip
  \noindent
  \NFMAppendixGrid{#1}{547}{class 547: electric locomotive}\hfill%
  \NFMAppendixGrid{#1}{628}{class 628: ocean liner}\par\medskip
  \noindent
  \NFMAppendixGrid{#1}{668}{class 668: mosque}\hfill%
  \NFMAppendixGrid{#1}{947}{class 947: mushroom}\par\medskip
  \noindent
  \NFMAppendixGrid{#1}{979}{class 979: valley}\hfill%
  \NFMAppendixGrid{#1}{980}{class 980: volcano}\par
}

\NFMAppendixPage{nfm_r}{NFM-R}{5.5}
\clearpage
\NFMAppendixPage{nfm_d}{NFM-D}{3.3}
\clearpage
\endgroup

\end{document}